\documentclass[10pt]{article} % For LaTeX2e
\usepackage[preprint]{tmlr}
\usepackage{amsmath,amssymb,amsthm}
\allowdisplaybreaks
\usepackage{graphicx}
\usepackage{multirow}
\usepackage{array}
\usepackage{subfig}
\usepackage{verbatim}
\usepackage{booktabs}
\usepackage{algorithm}
\usepackage{algpseudocode}
\usepackage{float}
\usepackage{arydshln}
\usepackage{xcolor}
\usepackage{colortbl}
\usepackage{url}
\usepackage{hyperref}

\newtheorem{theorem}{Theorem}
\newtheorem{proposition}{Proposition}
\newtheorem{lemma}{Lemma}

\newtheorem{definition}{Definition}

\title{Private and Distributionally Robust Federated Multiclass Classification}

\author{\name Andrea P. Hernandez \email ahernandez327@gatech.edu \\
       \addr H. Milton Stewart School of Industrial and Systems Engineering\\
       Georgia Institute of Technology\\
       Atlanta, GA 30332-0205, USA
       \AND
       \name Michael Ibrahim \email mibrahim41@gatech.edu \\
       \addr H. Milton Stewart School of Industrial and Systems Engineering\\
       Georgia Institute of Technology\\
       Atlanta, GA 30332-0205, USA
       \AND
       \name Nagi Gebraeel \email nagi.gebraeel@isye.gatech.edu \\
       \addr H. Milton Stewart School of Industrial and Systems Engineering\\
       Georgia Institute of Technology\\
       Atlanta, GA 30332-0205, USA}

\def\month{MM}  % Insert correct month for camera-ready version
\def\year{YYYY} % Insert correct year for camera-ready version
\def\openreview{\url{https://openreview.net/forum?id=XXXX}} % Insert correct link to OpenReview for camera-ready version

\begin{document}
\raggedbottom

\maketitle

\begin{abstract}
Fault diagnosis in asset-intensive industries requires multiclass models trained on noisy, sensitive data collected across geographically distributed client sites. Federated learning enables collaborative training without centralizing data, but most approaches do not provide formal privacy guarantees for individual records stored at client sites and can degrade under label noise. To address these limitations, we develop a differentially private, distributionally robust framework for federated multiclass classification that provides sample-level privacy for records stored at each client. We build on a federated Wasserstein distributionally robust multiclass objective with the Crammer--Singer loss, solved using an Alternating Direction Method of Multipliers (ADMM)--based algorithm. To protect the released client updates, our method injects calibrated Gaussian noise into the proximal linearized client updates of both the local model and the local distributionally robust variable before aggregation at the central server, and we quantify the cumulative privacy loss using the moments accountant. We provide per-round sample-level \((\epsilon,\delta)\)-differential privacy guarantees and an overall composed privacy guarantee across all training rounds, together with convergence results and an explicit privacy--utility tradeoff. Experiments on simulated industrial fault data and UCI benchmark datasets show that our method outperforms differentially private federated empirical risk minimization and stochastic gradient descent baselines in most tested settings, with the largest gains under strong privacy-induced noise.
\end{abstract}

\section{Introduction}

In asset-intensive industries such as mining, energy, and aerospace, firms increasingly rely on long-term service agreements with Original Equipment Manufacturers (OEMs) to ensure the reliability of critical components over extended periods of time \citep{velmurugan2021asset}. These agreements
are particularly attractive for customers because they shift the risk of equipment failure and maintenance to OEMs, who can reduce maintenance costs
throughout their customer base and pass on savings \citep{bollapragada2007managing}. As part of the service, the OEM is responsible for diagnosing and distinguishing among multiple fault types that can lead to equipment failure. To meet the performance guarantees specified in these
agreements, OEMs must rely on advanced fault diagnostic tools capable of
detecting faults early and accurately. In recent years, data-driven machine
learning (ML) classifiers have become central to enabling such capabilities.
These models have been successfully applied to diagnose faults in various
systems, including gas turbines, lithium-ion batteries, and aircraft fuel systems \citep{li2023intelligent, samanta2021machine, sun2021fault}. 

Despite their promise, the implementation of ML-based fault diagnostic systems in real-world industrial settings presents several key challenges. First, training robust ML models typically requires access to large volumes of operational data generated across geographically distributed client sites. However, sharing such data raises privacy concerns and risks disclosing proprietary or sensitive operational information, often restricted by contractual, regulatory, or competitive constraints. As a result, centralized data collection is frequently infeasible. Second, transmitting high-dimensional sensor data to a centralized cloud system for processing can saturate bandwidth and introduce latency, delaying fault detection \citep{jouini2024survey}. Such delays increase the risk of missing early signs of failure, which undermines the effectiveness of predictive maintenance strategies. Third, harsh industrial environments often introduce noise into sensor readings, and inconsistencies in fault labeling across sites or operators further amplify data uncertainty. These issues collectively degrade model performance and complicate accurate fault classification. 

Federated learning (FL) \citep{syed2023federated} has emerged as a promising alternative to centralized training when data pooling is restricted or impractical because of privacy, ownership, or bandwidth constraints. In FL, clients keep their data locally, train models independently, and send model updates to a central server for aggregation. This decentralized approach has been successfully applied to fault diagnosis tasks \citep{du2023efficient, zhang2021federated, yu2023fedcae}, making it well suited to our problem setting. However, privacy concerns persist even in federated settings. As shown by \citet{dayal2023comparative} and \citet{geng2023improved}, shared model updates can leak sensitive information through the iterative learning process. In industrial fault diagnosis, such leakage may reveal operating conditions or fault severity, thereby risking a client’s competitive advantage. Therefore, while FL reduces the need for centralized data sharing, it does not by itself guarantee privacy, which motivates the incorporation of stronger privacy-preserving mechanisms.

In addition to privacy concerns, industrial fault diagnosis data often exhibit significant uncertainty due to noisy labels and features. Distributionally Robust Optimization (DRO) \citep{kuhn2019wasserstein, ndong2018distributionally, van2021data, mohajerin2018data, delage2010distributionally} provides a principled way to address such uncertainty by constructing a set of plausible distributions, called an ambiguity set, around a nominal distribution, most commonly the empirical distribution, and minimizing the worst-case expected loss over this set. We defer the formal definition to Section~\ref{sec2}. Recent work has also explored DRO in federated learning \citep{ibrahim2024federated}, enabling decentralized training that is more robust to noise. However, most existing approaches focus on binary classification or, when extended to multiple classes, rely on one-vs-one (OvO) or one-vs-all (OvA) strategies for multiclass classification \citep{zgarni2018classification, liu2023fault, islam2019reliable, mandal2020assessment}. Such decompositions can be limiting in practice because they either incur high computational cost or fail to capture relationships among classes. These limitations point to the need for a federated DRO framework that can directly handle multiclass fault classification while remaining robust to uncertainty.

While federated learning and DRO together address decentralization and uncertainty, they do not by themselves eliminate the risk of privacy leakage through shared model updates. To provide strong privacy guarantees, a widely used mechanism is Differential Privacy (DP) \citep{dwork2006differential, dwork2014algorithmic}, which ensures that the distribution of an algorithm’s output changes only slightly when any single data point is added, removed, or modified. In learning settings, this is commonly enforced by adding calibrated noise to the model update before it is released \citep{chaudhuri2008privacy, chaudhuri2011differentially}. Although effective for privacy protection, such perturbations can interfere with optimization and degrade model performance. Furthermore, many existing DP methods rely on assumptions such as strong convexity or smoothness of the objective function, which may not hold for a broad class of loss functions. As a result, incorporating DP into our federated DRO multiclass setting is nontrivial and requires adapting the privacy analysis to this framework.

Our work connects three lines of research that have largely been developed in isolation. Prior federated DRO methods based on mixtures of Wasserstein balls have been developed for binary classification without privacy guarantees~\citep{ibrahim2024federated}, while existing multiclass DRO formulations have primarily been studied in centralized, non-private settings~\citep{ibrahim2024wasserstein}. Separately, differentially private Alternating Direction Method of Multipliers (ADMM) methods have been developed for empirical risk minimization (ERM), including settings without smoothness or strong convexity assumptions~\citep{huang2019dp}, but these methods do not address the federated DRO structure considered here. The one setting in which all three meet uses DRO for a different purpose, namely to counteract the uncertainty introduced by perturbing the local training data for privacy~\citep{shi2022distributionally}. Our work instead combines a federated Wasserstein DRO formulation, a native multiclass loss, and sample-level differential privacy for the released client updates. We build on the mixture-of-Wasserstein-balls formulation of~\citet{ibrahim2024federated}, which considers binary classification and releases exact client updates; we address the multiclass setting directly through the Crammer--Singer loss and release perturbed client updates that satisfy sample-level differential privacy.

\noindent\textbf{Contributions.} The objective of this work is to develop a federated distributionally robust multiclass classification framework for fault diagnosis with formal differential privacy guarantees. The proposed model enables collaborative training across \(G\) clients, accounts for uncertainty in both features and labels, and protects sensitive information through privacy-preserving client updates. To the best of our knowledge, this is among the first works to integrate federated learning, Wasserstein DRO, multiclass classification, and differential privacy into a unified framework for fault diagnosis. The main contributions of this paper are as follows:
\begin{enumerate}
    \item We formulate a federated distributionally robust multiclass classification problem for fault diagnosis that captures uncertainty in both features and labels. We further show that the resulting client update admits a tractable convex optimization formulation, enabling efficient decentralized training.

   \item We propose a differentially private ADMM-based algorithm for the resulting multiclass federated DRO problem, which perturbs both the local model and the local DRO variable with calibrated Gaussian noise. The proposed methodology allows privacy-preserving training under weaker assumptions than those commonly required in existing DP methods, such as strong convexity and smoothness.

    \item We derive sample-level differential privacy guarantees for the noisy client updates and use the moments accountant to compose the resulting privacy loss across training rounds.

    \item We analyze the convergence behavior of the proposed noisy training algorithm and characterize the resulting privacy--utility tradeoff.

    \item We evaluate the proposed method on simulated industrial fault diagnosis data and several UCI benchmark datasets through a comprehensive benchmarking study. The results show that the proposed approach remains competitive under privacy constraints, with additional experiments assessing scalability with respect to the number of clients and classes.

\end{enumerate}

The remainder of this paper is organized as follows. Section~\ref{sec:prelim} introduces the problem setup, the distributionally robust learning framework, and the differential privacy preliminaries. Section~\ref{sec:formulation} develops the federated multiclass DRO formulation and the non-private training algorithm. Section~\ref{sec:dp_algorithm} presents the proposed differentially private algorithm. Section~\ref{sec:privacy} establishes its privacy guarantees, and Section~\ref{sec:convergence} analyzes its convergence and the resulting privacy--utility tradeoff. Section~\ref{sec:experiments} reports the experimental evaluation, and Section~\ref{sec:conclusion} concludes the paper.

\section{Problem Setup and Preliminaries}
\label{sec:prelim}
\subsection{Problem Setup}
We consider a multiclass classification problem in a federated setting with a set of clients $[G] := \{1,\dots,G\}$ and a central server. Each data point is represented by a tuple $\xi=(x,y)$ drawn from a true underlying distribution $\mathbb{P}$, where the feature vector satisfies $x \in \mathcal{X} \subseteq \mathbb{R}^P$ and the class label satisfies $y \in [K] := \{1,2,\dots,K\}$.

To facilitate a structured formulation of robustness against label perturbations, we adopt a one-hot encoding of the label. In this representation, the label vector $\mathbf{y} \in \mathcal{Y} \subset \mathbb{R}^K$ satisfies
\[
\mathcal{Y} = \left\{ \mathbf{y} \in \mathbb{R}^K : \sum_{k=1}^K \mathbf{y}_k = 1,\ \mathbf{y}_k \in \{0,1\} \ \forall k \in [K] \right\}.
\]
With a slight abuse of notation, in the remainder of the paper we use $y$ to denote either the class index in $[K]$ or its one-hot encoding $\mathbf{y}$, depending on the context.

The goal is to learn a linear multiclass classifier parameterized by $M \in \mathbb{R}^{K \times P}$ by minimizing a loss function $\ell(M;\xi)$ over training data, where $\xi=(x,\mathbf{y}) \in \mathcal{X}\times\mathcal{Y}$. In this work, we instantiate $\ell$ using the Crammer--Singer loss: 
\[
\ell_{\mathrm{CS}}(M;\xi)
:=
\max_{k\in[K]}
\left\{
v_{(k)}^\top (Mx - \mathbf{y}) + 1
\right\}
- \mathbf{y}^\top Mx,
\]
where $v_{(k)} \in \mathbb{R}^K$ denotes the one-hot vector for class \( k \). 

We assume access to a training dataset $S=\left\{\hat{\xi}_n=(\hat{x}_n,\hat{y}_n)\right\}_{n=1}^N$,
which is partitioned across the $G$ clients. Each client $g \in [G]$ holds a local dataset $\left\{\hat{\xi}_i^{(g)}=(\hat{x}_i^{(g)},\hat{y}_i^{(g)})\right\}_{i=1}^{N_g}$,
drawn from an unknown local distribution $\mathbb{P}_g$, with corresponding empirical distribution $\widehat{\mathbb{P}}_{N_g}$.

Clients communicate only with the central server and do not exchange information directly with one another. They may share local model updates (e.g., parameters or gradients), but they do not share raw data.

\subsection{Distributionally Robust Learning}
\label{sec2}

Distributionally Robust Optimization (DRO) provides a principled framework for learning models that are robust to distributional shifts \citep{kuhn2019wasserstein, ndong2018distributionally, van2021data, mohajerin2018data, delage2010distributionally}. Instead of minimizing empirical risk, DRO minimizes the worst-case expected loss over an ambiguity set of distributions centered around a nominal distribution. In our setting, the nominal distribution is the empirical distribution $\widehat{\mathbb{P}}_N$ constructed from the training data. Given a loss function $\ell(M;\xi)$, the DRO formulation takes the form
\begin{equation}
\inf_{M} \sup_{\mathbb{Q} \in \mathcal{A}} \mathbb{E}^{\mathbb{Q}}[\ell(M;\xi)],
\label{eq:dro_formulation}
\end{equation}
where $\mathcal{A}$ is an ambiguity set of probability distributions.

In this work, we use a Wasserstein ambiguity set centered at $\widehat{\mathbb{P}}_N$ due to its flexibility and established generalization guarantees:
\begin{equation}
\mathcal{A}_{r,d,q}(\Xi)
:=
\left\{
\mathbb{Q} \in \mathcal{P}(\Xi)
:\,
W_{d,q}(\mathbb{Q},\widehat{\mathbb{P}}_N) \le r
\right\},
\label{eq:wasserstein_set}
\end{equation}
where $\mathcal{P}(\Xi)$ denotes the set of probability distributions supported on $\Xi$, and $W_{d,q}$ denotes the type-$q$ Wasserstein distance induced by the transportation cost $d(\xi,\xi')$.

To model perturbations in both features and labels, we use the type-1 Wasserstein distance with a separable transportation cost \citep{shafieezadeh2019regularization, shafieezadeh2015distributionally}:
\begin{equation}
d(\xi,\xi')
:=
\|x-x'\| + \kappa\,\mathbf{1}\{y \neq y'\},
\label{eq:cost_function}
\end{equation}
where $\xi=(x,y)$, $\xi'=(x',y')$, $\kappa > 0$ penalizes label changes, and $\|\cdot\|$ is a norm on $\mathbb{R}^P$.

\subsection{Differential Privacy}

In federated learning, clients share model updates with a central server, which in our setting is the OEM, rather than sharing raw data. However, these updates may still reveal information about the underlying local datasets through inversion attacks \citep{fredrikson2015model} or membership attacks \citep{shokri2017membership}. In our setting, such leakage could reveal whether a client site operates under a sensitive regime or exhibits unusually high failure rates. Therefore, our goal is to protect individual records within each client's local dataset, rather than to hide whether a client participates in training. We consider two possible adversaries: an honest-but-curious server that follows the protocol but may use client updates to infer information about local records, and an external eavesdropper who has access only to the communicated model updates. We assume that the adversary cannot directly access a client's local dataset.

To address this issue, we adopt differential privacy (DP), a rigorous mathematical framework that provides formal privacy guarantees for individual data records. DP ensures that replacing a single data sample in a dataset does not significantly alter the output distribution of a randomized algorithm, thereby limiting the amount of information that can be inferred about that sample \citep{dwork2006differential}.

\begin{definition}[(\(\epsilon, \delta\))-Differential Privacy \citep{dwork2014algorithmic}]
A randomized mechanism \( \mathcal{M} \) satisfies \((\epsilon, \delta)\)-differential privacy if, for all adjacent datasets \( D \) and \( D' \) differing in one element, and for all measurable sets \( S \),
\[
\mathbb{P}[\mathcal{M}(D) \in S] \le e^\epsilon \, \mathbb{P}[\mathcal{M}(D') \in S] + \delta.
\]
\end{definition}

Here, \(\epsilon\) denotes the privacy budget, with smaller values corresponding to stronger privacy guarantees, while \(\delta\) denotes a small failure probability under which the guarantee may not hold exactly. In this work, we employ the Gaussian mechanism, which adds noise sampled from a Gaussian distribution to each model update before it is sent to the server.

\section{Federated Multiclass Distributionally Robust Formulation}
\label{sec:formulation}

We extend the DRO formulation introduced in
Section~\ref{sec2} to the federated multiclass setting.
Specifically, we use the Crammer--Singer loss for multiclass
classification and define the global ambiguity set as a weighted
mixture of local client Wasserstein balls~\citep{ibrahim2024federated}.
Formally, the ambiguity set is given by
\begin{equation}
\mathcal{A}_G := \bigg\{ \mathbb{Q} = \sum_{g=1}^{G} \alpha_g \mathbb{Q}_g : \alpha_g \geq 0,\ \sum_{g=1}^{G} \alpha_g = 1,\ \mathbb{Q}_g \in \mathcal{A}^{(g)}_{r_g,1,d}(\Xi),\ g=1,\ldots,G \bigg\},
\label{eq:mobw}
\end{equation}
where \(\mathcal{A}^{(g)}_{r_g,1,d}(\Xi)\) denotes the type-1 Wasserstein ball of radius \(r_g\) centered at the empirical distribution \(\widehat{\mathbb{P}}_{N_g}\) of client \(g\), supported on \(\Xi\), and defined by the cost function \(d(\xi,\xi')\). This construction preserves the separable structure of the DRO problem, enabling distributed optimization across clients. Moreover, the Mixture of Wasserstein Balls (MoWB) formulation inherits the finite-sample generalization guarantees established for classical Wasserstein DRO~\citep{ibrahim2024federated}.

To solve the federated optimization problem, we introduce a consensus constraint and adopt an ADMM-based approach. Specifically, letting \( M_g \) denote the local model for client \( g \) and \( M \) the global model, the problem can be written as
\begin{equation}
\label{eq:fed_dro_admm}
\begin{aligned}
\inf_{\{M_g\}, M} \quad
& \sum_{g=1}^G \alpha_g
\sup_{\mathbb{Q}_g \in \mathcal{A}^{(g)}_{r_g,1,d}(\Xi)}
\mathbb{E}^{\mathbb{Q}_g}
\left[ \ell_{\mathrm{CS}}(M_g; \xi) \right] \\
\text{s.t.} \quad
& M_g = M, \quad \forall g \in [G].
\end{aligned}
\end{equation}

\begin{algorithm}[!t]
\caption{Non-Private Fed DRO}
\label{alg:fed_dro}
\begin{algorithmic}[1]
\State \textbf{Initialize:} $M^{(0)}$, $\{M_g^{(0)}\}_{g=1}^{G}$, $\{\mu_g^{(0)}\}_{g=1}^{G}$
\State \textbf{Parameters:} $\rho$, $\{r_g\}_{g=1}^{G}$, $\kappa$, $T$
\State \textbf{Output:} $M^{*}$

\For{$t = 1, \dots, T$}
    \State \textbf{Client Update:}
    \For{clients $g = 1, \dots, G$}
        \State Solve for $(M_g^{(t)},\lambda_g^{(t)})\gets(M_g^*,\lambda_g^*)$, minimizer of $\mathcal{J}_g(M^{(t-1)},\mu_g^{(t-1)})$
        \State Send $M_g^{(t)}$ to the central server
    \EndFor

    \State \textbf{Server Update:}
    \State $M^{(t)} \gets \sum_{g=1}^{G}\alpha_g\left(M_g^{(t)} + \mu_g^{(t-1)}\right)$
    \State Broadcast $M^{(t)}$ to all clients

    \State \textbf{Client Update:}
    \For{clients $g = 1, \dots, G$}
        \State $\mu_g^{(t)} \gets \mu_g^{(t-1)} + M_g^{(t)} - M^{(t)}$
    \EndFor
\EndFor
\end{algorithmic}
\end{algorithm}

To handle the consensus constraints, we construct the augmented Lagrangian as
\[
\mathcal{L}_{\rho}(M_1,\dots,M_G,M,\mu_1,\dots,\mu_G)
=
\sum_{g=1}^{G}\alpha_g \,\mathcal{L}_{\rho,g}(M_g,M,\mu_g),
\]
where \(\mu_g\) is the scaled dual variable associated with client \(g\)'s consensus constraint and has the same dimension as \(M_g\). The local augmented Lagrangian is given by
\begin{equation}
\mathcal{L}_{\rho,g}(M_g,M,\mu_g) = \sup_{\mathbb{Q}_g \in \mathcal{A}^{(g)}_{r_g,1,d}(\Xi)} \mathbb{E}^{\mathbb{Q}_g}\!\left[\ell_{\mathrm{CS}}(M_g;\xi)\right] + \frac{\rho}{2}\left\|M_g-M+\mu_g\right\|_F^2,
\label{eq:local_lagrangian}
\end{equation}
where \(\rho > 0\) is the penalty parameter associated with the consensus constraint.

From the augmented Lagrangian, we obtain the optimization problems solved by the clients and the central server. Specifically, each client updates its local model by minimizing with respect to \(M_g\), while the server updates the global model by minimizing with respect to \(M\). The local problem solved by each client is given in Proposition~\ref{prop:client_update}.

\begin{proposition}[Client update]
\label{prop:client_update}
Given the updated global model \(M^*\), the client update \(M_g^*\) is obtained by solving the following convex optimization problem:
\begin{equation}
\label{eq:orig_cs}
\begin{aligned}
\mathcal{J}_g(M,\mu_g): \quad \min_{M_g,\, \lambda_g,\, \{s_n\}} \quad
& \lambda_g\, r_g + \frac{1}{N_g}\sum_{n=1}^{N_g} s_n + \frac{\rho}{2}\,\|M_g - M + \mu_g\|_F^2 \\
\text{s.t.} \quad
& \max_{c \in [K]} \Bigl[ v_{(c)}^\top (M_g x_n - y_n) + 1 \Bigr] - y_n^\top M_g x_n \le s_n, \quad \forall n \in [N_g], \\
& \max_{c \in [K]} \big[ v_{(c)}^\top (M_g x_n - y) + 1 \big] - y^\top M_g x_n - \kappa\,\lambda_g \le s_n, \\
& \qquad \forall n \in [N_g],\ \forall y \in \mathcal{Y},\ y \neq y_n, \\
& \lambda_g \ge \big\| (v_{(i)} - v_{(j)})^\top M_g \big\|_*, \quad \forall i,j \in [K],
\end{aligned}
\end{equation}
where \(\|\cdot\|_*\) denotes the dual of the norm used in the transportation cost \(d(\xi,\xi')\) in \eqref{eq:cost_function}.
\end{proposition}

\begin{proof}
The proof is given in Appendix~\ref{app:proof_prop_client}.

\end{proof}

After solving its local subproblem, each client shares its updated local model with the central server. The global model is then updated at the server according to the closed-form expression given in Proposition~\ref{prop:server_update}.

\begin{proposition}[Server update]
\label{prop:server_update}
Given the updated local models \(M_g^*\) and scaled dual variables \(\mu_g^*\), the global model is updated according to
{\setlength{\abovedisplayskip}{4pt}
\setlength{\belowdisplayskip}{4pt}
\setlength{\abovedisplayshortskip}{4pt}
\setlength{\belowdisplayshortskip}{4pt}
\begin{equation}
M^{*}
=
\sum_{g=1}^{G} \alpha_g \left(M_g^{*} + \mu_g^{*}\right).
\label{eq:global_update}
\end{equation}
}
\end{proposition}

\begin{proof}
The proof follows as in Proposition~5 of \citet{ibrahim2024federated}, with the vector variable \(w\) replaced by the matrix variable \(M\) and the Euclidean norm replaced by the Frobenius norm.
\end{proof}

The updated global model \(M^*\) is then shared with all clients, after which they update their Lagrange multipliers and repeat the process. Algorithm~\ref{alg:fed_dro} summarizes the non-private Fed DRO algorithm.

\section{DP Federated DRO Algorithm}
\label{sec:dp_algorithm}

In this section, we develop the differentially private version of the federated multiclass DRO algorithm. In the non-private setting, each local update \(M_g^{(t)}\) is obtained by solving the client subproblem \(\mathcal{J}_g(M^{(t-1)},\mu_g^{(t-1)})\) in \eqref{eq:orig_cs} and is sent directly to the server. In contrast, the private algorithm releases the perturbed update \(\widetilde{M}_g^{(t)}\) while keeping the perturbed local variable \(\widetilde{\lambda}_g^{(t)}\) local, where \(\widetilde{M}_g^{(t)}\) and \(\widetilde{\lambda}_g^{(t)}\) are obtained by adding Gaussian noise to \(M_g^{(t)}\) and \(\lambda_g^{(t)}\), respectively. However, directly privatizing the exact client update is difficult because \(\mathcal{J}_g\) does not satisfy the smoothness and strong convexity assumptions commonly used in sensitivity analysis. We therefore replace the exact local DRO subproblem with a proximal linearized approximation, which enables tractable sensitivity analysis while preserving the structure of the original update. To construct this approximation, we first move the norm constraint on \(\lambda_g\) in (7) into the objective through a Bertsekas-type quadratic penalty, where \(\tau_g\ge0\) denotes the associated multiplier, updated by the algorithm in each round, and \(\gamma>0\) is the corresponding penalty parameter. This step yields an unconstrained client objective that is amenable to linearization. Letting \(\phi(M_g):=\max_{i,j}\|(v_{(i)}-v_{(j)})^\top M_g\|_*\), we write the resulting augmented Lagrangian as:
\begin{align}
\mathcal{L}_{\rho,g}(M_g,\lambda_g)
&= \lambda_g\,r_g + \frac{1}{N_g}\sum_{n=1}^{N_g} \max\!\Big\{ \ell_{\mathrm{CS}}(M_g;x_n,y_n),\ \max_{y\ne y_n}\big(\ell_{\mathrm{CS}}(M_g;x_n,y)-\kappa\,\lambda_g\big) \Big\} \notag\\
&\quad + \frac{\rho}{2}\,\|M_g - M + \mu_g\|_F^2 + \frac{1}{2\gamma}\Big( \max\{0,\tau_g+\gamma(\phi(M_g)-\lambda_g)\}^{2}-\tau_g^{2} \Big).
\label{eq:unconstrained_max_form}
\end{align}

\begin{algorithm}[H]
\caption{DP-Fed DRO}
\label{alg:dp_fed_dro}
\begin{algorithmic}[1]
\State \textbf{Initialize:} $M^{(0)}$, $\{M_g^{(0)}\}_{g=1}^{G}$, $\{\widetilde{M}_g^{(0)}\}_{g=1}^{G}$, $\{\lambda_g^{(0)}\}_{g=1}^{G}$, $\{\widetilde{\lambda}_g^{(0)}\}_{g=1}^{G}$, $\{\mu_g^{(0)}\}_{g=1}^{G}$, $\{\tau_g^{(0)}\}_{g=1}^{G}$
\State \textbf{Parameters:} $\rho$, $\gamma$, $\{\eta_g\}$, $\{\nu_g\}$, $\{\sigma_{M,g}^2\}$, $\{\sigma_{\lambda,g}^2\}$, $T$
\State \textbf{Output:} $M^{*}$

\For{$t = 1, \dots, T$}
    \State \textbf{Client Update:}
    \For{clients $g = 1, \dots, G$}
        \State $(M_g^{(t)},\lambda_g^{(t)}) \gets \arg\min_{(M_g,\lambda_g)} \widehat{\mathcal{L}}_{\rho,g}^{\,t}(M_g,\lambda_g)$
        \State $\xi_g^{(t)} \gets \mathcal{MN}_{K,P}\!\left(0, I_K, \sigma_{M,g}^2 I_P\right)$
        \State $\zeta_g^{(t)} \gets \mathcal{N}\!\left(0,\sigma_{\lambda,g}^2\right)$
        \State $\widetilde{M}_g^{(t)} \gets M_g^{(t)} + \xi_g^{(t)}$
        \State $\widetilde{\lambda}_g^{(t)} \gets \lambda_g^{(t)} + \zeta_g^{(t)}$
        \State Send $\widetilde{M}_g^{(t)}$ to the central server
    \EndFor
    \State \textbf{Server Update:}
    \State $M^{(t)} \gets \sum_{g=1}^{G}\alpha_g\left(\widetilde{M}_g^{(t)} + \mu_g^{(t-1)}\right)$
    \State Broadcast $M^{(t)}$ to all clients
    \State \textbf{Client Update:}
    \For{clients $g = 1, \dots, G$}
        \State $\mu_g^{(t)} \gets \mu_g^{(t-1)} + \widetilde{M}_g^{(t)} - M^{(t)}$
        \State $\tau_g^{(t)} \gets \max\!\left\{0,\tau_g^{(t-1)} + \gamma\bigl(\phi(\widetilde{M}_g^{(t)})-\widetilde{\lambda}_g^{(t)}\bigr)\right\}$
    \EndFor
\EndFor
\end{algorithmic}
\end{algorithm}

Then, we form a first-order approximation of the augmented
Lagrangian and add \(\ell_2\)-norm proximal terms. For client \(g\)
at iteration \(t\), the resulting approximation is given by
\begin{align}
\widehat{\mathcal L}_{\rho,g}^t(M_g,\lambda_g)
={}& f(\widetilde M_g^{t-1},\widetilde\lambda_g^{t-1}) + \lambda_g r_g + \left\langle \partial_{M_g}f(\widetilde M_g^{t-1},\widetilde\lambda_g^{t-1}), M_g-\widetilde M_g^{t-1} \right\rangle_F \notag\\*
&+ \left\langle \partial_{\lambda_g}f(\widetilde M_g^{t-1},\widetilde\lambda_g^{t-1}), \lambda_g-\widetilde\lambda_g^{t-1} \right\rangle + \frac{\rho}{2} \left\|M_g-M^{t-1}+\mu_g^{t-1}\right\|_F^2 \notag\\*
&+ \frac{1}{2\eta_g} \left\|M_g-\widetilde M_g^{t-1}\right\|_F^2 + \frac{1}{2\nu_g} \left\|\lambda_g-\widetilde\lambda_g^{t-1}\right\|^2 .
\label{eq:prox_linearized_lagrangian}
\end{align}

where \(f(M_g,\lambda_g)\) collects the nonsmooth terms in the client objective
\begin{align}
f(M_g,\lambda_g)
:={}& \frac{1}{N_g}\sum_{n=1}^{N_g} \max\!\Big\{ \ell_{\mathrm{CS}}(M_g; x_n, y_n),\ \max_{y \in \mathcal{Y}\setminus\{y_n\}} \big(\ell_{\mathrm{CS}}(M_g; x_n, y) - \kappa\,\lambda_g\big) \Big\} \notag\\
&+ \frac{1}{2\gamma}\Big( \max\{0,\, \tau_g + \gamma\big(\phi(M_g)-\lambda_g\big)\}^{2} - \tau_g^{2} \Big).
\label{eq:def_f_local}
\end{align}

Here, \(\eta_g>0\) and \(\nu_g>0\) are proximal step sizes held constant across rounds and defined in Section~\ref{sec:convergence}. 

To solve the problem, we first initialize the global model \(M\), the local models \(M_g\), the local DRO variable \(\lambda_g\), and the dual variables \(\mu_g\) and \(\tau_g\). At each round \(t\), every client solves its local subproblem by minimizing \(\widehat{\mathcal{L}}_{\rho,g}^t(M_g,\lambda_g)\) with respect to \((M_g,\lambda_g)\), yielding the local updates \(M_g^t\) and \(\lambda_g^t\). Each client then samples a noise matrix \(\xi_g^t \sim \mathcal{MN}_{K,P}(0, I_K,(\sigma_{M,g})^2 I_P)\) and a scalar noise variable \(\zeta_g^t \sim \mathcal{N}(0,(\sigma_{\lambda,g})^2)\), where \(\mathcal{MN}_{K,P}(0, I_K,(\sigma_{M,g})^2 I_P)\) denotes a \(K \times P\) matrix normal distribution with zero mean and entrywise variance \((\sigma_{M,g})^2\). The noise terms are added to \(M_g^t\) and \(\lambda_g^t\) to obtain the perturbed variables \(\widetilde{M}_g^t\) and \(\widetilde{\lambda}_g^t\), respectively. The perturbed local updates \(\widetilde{M}_g^t\) are shared with the central server, while \(\widetilde{\lambda}_g^t\) remains local and is carried forward to the next iteration. The server uses \(\widetilde{M}_g^t\) together with the current dual variables \(\mu_g\) to compute the updated global model \(M^t\). The server then shares \(M^t\) with all clients, which update their dual variables using \(\widetilde{M}_g^t\) and \(\widetilde{\lambda}_g^t\) and proceed to the next round. This process is repeated for \(T\) rounds.

The design of Algorithm~\ref{alg:dp_fed_dro} is guided by three main considerations. First, the approximate augmented Lagrangian enables us to handle a nonsmooth objective by replacing it with a first-order approximation. This makes the local subproblem easier to solve, leads to a closed-form update, and allows faster computation. More importantly, it enables us to derive upper bounds on the sensitivity of the local updates without requiring the original objective to be smooth or strongly convex. Second, the proximal terms $\|M_g - \widetilde{M}_g^{t-1}\|$ and $\|\lambda_g - \widetilde{\lambda}_g^{t-1}\|$ keep the updated variables close to the previous perturbed iterates, which regularizes the local surrogate problem and prevents the linearized approximation from deviating too far from its expansion point. This helps stabilize the optimization procedure and promotes smoother updates across iterations. Third, the step sizes \(\eta_g\) and \(\nu_g\) control both optimization progress and privacy. Since the sensitivities of the local updates scale with the step sizes, the step size directly determines the amount of injected noise, which allows the two to be tuned jointly through a single parameter.

\section{Differential Privacy Guarantees}
\label{sec:privacy}
In this section, we establish the differential privacy guarantees of the proposed algorithm. We begin by deriving sensitivity bounds for the local model update and the local DRO variable, which determine the variances of the Gaussian perturbations and serve as the basis for the privacy analysis across rounds. The following lemma first provides the sensitivity bound for the local model update.

\begin{lemma}[Sensitivity of the local model update]
\label{lem:sensitivity_Mg}
Let \(D_g\) and \(D_g'\) be two adjacent local datasets that differ in exactly one sample. Assume that the feature vectors satisfy \(\|x_n\|_2 \le L\) for all \(n\), and let \(M_g^t(D_g)\) denote the client update obtained from the proximal linearized subproblem at iteration \(t\). Then the \(\ell_2\)-sensitivity of the released client update satisfies
\[
\max_{D_g \sim D_g'}
\left\|M_g^t(D_g)-M_g^t(D_g')\right\|_F
\le
\frac{2\sqrt{2}\,L}{N_g\left(\rho+\frac{1}{\eta_g}\right)}.
\]
\end{lemma}

\begin{proof}
The proof is given in Appendix~\ref{app:proof_lem_sens_M}.
\end{proof}

The assumption in Lemma~\ref{lem:sensitivity_Mg} implies a uniform bound on the Frobenius norm of each individual subgradient, which is ensured by the normalization of the feature vectors. We next derive
the sensitivity of the local DRO variable. 

\begin{lemma}[Sensitivity of the local DRO variable]
\label{lem:sensitivity_lambda}
Let \(D_g\) and \(D_g'\) be two adjacent local datasets that differ
in exactly one sample, and let \(\lambda_g^t(D_g)\) denote the
local DRO variable obtained from the proximal linearized
subproblem at iteration \(t\). Then,
\[
\max_{D_g\sim D_g'}
\left|
\lambda_g^t(D_g)-\lambda_g^t(D_g')
\right|
\le
\frac{\kappa\nu_g}{N_g}.
\]
\end{lemma}

\begin{proof}
The proof is given in Appendix~\ref{app:proof_lem_sens_lambda}.
\end{proof}

Using the sensitivity bounds from Lemmas~1 and~2, we now calibrate
the Gaussian noise added to the local model update and DRO variable,
and establish the corresponding differential privacy guarantees for
the perturbed client state at each iteration of Algorithm~2.

\begin{theorem}[Differential Privacy of the Perturbed Client State]
\label{thm:dp_client_update}
Under the assumptions of Lemmas~\ref{lem:sensitivity_Mg} and
\ref{lem:sensitivity_lambda}, consider the perturbed variables
\[
\widetilde{M}_g^t=M_g^t+\xi_g^t,
\qquad
\widetilde{\lambda}_g^t
=\lambda_g^t+\zeta_g^t,
\]
where \(\xi_g^t\sim\mathcal{MN}_{K,P}(0,I_K,\sigma_{M,g}^2 I_P)\) and
\(\zeta_g^t\sim\mathcal{N}(0,\sigma_{\lambda,g}^2)\), with
\[
\sigma_{M,g} =
\frac{2\sqrt{2}L}
{N_g\left(\rho+\frac{1}{\eta_g}\right)}
\frac{\sqrt{2\ln(1.25/\delta_M)}}{\epsilon_M},
\]
and
\[
\sigma_{\lambda,g} =
\frac{\kappa\nu_g}{N_g}
\frac{\sqrt{2\ln(1.25/\delta_\lambda)}}{\epsilon_\lambda}.
\]
Then \(\widetilde{M}_g^t\) satisfies
\((\epsilon_M,\delta_M)\)-differential privacy and
\(\widetilde{\lambda}_g^t\) satisfies
\((\epsilon_\lambda,\delta_\lambda)\)-differential privacy.
\end{theorem}

\begin{proof}
By Lemma~\ref{lem:sensitivity_Mg}, the local model update
\(M_g^t\) has bounded \(\ell_2\)-sensitivity, and by
Lemma~\ref{lem:sensitivity_lambda}, the local DRO variable
\(\lambda_g^t\) has bounded sensitivity. Applying the Gaussian
mechanism \cite[Theorem~3.22]{dwork2014algorithmic} separately
to each variable, with noise calibrated to its corresponding
sensitivity, yields the stated privacy guarantees for
\(\widetilde{M}_g^t\) and \(\widetilde{\lambda}_g^t\).
\end{proof}

Theorem~\ref{thm:dp_client_update} establishes the individual privacy guarantees for the
two Gaussian mechanisms applied at each training iteration. We now
extend this analysis to the full training procedure over \(T\) rounds.
Since the perturbed variables are generated sequentially and each
iteration depends on the previous perturbed state, the overall
mechanism forms an adaptive composition of the two Gaussian
mechanisms across iterations. To obtain a tighter bound on the
cumulative privacy loss, we use the moments accountant method
\citep{abadi2016deep}, which tracks and composes the log moments of
the privacy loss random variables.

\begin{theorem}[Moments Accountant Composition]
\label{thm:composition}
Under the assumptions of Lemmas~\ref{lem:sensitivity_Mg} and~\ref{lem:sensitivity_lambda}, suppose that at each
iteration \(t=1,\ldots,T\), client \(g\) applies the Gaussian
mechanisms of Theorem~\ref{thm:dp_client_update} to the model update \(M_g^t\) and the local
DRO variable \(\lambda_g^t\), with per-variable privacy budgets
\(\epsilon_M,\epsilon_\lambda>0\) and
\(\delta_M=\delta_\lambda=\delta\).
Then the adaptive composition of the two Gaussian mechanisms over
\(T\) iterations satisfies
\((\bar{\epsilon},\bar{\delta})\)-differential privacy, where
\[
\bar{\epsilon}
=
\frac{T\left(\epsilon_M^2+\epsilon_\lambda^2\right)}{4\ln(1.25/\delta)}
+
\sqrt{
\frac{T\left(\epsilon_M^2+\epsilon_\lambda^2\right)\ln(1/\bar{\delta})}
{\ln(1.25/\delta)}
}.
\]
\end{theorem}

\begin{proof}
The proof is given in Appendix~\ref{app:proof_comp_theorem}.
\end{proof}

Notice that the first term in \(\bar\epsilon\) grows linearly in \(T\), whereas the second grows only as \(\sqrt{T}\). For small per-round budgets, the quadratic factor \(\epsilon_M^2+\epsilon_\lambda^2\) makes the first term negligible, so the cumulative privacy loss is dominated by the \(\sqrt{T}\) term. In other words, the moments accountant yields a cumulative budget that grows roughly as \(\sqrt{T}\) rather than linearly in \(T\), which is precisely its advantage over composing the per-round guarantees directly.

\section{Convergence Analysis}
\label{sec:convergence}
In this section, we analyze the convergence of the proposed differentially private federated multiclass DRO algorithm. We begin with a single-iteration bound, which characterizes one round of the local and global updates and serves as the basis for the convergence result over all training rounds. To simplify notation, for each client \(g\), define the stacked variable $z_g := (M_g,\lambda_g).$ Correspondingly, let $\widetilde z_g^t := (\widetilde M_g^t, \widetilde\lambda_g^t)$ denote the perturbed iterate of Algorithm~\ref{alg:dp_fed_dro}. For each sample \(n = 1,\dots,N_g\), define the local robust loss
\[
\begin{aligned}
L_n(M_g,\lambda_g)
:=
\max\!\Big\{
&\ell_{\mathrm{CS}}(M_g;x_n,y_n),\\
&\max_{y\neq y_n}
\bigl(\ell_{\mathrm{CS}}(M_g;x_n,y)-\kappa\lambda_g\bigr)
\Big\}.
\end{aligned}
\]

Also define the local constraint-penalty term
\[
R_g(M_g,\lambda_g;\tau_g)
:=
\frac{1}{2\gamma}\left(
\max\{0,\tau_g+\gamma\bigl(\phi(M_g)-\lambda_g\bigr)\}^{2}
-\tau_g^{2}
\right).
\]

Then, the augmented non-consensus part of the client objective can be written compactly as
\[
h_g(z_g;\tau_g)
:=
\lambda_g r_g
+
\frac{1}{N_g}\sum_{n=1}^{N_g} L_n(M_g,\lambda_g)
+
R_g(M_g,\lambda_g;\tau_g).
\]

Let \(z_g^\star := (M^\star,\lambda_g^\star)\) denote an optimal point of the
exact constrained local DRO reformulation, which in particular satisfies
\(\lambda_g^\star\ge\phi(M^\star)\). Since \(\tau_g\ge 0\) along the trajectory of
Algorithm~\ref{alg:dp_fed_dro}, feasibility implies
$R_g(z_g^\star;\tau_g)\le 0$ for every such \(\tau_g\), with equality when the
constraint is active. Hence, at the comparison point, the augmented
non-consensus client objective is bounded above by the original
non-consensus client objective.

We evaluate convergence using two quantities. The term
\(h_g^{t}(\widetilde z_g^{\,t-1})-h_g^{t}(z_g^\star)\) measures the augmented
objective gap of the perturbed client iterate at round \(t\), while
\(\|\bar{\widetilde M}_g^{\,T}-\bar M^{\,T}\|_F\), with
\(\bar{\widetilde M}_g^{\,T}:=\frac{1}{T}\sum_{t=1}^{T}\widetilde M_g^{\,t}\)
and \(\bar M^{\,T}:=\frac{1}{T}\sum_{t=1}^{T}M^{\,t}\), measures the violation
of the consensus constraint between the averaged local and global models.
Accordingly, the convergence result is stated in terms of the combined metric
\begin{align*}
\frac{1}{T}\sum_{t=1}^{T}\sum_{g=1}^{G}\alpha_g
\bigl(h_g^{t}(\widetilde z_g^{\,t-1})-h_g^{t}(z_g^\star)\bigr)
+\beta\sum_{g=1}^{G}\alpha_g
\|\bar{\widetilde M}_g^{\,T}-\bar M^{\,T}\|_F .
\end{align*}
Since \(h_g\) is convex, though possibly nonsmooth, in \(z_g\) for fixed
\(\tau_g\), the following lemma establishes the key one-step inequality used
to derive this convergence result.

\begin{lemma}[Single-iteration bound]\label{lem:single_iter}
Assume that, for each client $g$, $h_g^{t}$ is convex and that the proximal
step sizes satisfy $\eta_g>0$ and $\nu_g>0$. Then, for any
$z_g=(M_g,\lambda_g)$ and $\mu_g$, the iterates generated by
Algorithm~\ref{alg:dp_fed_dro} satisfy
\begin{align}
&\sum_{g=1}^G \alpha_g \Big( h_g^{t}(\widetilde z_g^{\,t-1})-h_g^{t}(z_g) + \rho\big\langle \widetilde M_g^t-M_g,\ \mu_g^t\big\rangle_F + \rho\big\langle \mu_g-\mu_g^t,\ \widetilde M_g^t-M^t\big\rangle_F \Big) \notag\\
&\le \sum_{g=1}^G \alpha_g \Bigg[ \frac{\eta_g}{2} \left\| \big[h_g^{t\prime}(\widetilde z_g^{\,t-1})\big]_M - \left(\rho+\frac{1}{\eta_g}\right)\xi_g^t \right\|_F^2 + \frac{\nu_g}{2} \left| \big[h_g^{t\prime}(\widetilde z_g^{\,t-1})\big]_\lambda - \frac{1}{\nu_g}\,\zeta_g^t \right|^2 \notag\\
&\qquad - \left(\rho+\frac{1}{\eta_g}\right) \big\langle \xi_g^t,\ M_g-\widetilde M_g^{\,t-1}\big\rangle_F - \frac{1}{\nu_g}\,\zeta_g^t \big(\lambda_g-\widetilde\lambda_g^{\,t-1}\big) \notag\\
&\qquad + \frac{1}{2\eta_g} \Big( \|M_g-\widetilde M_g^{\,t-1}\|_F^2 - \|M_g-\widetilde M_g^t\|_F^2 \Big) + \frac{1}{2\nu_g} \Big( (\lambda_g-\widetilde\lambda_g^{\,t-1})^2 - (\lambda_g-\widetilde\lambda_g^t)^2 \Big) \notag\\
&\qquad + \frac{\rho}{2} \Big( \|M_g-M^{t-1}\|_F^2 - \|M_g-M^{t}\|_F^2 \Big) + \frac{\rho}{2} \Big( \|\mu_g-\mu_g^{t-1}\|_F^2 - \|\mu_g-\mu_g^{t}\|_F^2 \Big) \Bigg].
\label{eq:lemma3}
\end{align}
\end{lemma}
\begin{proof}
The proof is given in Appendix~\ref{app:proof_lem_single_iter}.
\end{proof}

Now, we can leverage Lemma~\ref{lem:single_iter} to obtain the final convergence theorem for Algorithm~\ref{alg:dp_fed_dro}.

\begin{theorem}[Convergence bound]
\label{thm:ergodic_convergence_dp_fed_dro}
Assume that, for each client \(g=1,\dots,G\) and round \(t=1,\dots,T\), the local objective \(h_g^{t}\) is convex in \(z_g\), and that, over the finite training horizon \(T\), its subgradients along the trajectory of Algorithm~\ref{alg:dp_fed_dro} are bounded as \(\|\partial_{M_g}(h_g-R_g)\|_F\le\sqrt{2}\,L\), \(|\partial_{\lambda_g}(h_g-R_g)|\le S:=\max\{r_g,\kappa-r_g\}\), \(\|\partial_{M_g}R_g\|_F\le\sqrt{2}\,C\), and \(|\partial_{\lambda_g}R_g|\le V\), for constants \(L,C,V>0\). Let \(c_M,c_{\lambda,g}>0\) satisfy \(\|M^\star\|_F\le c_M\) and \(|\lambda_g^\star|\le c_{\lambda,g}\) for all \(g\), and let the step sizes be chosen as
\begin{equation}
\eta_g := \frac{c_M}{\sqrt{T\left(2(L+C)^2+\frac{16KPL^2\ln(1.25/\delta_M)}{N_g^2\epsilon_M^2}\right)}},
\qquad
\nu_g := \frac{c_{\lambda,g}}{\sqrt{T\left((S+V)^2+\frac{2\kappa^2\ln(1.25/\delta_\lambda)}{N_g^2\epsilon_\lambda^2}\right)}}.
\label{eq:eta_nu_choice_final}
\end{equation}
Define
\begin{equation}
M_{1,T} := c_M \sum_{g=1}^G \alpha_g \Bigl( 2(L+C)^2 + \frac{16KPL^2\ln(1.25/\delta_M)}{N_g^2\epsilon_M^2} \Bigr)^{1/2} + \sum_{g=1}^G \alpha_g\, c_{\lambda,g} \Bigl( (S+V)^2 + \frac{2\kappa^2\ln(1.25/\delta_\lambda)}{N_g^2\epsilon_\lambda^2} \Bigr)^{1/2},
\label{eq:M1_final}
\end{equation}
and
\begin{equation}
M_2 := \frac{\rho}{2}c_M^2 + \frac{\beta^2}{2\rho}.
\label{eq:M2_final}
\end{equation}
Then, for any \(T\ge 1\) and any \(\beta>0\), we have
\begin{equation}
\mathbb{E}\Biggl[ \frac{1}{T}\sum_{t=1}^{T}\sum_{g=1}^G \alpha_g \Bigl( h_g^{t}(\widetilde z_g^{\,t-1})-h_g^{t}(z_g^\star) \Bigr) + \beta\sum_{g=1}^G \alpha_g \|\bar{\widetilde M}_g^{\,T}-\bar M^{\,T}\|_F \Biggr] \le \frac{M_{1,T}}{\sqrt{T}} + \frac{M_2}{T}.
\label{eq:final_theorem_fixed}
\end{equation}
\end{theorem}

\begin{proof}
The proof is given in Appendix~\ref{app:proof_thm3}.
\end{proof}

Theorem~\ref{thm:ergodic_convergence_dp_fed_dro} establishes a finite-horizon
convergence bound for the proposed differentially private federated
multiclass DRO algorithm. For the prescribed training horizon, the leading term
scales as \(O(M_{1,T}/\sqrt{T})\), with an additional lower-order term of order
\(O(T^{-1})\). The result also makes explicit the inherent privacy–utility tradeof: stronger privacy guarantees, induced by smaller values of \(\epsilon_M,\epsilon_\lambda\) and \(\delta_M,\delta_\lambda\), require injecting larger Gaussian noise, which increases the constants in the convergence bound and may reduce utility. Notice also that the privacy-dependent terms in \eqref{eq:M1_final} decrease as the local sample size \(N_g\) grows, so the utility cost of privacy diminishes for clients that hold more data.

\begin{figure*}[t]
    \centering

    \begin{minipage}[t]{0.192\textwidth}
        \centering
        \includegraphics[width=\linewidth,trim=4 6 4 6,clip]{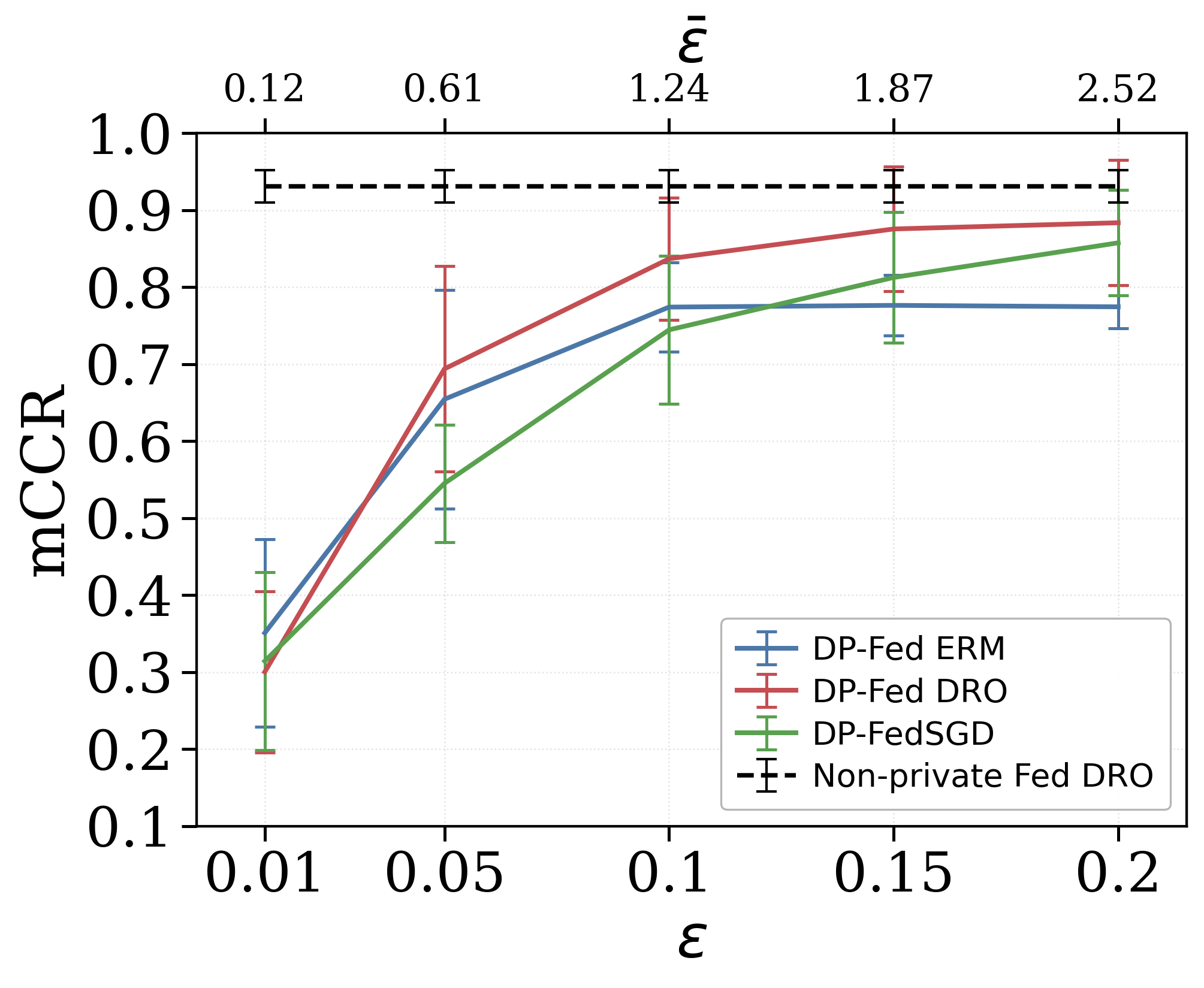}
        \vspace{2pt}
        \footnotesize $\delta = 1 \times 10^{-8}$
    \end{minipage}\hfill
    \begin{minipage}[t]{0.192\textwidth}
        \centering
        \includegraphics[width=\linewidth,trim=4 6 4 6,clip]{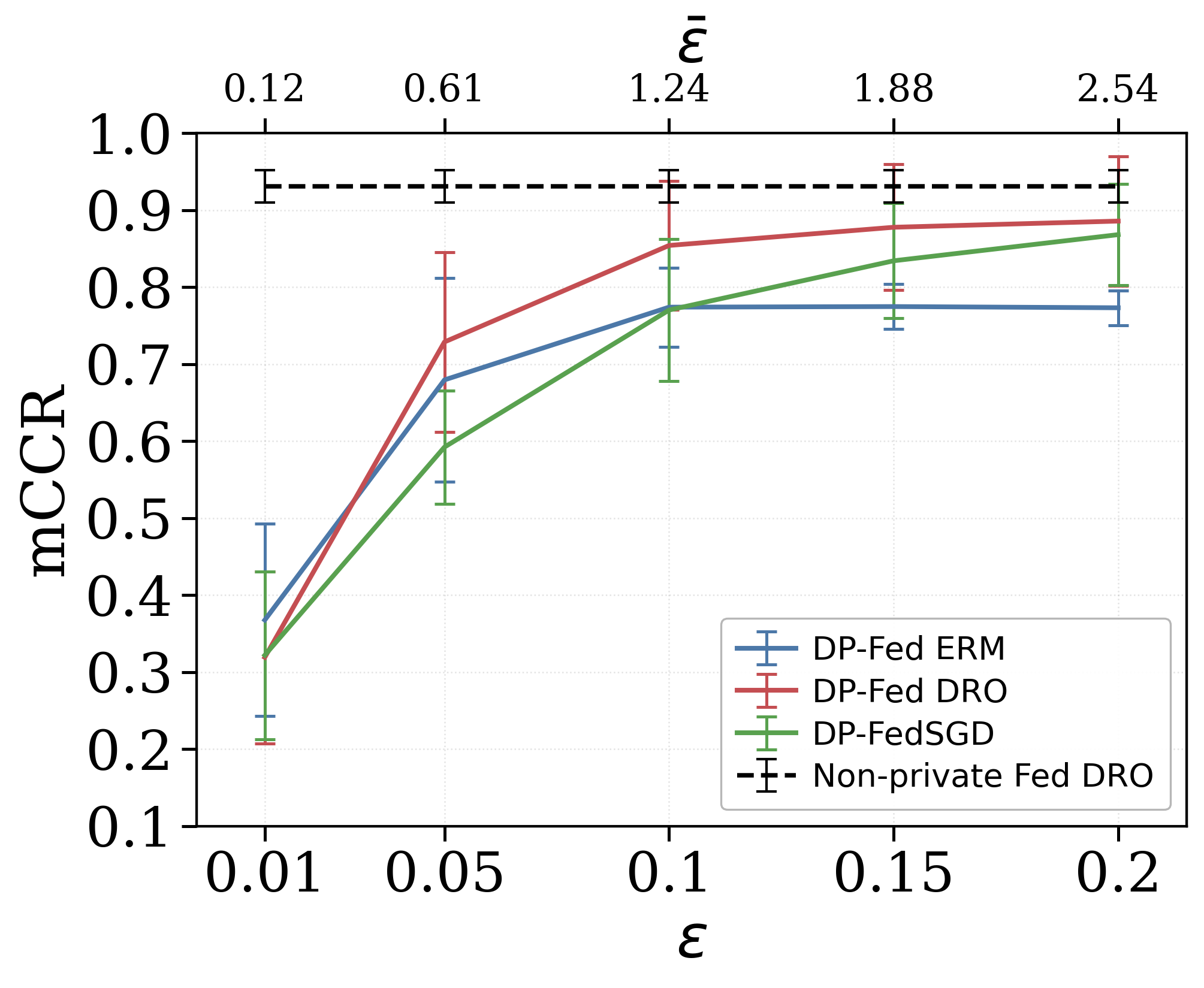}
        \vspace{2pt}
        \footnotesize $\delta = 1 \times 10^{-6}$
    \end{minipage}\hfill
    \begin{minipage}[t]{0.192\textwidth}
        \centering
        \includegraphics[width=\linewidth,trim=4 6 4 6,clip]{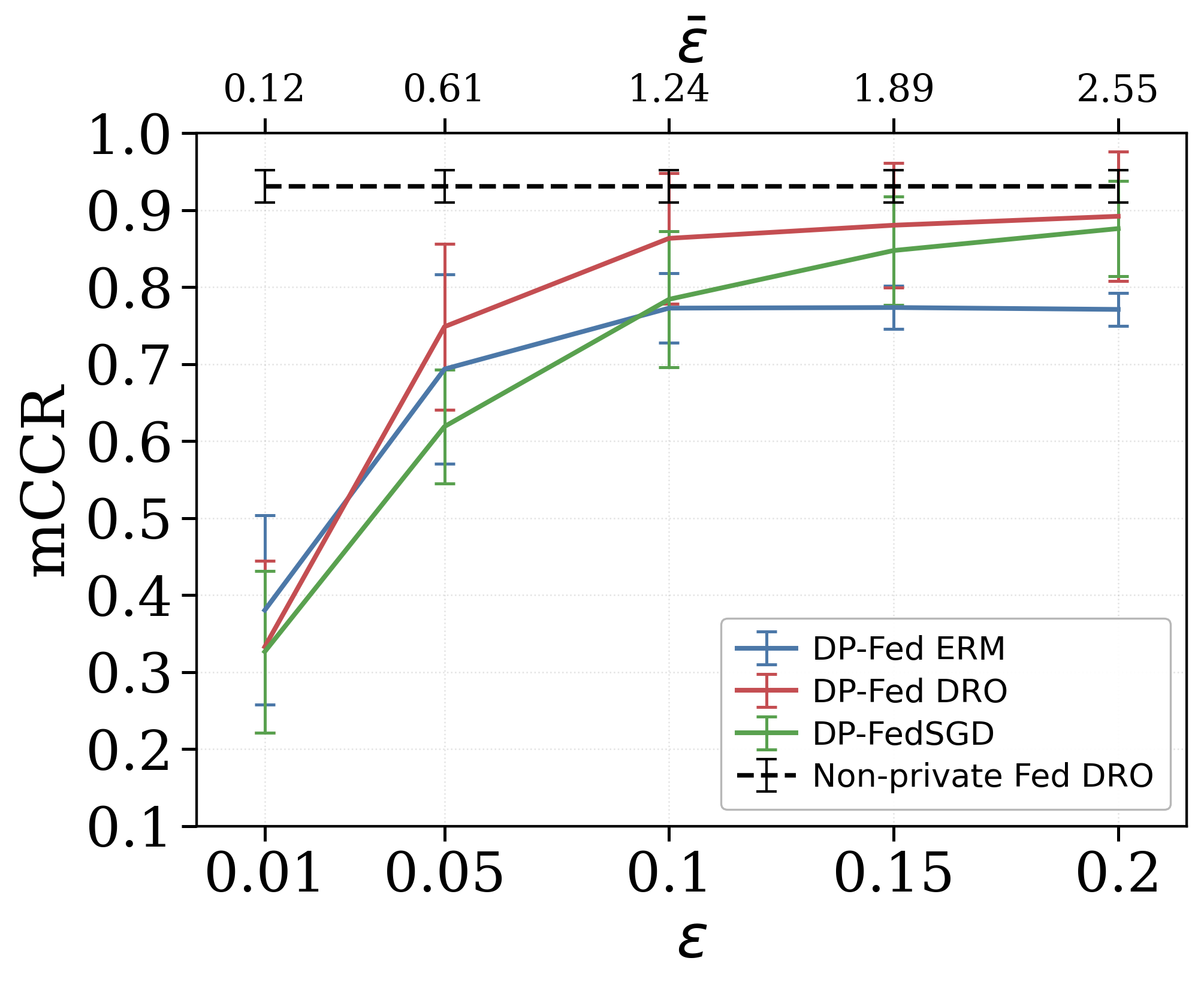}
        \vspace{2pt}
        \footnotesize $\delta = 1 \times 10^{-5}$
    \end{minipage}\hfill
    \begin{minipage}[t]{0.192\textwidth}
        \centering
        \includegraphics[width=\linewidth,trim=4 6 4 6,clip]{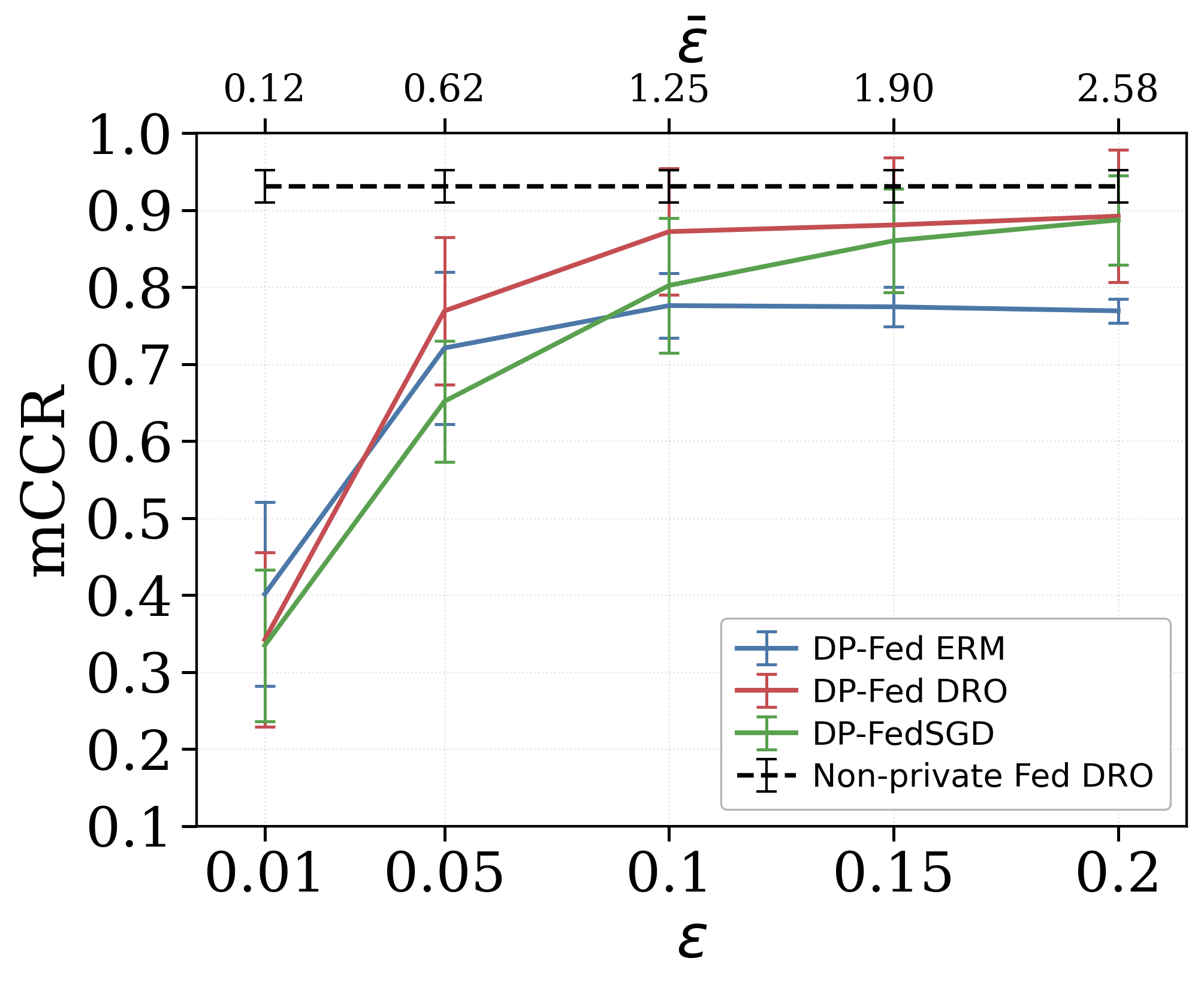}
        \vspace{2pt}
        \footnotesize $\delta = 1 \times 10^{-4}$
    \end{minipage}\hfill
    \begin{minipage}[t]{0.192\textwidth}
        \centering
        \includegraphics[width=\linewidth,trim=4 6 4 6,clip]{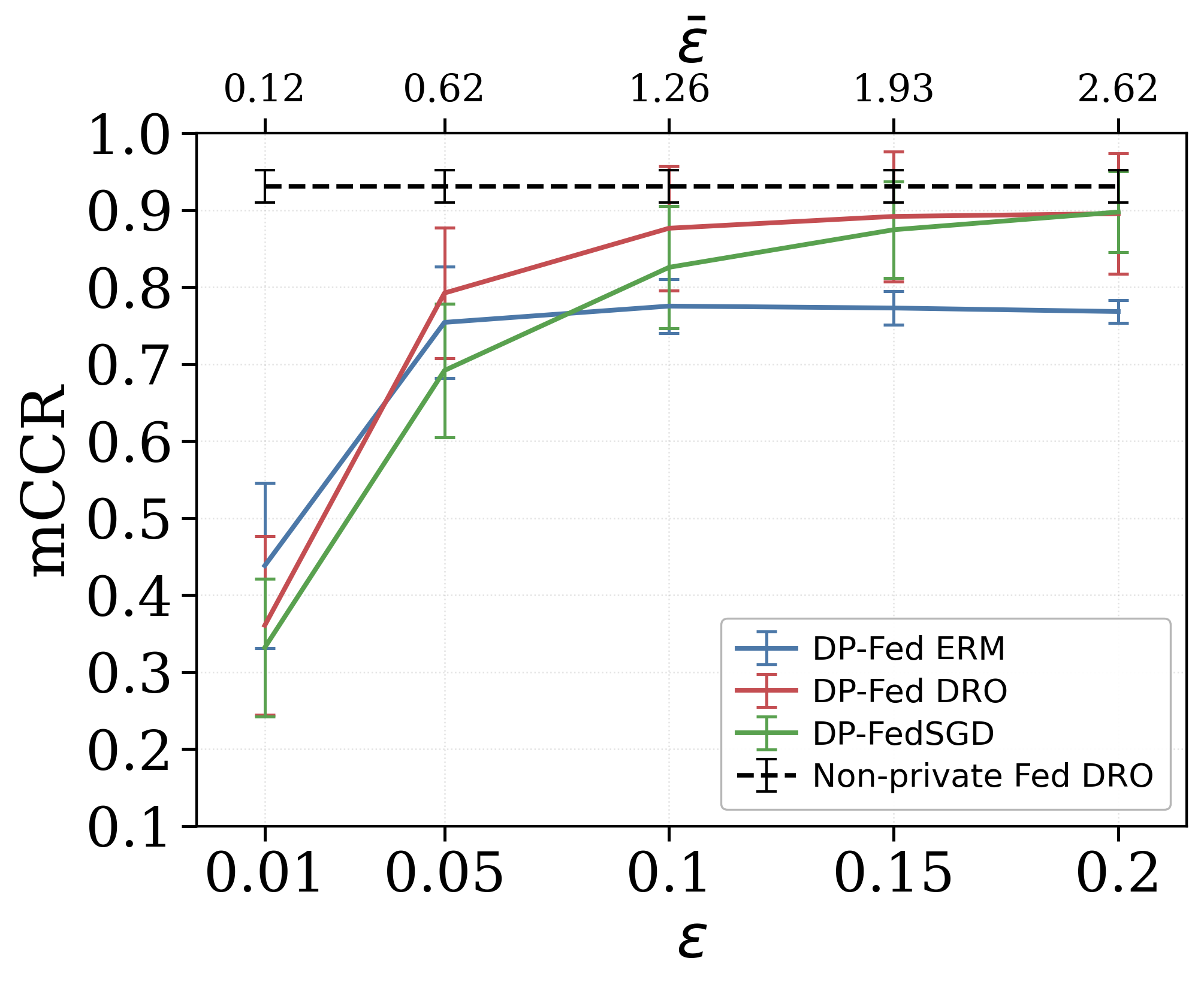}
        \vspace{2pt}
        \footnotesize $\delta = 1 \times 10^{-3}$
    \end{minipage}

    \caption{mCCR of the proposed DP-Fed DRO method as a function of
    \(\epsilon\) across different values of \(\delta\), compared with
    DP-Fed ERM, DP-Fed SGD, and non-private Fed DRO.}
    \label{fig:privacy_tradeoff}
\end{figure*}

\begin{figure}[t]
    \centering
    \begin{minipage}[t]{0.24\textwidth}
        \centering
        \includegraphics[width=\linewidth]{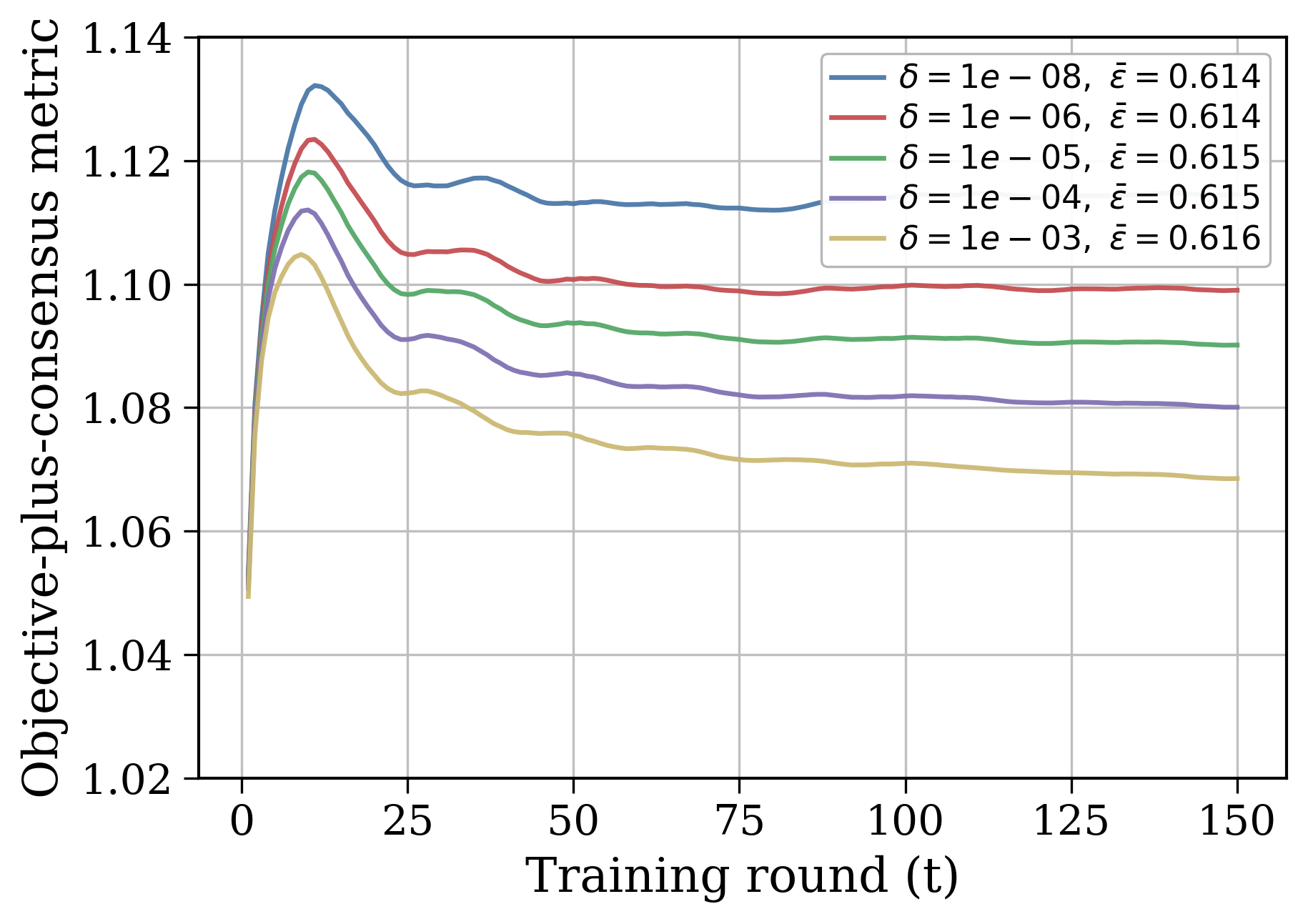}
        
        \vspace{2pt}
        \footnotesize $\epsilon = \text{0.05}$
    \end{minipage}
    \hfill
    \begin{minipage}[t]{0.24\textwidth}
        \centering
        \includegraphics[width=\linewidth]{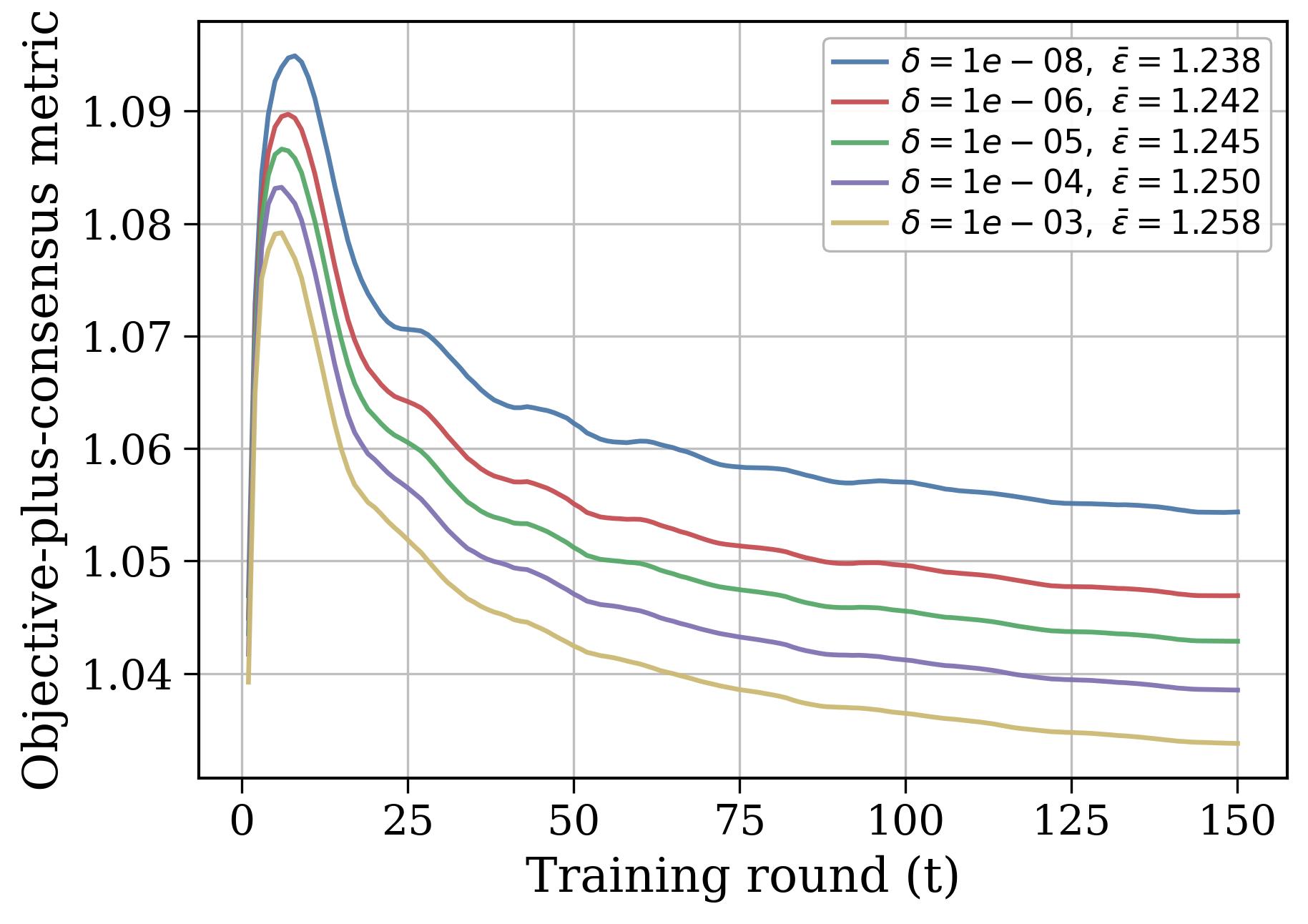}
        
        \vspace{2pt}
        \footnotesize $\epsilon = \text{0.1}$
    \end{minipage}
    \hfill
    \begin{minipage}[t]{0.24\textwidth}
        \centering
        \includegraphics[width=\linewidth]{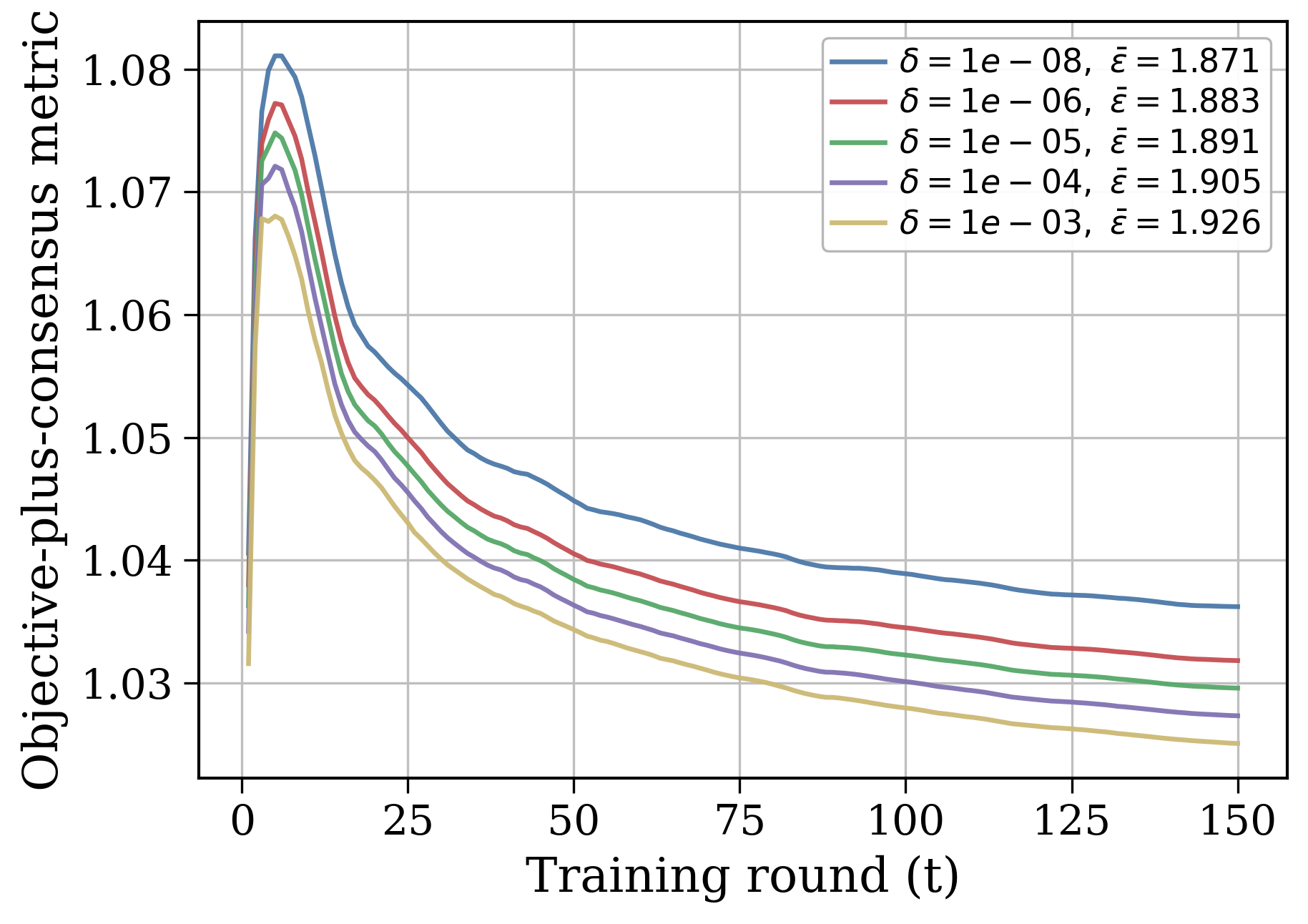}
        
        \vspace{2pt}
        \footnotesize $\epsilon = \text{0.15}$
    \end{minipage}
    \hfill
    \begin{minipage}[t]{0.24\textwidth}
        \centering
        \includegraphics[width=\linewidth]{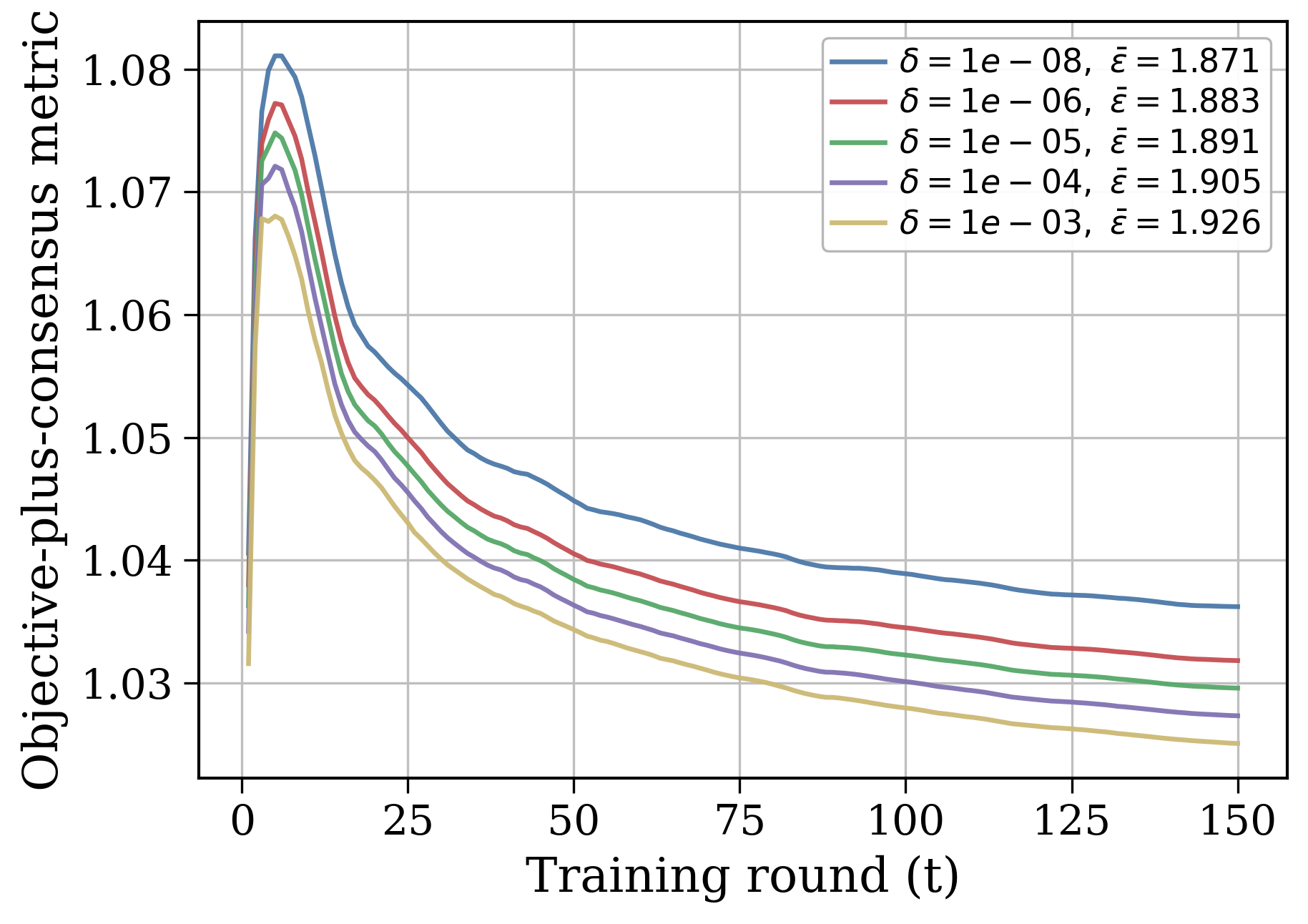}
        
        \vspace{2pt}
        \footnotesize $\epsilon = \text{0.2}$
    \end{minipage}
    \caption{Convergence of the DP-Fed DRO objective and consensus metric over training rounds for different per-round $\epsilon$ and $\delta$ values.}
    \label{fig:epsilon_convergence}
\end{figure}

\section{Experimental Evaluation}

\label{sec:experiments}

In this section, we evaluate the performance, sensitivity, and scalability of the proposed differentially private federated multiclass DRO algorithm. We first study the sensitivity of the algorithm to the privacy parameters and its convergence behavior on simulated industrial fault diagnosis data. We next benchmark the method against private and non-private baselines on five UCI multiclass datasets. Finally, we examine scalability with respect to the number of clients and the number of classes. Unless otherwise stated, each experiment is run for $T = 150$ training rounds and repeated over 10 random seeds, each redrawing the data partition and all injected noise. We report the mean performance across seeds, and all error bars represent one standard deviation. We use uniform client weights, that is, $\alpha_g = 1/G$ for all $g=1,\dots,G$, and all clients share the same local hyperparameter settings. Across all experiments, we compare the proposed method against several relevant baselines. These include a differentially private federated ERM model, denoted DP-Fed ERM, based on~\citet{huang2019dp} and instantiated with the Crammer--Singer multiclass hinge loss and the Frobenius regularizer $R(M)=\frac{1}{2}\|M\|_F^2$, as well as a federated adaptation of differentially private stochastic gradient descent~\citep{abadi2016deep}, denoted DP-Fed SGD. We also include the non-private federated multiclass DRO model described in Algorithm~\ref{alg:fed_dro} in the sensitivity analysis. In the benchmarking experiments, we further compare against an OvA federated DRO formulation, which trains independent binary DRO classifiers using the binary hinge loss, and a centralized non-private benchmark.

\subsection{Simulated Industrial Fault Diagnosis Data Experiments}
\label{sec:pump}

We begin by evaluating the proposed method on simulated industrial fault diagnosis data generated from degrading pumps. In the first experiment, we study the sensitivity of the algorithm to the per-round differential privacy hyperparameters $\epsilon$ and $\delta$ through the mean correct classification rate (mCCR). The per-round budget is split as $\epsilon_M^2=0.95\,\epsilon^2$ and $\epsilon_\lambda^2=0.05\,\epsilon^2$, with $\delta_M=\delta_\lambda=\delta$. The corresponding cumulative budget $\bar{\epsilon}$ from Theorem~\ref{thm:composition}, with $\bar{\delta}=\delta$, is shown on the top axis of Figure.~\ref{fig:privacy_tradeoff}. The second experiment examines how these hyperparameters affect convergence. At each round $t$, we track $\frac{1}{t}\sum_{s=1}^{t}\sum_{g=1}^{G}\alpha_{g}h_{g}^{s}(\widetilde{z}_{g}^{\,s-1})+\beta\sum_{g=1}^{G}\alpha_{g}\|\bar{\widetilde{M}}_{g}^{\,t}-\bar{M}^{\,t}\|_{F}$, the left-hand side of Theorem~\ref{thm:ergodic_convergence_dp_fed_dro} without the benchmark term $h_g^{s}(z_g^\star)$. This metric captures both the augmented objective at the privatized iterates and the violation of the consensus constraint.

\textbf{Dataset.} We use simulation-generated data from a physics-based pump model \citep{mathworks_pump_fault_example}: one healthy class and three fault classes (cylinder leaks, locked inlet, increased bearing friction), so \(K=4\). We set \(G=4\), \(P=13\), \(N=1600\) training samples, and \(N_{\text{Test}}=400\). Each client holds all four classes but most of its samples come from a single dominant class, and all clients are equally sized. Gaussian noise is added to both the training and test features, twice as strongly on the training side, and $15\%$ of training labels are flipped at random. Finally, a constant offset is applied to five of the test features, displacing the deployment distribution. This is the type of shift a Wasserstein ambiguity set is designed to hedge against. All perturbations are applied before preprocessing, and this single setting is used in all pump experiments. We then apply min-max normalization to each feature column, with statistics computed on the training data, so that training values lie in \([0,1]\), followed by per-sample \(\ell_2\) normalization so that each sample has \(\ell_2\)-norm at most \(1\).

\textbf{Hyperparameters.} We set \(\rho=0.8\) for the ADMM-based models and \(\gamma=0.1\) for the norm-constraint penalty. The label transportation cost and Wasserstein ball radius are \(\kappa=0.5\) and \(r=0.1\) for DP-Fed DRO, and \(\kappa=2\) and \(r=10^{-2}\) for the non-private Fed DRO. DP-Fed ERM uses \(\lambda=10^{-4}\), and DP-Fed SGD uses \(\eta_k=8/\sqrt{k}\) and \(\lambda=10^{-6}\). We vary \(\epsilon \in \{0.01,0.05,0.1,0.15,0.2\}\) and \(\delta \in \{10^{-8},10^{-6},10^{-5},10^{-4},10^{-3}\}\), and set \(\beta=0.4\) in the convergence experiment. Throughout we use the \(\ell_{\infty}\)-norm. The step-size constants are assumed to be obtained a priori from a proxy dataset drawn from the same generative process: \(c_w=\|M^\star\|_F=7.56\) from a non-private solve for DP-Fed ERM, and \(c_M=2\), \(c_{\lambda,g}=6\) by cross-validation for DP-Fed DRO, which satisfy the bounds required in Theorem~\ref{thm:ergodic_convergence_dp_fed_dro}. Preprocessing gives \(L=1\), and we set \(C=0.4\) and \(V=1\), so that \(S=\max\{r,\kappa-r\}=0.4\).

\textbf{Results.} For this dataset, we examine the privacy--utility tradeoff and convergence behavior. Figure~\ref{fig:privacy_tradeoff} shows the effect of the privacy hyperparameters on model performance, with the cumulative budget $\bar{\epsilon}$ on the top axis. At the smallest per-round budget, $\epsilon=0.01$, the injected noise dominates the signal and all private models perform poorly with large variability across seeds. As $\epsilon$ increases, all three private models improve, since larger values of $\epsilon$ correspond to weaker privacy guarantees and lower injected noise, and DP-Fed DRO and DP-Fed SGD approach the non-private Fed DRO model. DP-Fed ERM, in contrast, saturates early and does not benefit from larger budgets. DP-Fed DRO is the best private method at every budget where the models are informative, with the clearest margin at moderate budgets, $\epsilon\in\{0.05,0.1\}$, where the noise is significant but does not overwhelm the signal. DP-Fed DRO is also markedly more robust to the choice of $\delta$ than DP-Fed SGD, whose performance degrades noticeably as $\delta$ tightens. As a result, the gap between the two is largest at the tightest values of $\delta$. This reflects the ability of the robust formulation to absorb the larger noise level induced by a smaller $\delta$ with little loss in performance.

The convergence behavior of DP-Fed DRO is shown in Figure~\ref{fig:epsilon_convergence} for the privacy budgets of practical interest, $\epsilon\in\{0.05,0.1,0.15,0.2\}$. In all cases, the metric rises during the first few rounds and then decreases steadily for the remainder of training. As $\epsilon$ increases, the curves settle at lower values and the initial peak becomes smaller. This is expected, since larger values of $\epsilon$ require less injected noise, which makes consensus easier to reach. Similarly, for a fixed $\epsilon$, smaller values of $\delta$ settle slightly higher, since they require more noise. This effect of $\delta$ is most visible at the smallest budget and fades as $\epsilon$ grows, with the curves nearly coinciding at $\epsilon=0.2$. This agrees with Theorem~\ref{thm:ergodic_convergence_dp_fed_dro}, where \(\delta\) only affects the noise term of the bound, which becomes small as \(\epsilon\) increases. Overall, the curves exhibit the same decreasing trend across all tested settings, indicating that the method retains stable convergence behavior under all considered privacy levels.

\begin{table}[t]
\caption{Comparison of Macro-F1 scores across five multiclass benchmark datasets, namely Dermatology~\citep{uci_dermatology}, Iris~\citep{uci_iris}, Wine~\citep{uci_wine}, Palmer Penguins~\citep{palmer_penguins_data}, and Seeds~\citep{uci_seeds}.}
\label{tab:uci_main_results}
\centering
\setlength{\tabcolsep}{3.5pt}
\begin{tabular}{lccccc}
\toprule
Model & Dermatology & Iris & Wine & Penguins & Seeds \\
\midrule
Centralized DRO 
& $0.905 \pm 0.070$ 
& $0.919 \pm 0.051$ 
& $0.943 \pm 0.032$ 
& $0.996 \pm 0.008$ 
& $0.969 \pm 0.023$ \\

\addlinespace[0.25em]
\cdashline{1-6}
\addlinespace[0.35em]

DP-Fed DRO  
& $0.505 \pm 0.103$ 
& $\mathbf{0.564 \pm 0.077}$ 
& $\mathbf{0.640 \pm 0.148}$ 
& $0.689 \pm 0.142$ 
& $\mathbf{0.733 \pm 0.144}$ \\

DP-Fed ERM  
& $0.527 \pm 0.104$ 
& $0.519 \pm 0.146$ 
& $0.590 \pm 0.113$ 
& $\mathbf{0.774 \pm 0.113}$ 
& $0.620 \pm 0.111$ \\

DP-Fed SGD
& $\mathbf{0.599 \pm 0.131}$
& $0.559 \pm 0.065$
& $0.606 \pm 0.121$
& $0.637 \pm 0.141$
& $0.690 \pm 0.145$ \\

\addlinespace[0.25em]
\cdashline{1-6}
\addlinespace[0.35em]

Non-Private Fed DRO 
& $0.919 \pm 0.044$ 
& $\mathbf{0.943 \pm 0.042}$ 
& $0.924 \pm 0.045$ 
& $\mathbf{0.993 \pm 0.013}$ 
& $\mathbf{0.981 \pm 0.015}$ \\

Non-Private Fed OvA DRO
& $\mathbf{0.957 \pm 0.017}$ 
& $0.862 \pm 0.082$ 
& $\mathbf{0.932 \pm 0.038}$ 
& $0.992 \pm 0.010$ 
& $0.979 \pm 0.014$ \\
\bottomrule
\end{tabular}
\end{table}

\subsection{UCI Data Experiment}

In this benchmarking experiment, we use the Macro-F1 score as the primary evaluation metric, since it weights all classes equally and is therefore better suited to the class-imbalanced client splits used in this experiment.

\textbf{Datasets.} We use five popular multiclass benchmark datasets: Dermatology with 6 classes, and Iris, Wine, Palmer Penguins, and Seeds with 3 classes each. For all datasets, we use an 80\%/20\% train-test split and set \(G=4\) clients. Features are min-max scaled using statistics computed on the training split, and each sample is normalized to have \(\ell_2\)-norm at most \(1\). The training samples are distributed across clients using a controlled class-imbalanced split, where each class is unevenly allocated across clients while ensuring that all clients retain samples from every class. To account for label uncertainty, 10\% of the training labels are flipped at random, with the test labels left clean; all results in Table~\ref{tab:uci_main_results} are reported in this setting.

\textbf{Hyperparameters.}
We fix the differential privacy parameters to $\epsilon=0.5$ and $\delta=10^{-6}$ across all datasets, with the per-round budget split as in Section~\ref{sec:pump}, and train for $T=50$ rounds, which corresponds to a cumulative budget of $\bar{\epsilon}\approx3.7$. All remaining hyperparameters are selected by 5-fold cross-validation on the training data. For the centralized and non-private DRO baselines, we select the label transportation cost $\kappa$ and the local Wasserstein ball radius $r$ over $\kappa \in \{0.25,0.5,0.75,1.0\}$ and $r \in \{10^{-5},10^{-4},10^{-3},10^{-2}\}$. For DP-Fed DRO these are selected over $\kappa \in \{0.5, 1, 2\}$ and seven values of $r$ spanning $0$ to $0.06$, jointly with the step-size constants $c_M \in \{0.5, 1, 2\}$ and $c_{\lambda} \in \{0.5, 1, 2, 6\}$. The regularization parameter for the DP-Fed ERM baseline is selected over $\lambda \in \{10^{-8},10^{-6},10^{-4}\}$, with its step-size constant $c_w$ derived from the norm of the non-private solution. For the DP-Fed SGD baseline, the initial learning rate is selected over $\eta \in \{0.05,0.1,0.2\}$, with the decaying step size $\eta_k=\eta/\sqrt{k}$, while the regularization parameter is fixed at $\lambda=10^{-6}$. We use the $\ell_{\infty}$-norm throughout and set $\gamma=0.1$ for the DRO penalty term. The remaining constants are fixed at $C=0.4$, $V=1$ and $L=1$, and $S$ follows from the selected pair as $S=\max\{r,\kappa-r\}$.

\textbf{Results.} Table~\ref{tab:uci_main_results} summarizes the Macro-F1 scores on the five multiclass benchmark datasets. Among the private methods, DP-Fed DRO achieves the best performance on Iris, Wine, and Seeds, with the clearest margins on Wine and Seeds, and ranks second on Palmer Penguins, giving it the strongest performance overall across the five datasets. DP-Fed ERM performs best on Palmer Penguins and DP-Fed SGD on Dermatology, so no single private baseline is consistently competitive, whereas DP-Fed DRO remains among the top private methods on most datasets. Notably, the three datasets on which DP-Fed DRO performs best are also the smallest, where each client holds the fewest samples. Since the sensitivity of the local updates scales inversely with the local sample size, these are the settings with the largest injected noise, suggesting that the robust formulation is most beneficial when the privacy-induced perturbations are most severe. Compared with the non-private models, privacy incurs a substantial cost on all datasets, most notably on Iris and Dermatology. This is expected under such strict privacy conditions, where each client holds only a few dozen samples and therefore requires large noise to protect each one. The non-private results also show that the native multiclass and OvA DRO formulations have mixed behavior across real datasets, with native multiclass Fed DRO performing better on Iris, OvA Fed DRO performing better on Dermatology, and both performing similarly on the remaining datasets. We therefore use a controlled synthetic study, presented in Appendix~\ref{app:synthetic_overlap}, to identify when the native multiclass formulation is most useful relative to OvA. The results show that its advantage is most evident when related fault classes have correlated features and partial overlap, so that the classes are difficult to separate but not completely inseparable.

\subsection{Scalability Experiments}

This experiment evaluates the scalability of the proposed algorithm with respect to both the number of clients and the number of classes. For client scalability, we consider two regimes. In the first, the number of clients increases while the total sample size is fixed, so that each client has fewer local samples as the federation grows. In the second, the number of clients and the total sample size increase together, so that the number of samples per client remains fixed. These two settings allow us to distinguish the effect of adding more clients from the effect of reducing the number of samples assigned to each client. To compare all models under a common stopping rule, we measure scalability using the runtime required to attain the peak mCCR within a prescribed tolerance.

\textbf{Dataset.}
To control the data generation process, we use the \texttt{make\_classification} function from the Scikit-learn Python package~\citep{pedregosa2011scikit}. Across all scalability experiments, we use an 80/20 train-test split. In the client scalability experiments, we fix \(K=4\) classes and \(P=10\) features, and vary the number of clients as \(G \in \{5,10,15,20,25,30\}\). The fixed-total-sample setting uses \(N=3000\) samples in total, while the fixed-local-sample setting uses 100 samples per client, in both cases before the train-test split. In the class scalability experiment, we fix \(G=4\) clients and \(P=10\) features, and vary the number of classes as \(K \in \{3,5,7,9\}\). Each class is generated from a Gaussian distribution with standard deviation \(1\), with class means located at the vertices of a \(P\)-dimensional hypercube with side length \(5\). We apply the same min-max and \(\ell_2\) normalization as in Section~\ref{sec:pump}. The samples are balanced across classes, and no label flipping is introduced.

\begin{figure*}[t]
    \centering

    \begin{tabular}{m{0.27\textwidth} m{0.27\textwidth} m{0.27\textwidth} m{0.12\textwidth}}
        \centering
        \includegraphics[width=\linewidth]{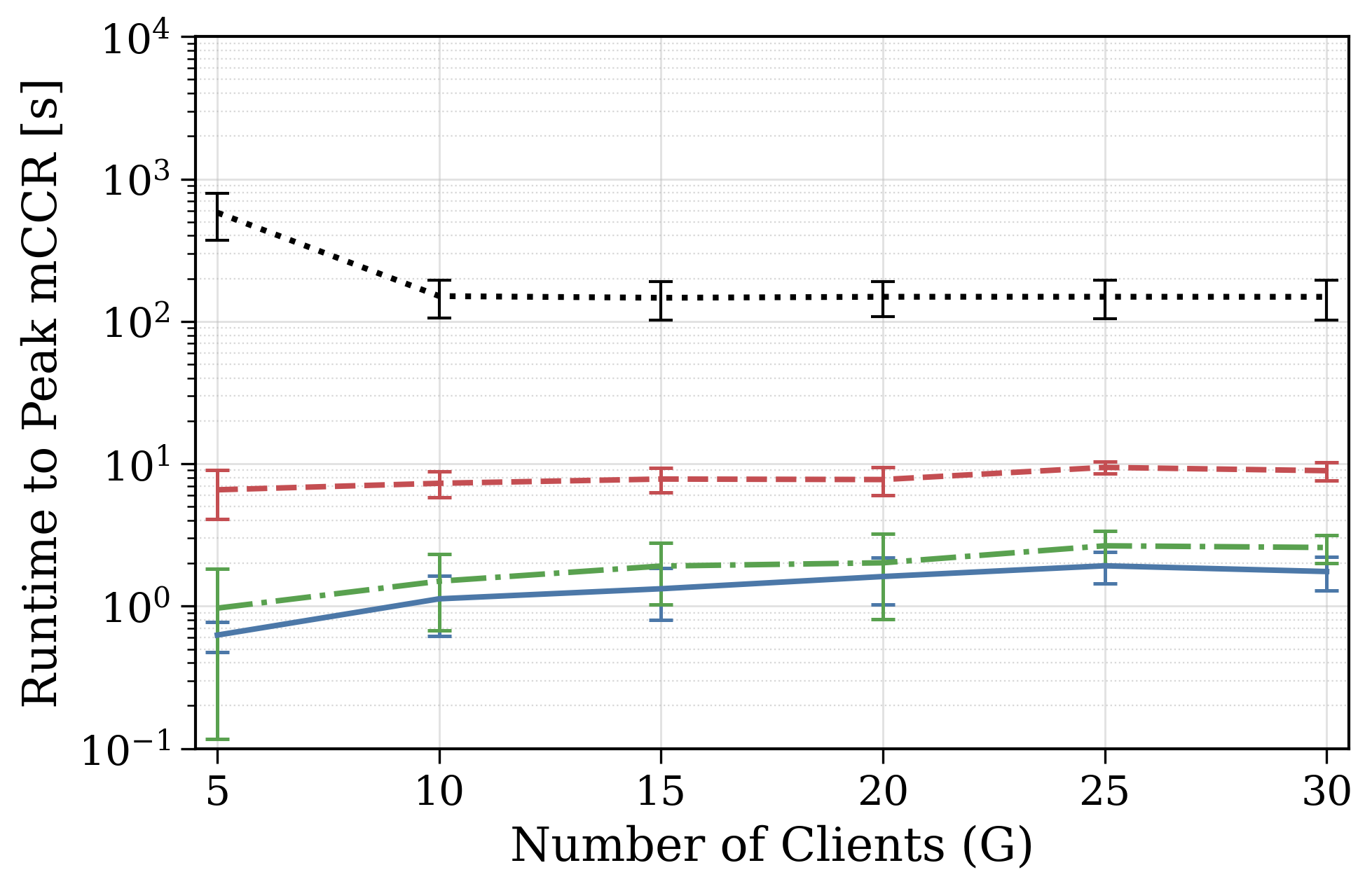}
        &
        \centering
        \includegraphics[width=\linewidth]{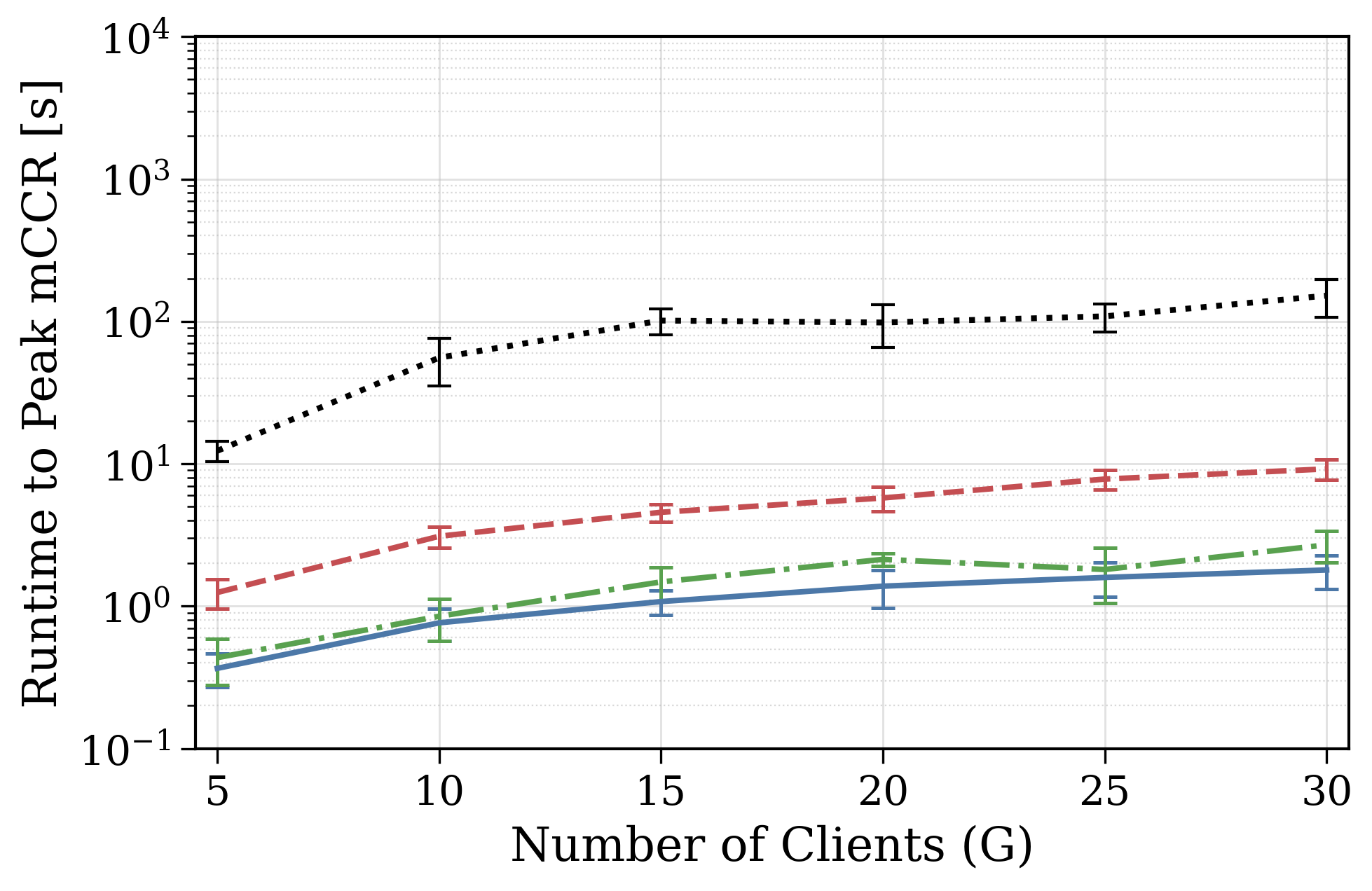}
        &
        \centering
        \includegraphics[width=\linewidth]{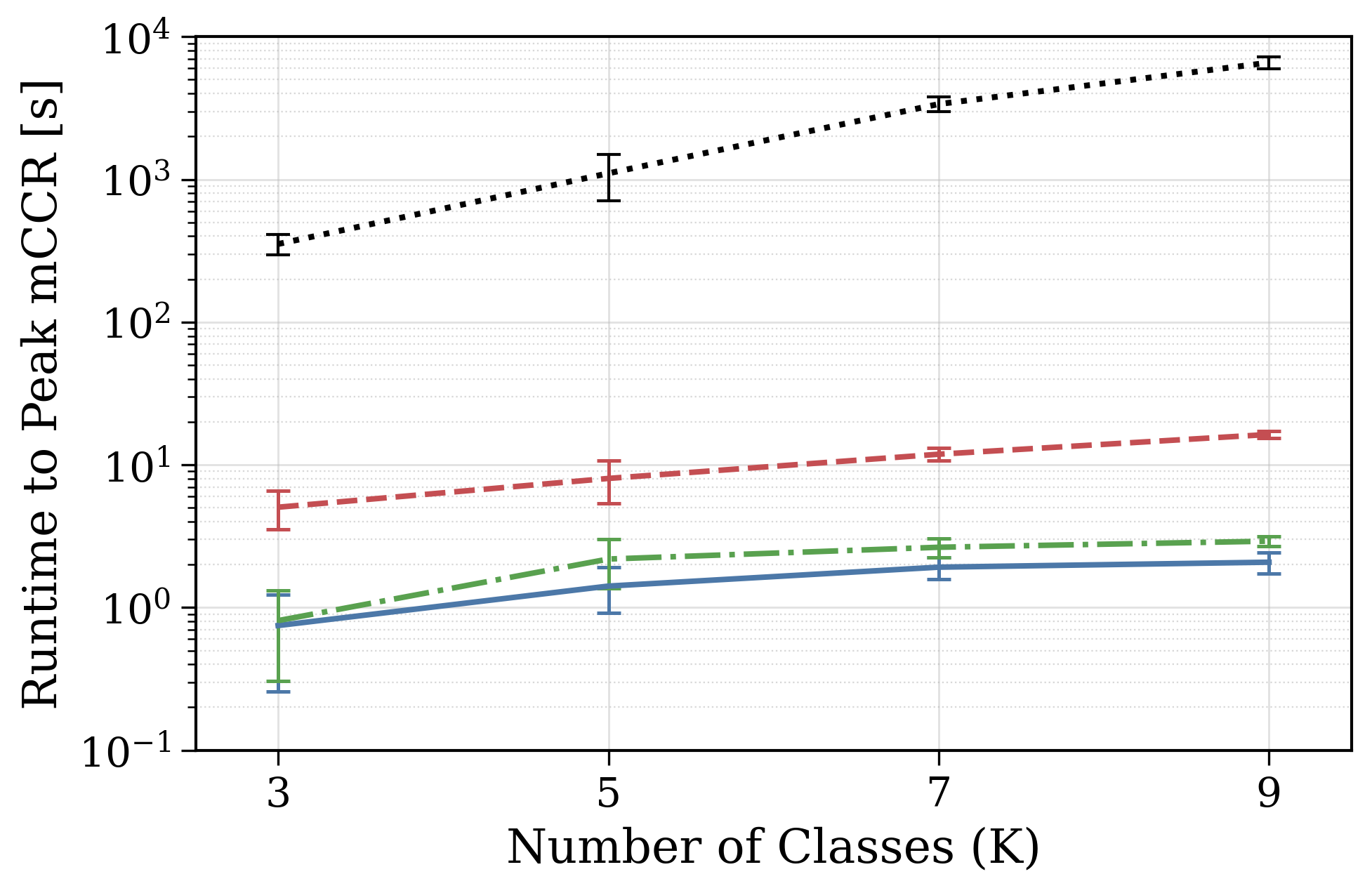}
        &
        \centering
        \includegraphics[width=0.95\linewidth]{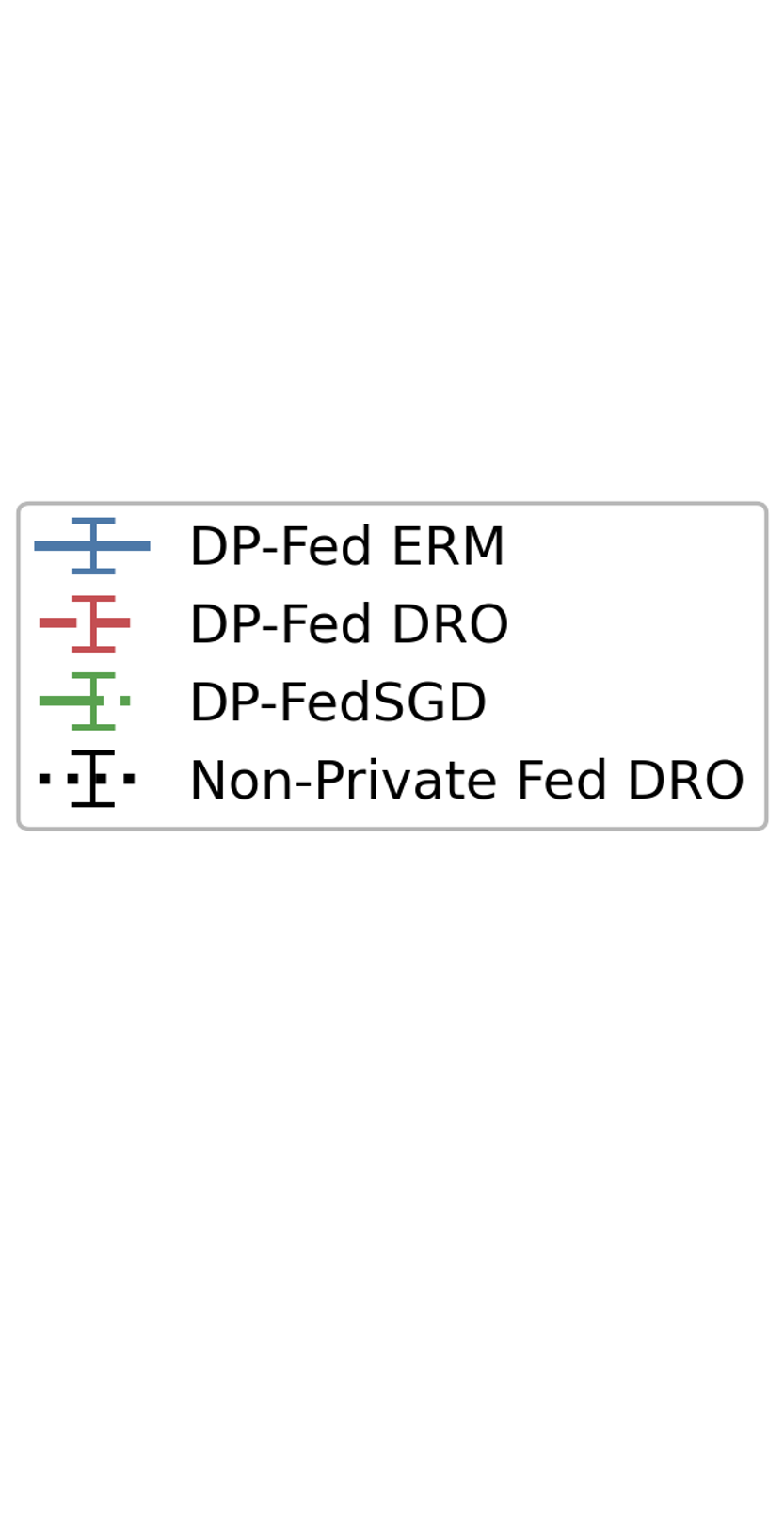}
        \tabularnewline

        \centering \footnotesize (a) Fixed total sample size
        &
        \centering \footnotesize (b) Fixed samples per client
        &
        \centering \footnotesize (c) Increasing number of classes
        &
        \tabularnewline
    \end{tabular}

    \caption{Log-scale runtime to peak mCCR for the tested federated models across the client- and class-scalability settings.}
    \label{fig:scalability_summary}
\end{figure*}

\textbf{Hyperparameters.}
Across all scalability experiments, we train for $T = 120$ rounds and fix the privacy parameters to $\epsilon = 0.2$ and $\delta = 10^{-8}$ per round, with the per-round budget split as in Section~\ref{sec:pump}, corresponding to a cumulative budget of $\bar\epsilon \approx 2.2$. The ADMM-based federated models use a global consensus parameter of \(\rho=0.5\). The DP-Fed DRO client problems use the label transportation cost \(\kappa=0.5\), the Wasserstein ambiguity radius \(r=0.02\), and the penalty parameter \(\gamma=0.1\), while the non-private DRO models use \(\kappa=0.25\) and \(r=10^{-5}\). The DP-Fed ERM baseline uses the regularization parameter \(\lambda=10^{-4}\), with \(c_w\) set on each configuration to the Frobenius norm of the non-private ERM solution, while DP-Fed SGD uses the decaying learning rate \(\eta_k=0.1/\sqrt{k}\) with the same regularization parameter. We use the \(\ell_{\infty}\)-norm throughout. The step-size constants are set to \(c_M=1\) and \(c_{\lambda,g}=1\) for all clients. The remaining constants are set to \(C=0.4\), \(V=1\), and \(L=1\), so that \(S=\max\{r,\kappa-r\}=0.48\).

\textbf{Results.}
The results are presented in Figure~\ref{fig:scalability_summary}. In the fixed-total-sample setting, the runtime required to reach peak mCCR remains relatively stable for the private models as the number of clients increases. Thus, adding clients does not substantially increase the runtime when the total dataset size is fixed. For the Non-Private Fed DRO baseline, the runtime decreases initially and then stabilizes, which is consistent with the fact that each client receives fewer local samples as \(G\) increases and the local DRO subproblems become smaller. Across all values of \(G\), Non-Private Fed DRO remains the most expensive method by at least an order of magnitude, because each round requires solving local DRO optimization subproblems, whereas the private models rely on simpler local updates. Among the private models, DP-Fed DRO is slower than DP-Fed ERM and DP-Fed SGD by a roughly constant factor, reflecting the additional terms introduced by the DRO formulation.

In the fixed-samples-per-client setting, the runtime generally increases with the number of clients. Unlike the fixed-total-sample setting, increasing \(G\) also increases the total sample size, so the models must process more data while coordinating across more local clients. The same relative ordering is observed: Non-Private Fed DRO has the largest runtime, while the three private models remain substantially faster and within an order of magnitude of one another. These results indicate that the proposed DP-Fed DRO method scales similarly to the private ERM and SGD baselines with respect to the number of clients, while maintaining the robustness benefits of the DRO formulation.

For scalability with respect to the number of classes, the runtime increases with \(K\) for all methods, especially for the Non-Private Fed DRO baseline. This behavior is expected because, when the feature dimension is fixed, increasing the number of classes makes the multiclass problem harder by requiring the models to separate more groups in the same feature space. The private models show a much milder increase. DP-Fed DRO remains slower than DP-Fed ERM and DP-Fed SGD by a roughly constant factor across the tested values of \(K\), so its relative overhead does not grow with the number of classes, and it stays orders of magnitude below the non-private model. This suggests that the additional cost of the private DRO formulation remains manageable as the number of classes grows.

\section{Conclusion}
\label{sec:conclusion}

In this paper, we address the problem of federated multiclass classification under feature and label uncertainty, with formal privacy guarantees for the local data. To this end, we develop a differentially private federated distributionally robust multiclass classification framework that combines Wasserstein-based robustness with sample-level differential privacy across all training rounds. We further derive a training algorithm that privatizes both the local model and the local DRO variable, and establish its convergence at an \(O(1/\sqrt{T})\) rate. On simulated industrial fault diagnosis data with distribution shift, DP-Fed DRO is the strongest private method at moderate privacy budgets, outperforms DP-Fed ERM at every budget where the private models are informative, and is the most robust to the choice of \(\delta\). On the benchmark datasets, DP-Fed DRO achieves the best private performance on most datasets, while the runtime experiments show that it scales with the number of clients and classes similarly to the private baselines. Comparisons with non-private native multiclass and OvA DRO show complementary behavior on real datasets, while the controlled synthetic study identifies correlated features and overlapping fault classes as the setting in which the native multiclass formulation provides its clearest advantage. Overall, these results show that robustness and privacy can be incorporated simultaneously in federated multiclass learning. Future work could investigate how to jointly optimize each client's communication schedule and its privacy budget allocation across training rounds, so that clients communicate only in selected rounds and spend their budget at different rates over training, under a fixed per-client privacy budget.

%\subsubsection*{Acknowledgments}
%Use unnumbered third level headings for the acknowledgments. All
%acknowledgments, including those to funding agencies, go at the end of the paper.
%Only add this information once your submission is accepted and deanonymized. 

\bibliography{main}
\bibliographystyle{tmlr}
\section*{Appendix}
\appendix
\section{Proofs}
\label{app:proofs}

\subsection{Proof of Proposition~\ref{prop:client_update}}
\label{app:proof_prop_client}
\begin{proof}
Given the updated global model \(M^\star\), client \(g\) updates its local model by minimizing the augmented local Lagrangian with respect to \((M_g,\lambda_g)\). By the same problem separability argument used in Proposition~1 of \citet{ibrahim2024federated}, this update is obtained from
\begin{align}
\inf_{M_g}\; \mathcal L_{\rho,g}(M_g,M^\star,\mu_g)
&= \inf_{M_g}\; \sup_{\mathbb Q_g\in \mathcal A^{(g)}_{r_g,1,d}(\Xi)} \mathbb E^{\mathbb Q_g}\!\left[\ell_{\mathrm{CS}}(M_g;\xi)\right] + \frac{\rho}{2}\|M_g-M^\star+\mu_g\|_F^2
\label{eq:prop41_pf_1}\\
&= \inf_{M_g,\lambda_g\ge 0} \bigg\{ \lambda_g r_g + \frac{1}{N_g}\sum_{n=1}^{N_g} \sup_{\xi\in\Xi} \bigl( \ell_{\mathrm{CS}}(M_g;\xi)-\lambda_g d(\xi,\hat\xi_n) \bigr) \notag\\
&\qquad + \frac{\rho}{2}\|M_g-M^\star+\mu_g\|_F^2 \bigg\}.
\label{eq:prop41_pf_2}
\end{align}

Introducing slack variables \(\{s_n\}_{n=1}^{N_g}\), we obtain
\begin{equation}
\label{eq:prop41_pf_3}
\begin{aligned}
&\sup_{\mathbb Q_g\in\mathcal A^{(g)}_{r_g,1,d}(\Xi)} \mathbb E^{\mathbb Q_g}\!\left[\ell_{\mathrm{CS}}(M_g;\xi)\right] \\
&\qquad = \begin{aligned}[t]
\inf_{\lambda_g,\{s_n\}} \quad & \lambda_g r_g + \frac{1}{N_g}\sum_{n=1}^{N_g}s_n \\
\text{s.t.}\quad & \sup_{\xi\in\Xi} \left( \ell_{\mathrm{CS}}(M_g;\xi) - \lambda_g d(\xi,\hat\xi_n) \right) \le s_n, \quad \forall n\in[N_g], \quad \lambda_g\ge 0.
\end{aligned}
\end{aligned}
\end{equation}

Since the inner supremum in \eqref{eq:prop41_pf_3} has the same Wasserstein robust structure as in the centralized multiclass DRO formulation, Theorem~1 in \citet{ibrahim2024wasserstein} applies here with \(M\) replaced by \(M_g\). Therefore, the constraints in \eqref{eq:prop41_pf_3} are equivalently rewritten as the following:
\begin{align}
\ell_{\mathrm{CS}}(M_g;x_n,y_n) &\le s_n,
\qquad \forall n\in[N_g],
\label{eq:prop41_pf_4}
\\
\ell_{\mathrm{CS}}(M_g;x_n,y)-\kappa\lambda_g &\le s_n,
\qquad
\forall n\in[N_g],\ \forall y\in\mathcal Y,\ y\neq y_n,
\label{eq:prop41_pf_5}
\\
\lambda_g &\ge \|(v_{(i)}-v_{(j)})^\top M_g\|_*,
\quad i,j\in[K].
\label{eq:prop41_pf_6}
\end{align}

Substituting \eqref{eq:prop41_pf_4}--\eqref{eq:prop41_pf_6}
into the clientwise augmented Lagrangian subproblem yields the
optimization problem stated in Proposition~\ref{prop:client_update}.

\end{proof}

\subsection{Proof of Lemma~\ref{lem:sensitivity_Mg}}
\label{app:proof_lem_sens_M}
\begin{proof}
Consider the proximal linearized client subproblem in
\eqref{eq:prox_linearized_lagrangian}. To make the dependence
on the local dataset explicit, denote the nonsmooth term in
\eqref{eq:def_f_local} by $f(M_g,\lambda_g;D_g).$
The first-order optimality condition for the update \(M_g^t(D_g)\)
with respect to \(M_g\) is
\begin{equation}
0 \in \partial_{M_g} f\!\left( \widetilde M_g^{t-1}, \widetilde{\lambda}_g^{t-1}; D_g \right) + \rho\left( M_g^t(D_g)-M^{t-1}+\mu_g^{t-1} \right) + \frac{1}{\eta_g} \left( M_g^t(D_g)-\widetilde M_g^{t-1} \right).
\label{eq:sens_pf_optimality}
\end{equation}
Equivalently,
\begin{equation}
\left( \rho+\frac{1}{\eta_g} \right) M_g^t(D_g) = - \partial_{M_g} f\!\left( \widetilde M_g^{t-1}, \widetilde{\lambda}_g^{t-1}; D_g \right) + \rho\left( M^{t-1}-\mu_g^{t-1} \right) + \frac{1}{\eta_g} \widetilde M_g^{t-1}.
\label{eq:sens_pf_update}
\end{equation}

Let \(D_g\sim D_g'\) be adjacent local datasets differing in
exactly one sample, indexed by \(j\). We condition on the
algorithmic state available at the beginning of iteration
\(t\), including the previous perturbed variables and server
broadcasts. Thus, \(M^{t-1}\), \(\mu_g^{t-1}\),
\(\widetilde M_g^{t-1}\),
\(\widetilde{\lambda}_g^{t-1}\), and
\(\tau_g^{t-1}\) are fixed when comparing the current update
under \(D_g\) and \(D_g'\). In particular, \(\tau_g^{t-1}\)
is computed from the previous perturbed variables according to
Algorithm~\ref{alg:dp_fed_dro}.

Under this fixed algorithmic state, the quadratic terms and the
linearized Bertsekas penalty contribution are identical for
the two adjacent datasets. Moreover, the robust-loss
contributions associated with all unchanged samples are
identical. Hence, the only difference in
\eqref{eq:sens_pf_update} arises from the single changed sample
in the empirical robust-loss average. Therefore,
\begin{equation}
M_g^t(D_g)-M_g^t(D_g') = -\left(\rho+\frac{1}{\eta_g}\right)^{-1} \frac{1}{N_g} \Bigl( \partial_{M_g}\ell_j^{(g)}(D_g) - \partial_{M_g}\ell_j^{(g)}(D_g') \Bigr),
\label{eq:sens_pf_difference}
\end{equation}
where
\(\partial_{M_g}\ell_j^{(g)}(D_g)\) and
\(\partial_{M_g}\ell_j^{(g)}(D_g')\)
denote selected subgradients of the robust-loss contribution
associated with the differing sample under \(D_g\) and \(D_g'\),
respectively.

For the multiclass Crammer--Singer loss, any active \(M_g\)-subgradient has the form
\begin{equation}
\left(v_{(c)}-y\right)x^\top
\label{eq:sens_pf_subgradient_form}
\end{equation}
for an active class \(c\neq y\). At points where multiple
classes are active, the subdifferential is the convex hull of
such matrices.

Using $\|ab^\top\|_F=\|a\|_2\|b\|_2,$
the one-hot encoding property
$\|v_{(c)}-y\|_2=\sqrt{2},$
and the bounded-feature assumption \(\|x\|_2\leq L\), we obtain
\begin{equation}
\left\|
\partial_{M_g}\ell_j^{(g)}(D_g)
\right\|_F
\leq \sqrt{2}L,
\qquad
\left\|
\partial_{M_g}\ell_j^{(g)}(D_g')
\right\|_F
\leq \sqrt{2}L.
\label{eq:sens_pf_subgradient_bound}
\end{equation}

Applying the triangle inequality to
\eqref{eq:sens_pf_difference} and then using
\eqref{eq:sens_pf_subgradient_bound} yields
\begin{align}
\left\| M_g^t(D_g)-M_g^t(D_g') \right\|_F
&\leq \left( \rho+\frac{1}{\eta_g} \right)^{-1} \frac{1}{N_g} \Bigl( \left\| \partial_{M_g}\ell_j^{(g)}(D_g) \right\|_F + \left\| \partial_{M_g}\ell_j^{(g)}(D_g') \right\|_F \Bigr) \notag\\
&\leq \frac{2\sqrt{2}L}{N_g\left(\rho+1/\eta_g\right)}.
\label{eq:sens_pf_final}
\end{align}
Taking the maximum over all adjacent datasets
\(D_g\sim D_g'\) yields the desired sensitivity bound.
\end{proof}

\subsection{Proof of Lemma~\ref{lem:sensitivity_lambda}}
\label{app:proof_lem_sens_lambda}
\begin{proof}
As in Lemma~\ref{lem:sensitivity_Mg}, condition on the algorithmic
state available at the beginning of iteration \(t\). The first-order
optimality condition for the \(\lambda_g\)-update gives
\[
\lambda_g^t(D_g)
=
\widetilde{\lambda}_g^{t-1}
-
\nu_g
\left(
r_g+
\partial_{\lambda_g}
f\!\left(
\widetilde M_g^{t-1},
\widetilde{\lambda}_g^{t-1};
D_g
\right)
\right).
\]

Let \(D_g\sim D_g'\) differ in exactly one sample, indexed by
\(j\). Subtracting the two updates,
\begin{equation}
\left| \lambda_g^t(D_g)-\lambda_g^t(D_g') \right| = \nu_g \left| \partial_{\lambda_g} f\!\left( \widetilde M_g^{t-1}, \widetilde{\lambda}_g^{t-1}; D_g \right) - \partial_{\lambda_g} f\!\left( \widetilde M_g^{t-1}, \widetilde{\lambda}_g^{t-1}; D_g' \right) \right|.
\label{eq:sens_lambda_difference}
\end{equation}

Under the fixed algorithmic state, the Bertsekas penalty
contribution and the robust-loss contributions associated with
all unchanged samples are identical under \(D_g\) and \(D_g'\).
Hence, only the differing sample contributes, so that
\begin{align}
\left|
\lambda_g^t(D_g)-\lambda_g^t(D_g')
\right|
&\le
\frac{\nu_g}{N_g}
\left|
\partial_{\lambda_g}\ell_j^{(g)}(D_g)
-
\partial_{\lambda_g}\ell_j^{(g)}(D_g')
\right|.
\end{align}

For each sample, the robust-loss subgradient with respect to
\(\lambda_g\) satisfies
\[
\partial_{\lambda_g}
\ell_n^{(g)}(M_g,\lambda_g)
\in[-\kappa,0].
\]
Therefore,
\[
\left|
\partial_{\lambda_g}\ell_j^{(g)}(D_g)
-
\partial_{\lambda_g}\ell_j^{(g)}(D_g')
\right|
\le
\kappa,
\]
which yields
\[
\left|
\lambda_g^t(D_g)-\lambda_g^t(D_g')
\right|
\le
\frac{\kappa\nu_g}{N_g}.
\]
Taking the maximum over all adjacent datasets gives the result.
\end{proof}

\subsection{Proof of Lemma~\ref{lem:single_iter}}
\label{app:proof_lem_single_iter}
\begin{proof}
For any client \(g\), let \(h_g^t(\cdot):=h_g(\cdot;\tau_g^{t-1})\). Since \(h_g^t\) is convex in \(z_g\), we have the subgradient inequality
\begin{equation}
h_g^t(\widetilde z_g^{\,t-1})-h_g^t(z_g)\le\big\langle h_g^{t\prime}(\widetilde z_g^{\,t-1}),\,\widetilde z_g^{\,t-1}-z_g\big\rangle .
\end{equation}

Adding the scaled-dual consensus term to both sides yields

\begin{equation}
h_g^t(\widetilde z_g^{\,t-1})-h_g^t(z_g)
+\rho\langle \widetilde M_g^t-M_g,\mu_g^t\rangle_F
\le
\big\langle h_g^{t\prime}(\widetilde z_g^{\,t-1}),
\,\widetilde z_g^{\,t-1}-z_g \big\rangle
+\rho\langle \widetilde M_g^t-M_g,\mu_g^t\rangle_F .
\label{eq:lem1_app_2}
\end{equation}

Define the padded dual variable $\bar\mu_g^t := (\mu_g^t,0)$.
Expanding $\widetilde z_g^{\,t-1}-z_g
=
(\widetilde z_g^{\,t-1}-\widetilde z_g^t)+(\widetilde z_g^t-z_g),$
we obtain
\begin{equation}
\left\langle h_g^{t\prime}(\widetilde z_g^{\,t-1}),
\,\widetilde z_g^{\,t-1}-z_g \right\rangle
+\rho\langle \widetilde M_g^t-M_g,\mu_g^t\rangle_F
=
\left\langle h_g^{t\prime}(\widetilde z_g^{\,t-1}),
\,\widetilde z_g^{\,t-1}-\widetilde z_g^t \right\rangle
+\left\langle h_g^{t\prime}(\widetilde z_g^{\,t-1})+\rho\bar\mu_g^t,
\,\widetilde z_g^t-z_g \right\rangle .
\label{eq:lem1_app_3}
\end{equation}

Define also the padded noise variables $\bar\xi_g^t := (\xi_g^t,0)$ and $\bar\zeta_g^t := (0,\zeta_g^t).$
For coefficients \(c_k\) and \(c_\ell\) specified below, adding and subtracting
\(c_k\bar\xi_g^t+c_\ell\bar\zeta_g^t\) in both inner products gives
\begin{multline}
\left\langle h_g^{t\prime}(\widetilde z_g^{\,t-1}),
\,\widetilde z_g^{\,t-1}-\widetilde z_g^t \right\rangle
+\left\langle h_g^{t\prime}(\widetilde z_g^{\,t-1})+\rho\bar\mu_g^t,
\,\widetilde z_g^t-z_g \right\rangle
\\
=
\left\langle
h_g^{t\prime}(\widetilde z_g^{\,t-1})
-c_k\bar\xi_g^t
-c_\ell\bar\zeta_g^t,\,
\widetilde z_g^{\,t-1}-\widetilde z_g^t
\right\rangle
\\
\qquad
-c_k\langle \xi_g^t,\,
M_g-\widetilde M_g^{\,t-1}\rangle_F
-c_\ell\zeta_g^t
\bigl(\lambda_g-\widetilde\lambda_g^{\,t-1}\bigr)
\\
\qquad
+\left\langle
h_g^{t\prime}(\widetilde z_g^{\,t-1})
+\rho\bar\mu_g^t
-c_k\bar\xi_g^t
-c_\ell\bar\zeta_g^t,\,
\widetilde z_g^t-z_g
\right\rangle .
\label{eq:lem1_app_4}
\end{multline}

Using the scaled-dual update
\(\bar\mu_g^t=\bar\mu_g^{t-1}+(\widetilde M_g^t-M^t,0)\)
and the decomposition
\[
(\widetilde M_g^t-M^t,0)
=
(\widetilde M_g^t-M^{t-1},0)
+
(M^{t-1}-M^t,0),
\]
we obtain
\begin{multline}
\left\langle
h_g^{t\prime}(\widetilde z_g^{\,t-1})
-c_k\bar\xi_g^t
-c_\ell\bar\zeta_g^t,\,
\widetilde z_g^{\,t-1}-\widetilde z_g^t
\right\rangle
\\
-c_k\langle \xi_g^t,\,
M_g-\widetilde M_g^{\,t-1}\rangle_F
-c_\ell\zeta_g^t
\bigl(\lambda_g-\widetilde\lambda_g^{\,t-1}\bigr)
\\
+\left\langle
h_g^{t\prime}(\widetilde z_g^{\,t-1})
+\rho\bar\mu_g^t
-c_k\bar\xi_g^t
-c_\ell\bar\zeta_g^t,\,
\widetilde z_g^t-z_g
\right\rangle
\\
=
\left\langle
h_g^{t\prime}(\widetilde z_g^{\,t-1})
-c_k\bar\xi_g^t
-c_\ell\bar\zeta_g^t,\,
\widetilde z_g^{\,t-1}-\widetilde z_g^t
\right\rangle
\\
-c_k\langle \xi_g^t,\,
M_g-\widetilde M_g^{\,t-1}\rangle_F
-c_\ell\zeta_g^t
\bigl(\lambda_g-\widetilde\lambda_g^{\,t-1}\bigr)
\\
+\left\langle
h_g^{t\prime}(\widetilde z_g^{\,t-1})
+\rho \bar\mu_g^{\,t-1}
+\rho(\widetilde M_g^t-M^{t-1},0)
-c_k\bar\xi_g^t
-c_\ell\bar\zeta_g^t,\,
\widetilde z_g^t-z_g
\right\rangle
\\
+\rho\langle
\widetilde M_g^t-M_g,\,
M^{t-1}-M^t
\rangle_F .
\label{eq:lem1_app_5}
\end{multline}

We now upper bound the last term in \eqref{eq:lem1_app_5} by rewriting the inner product as a difference of squared Frobenius norms.  
Using the identity
\begin{equation}
\|A-C\|_F^2 = \|A-B\|_F^2+\|B-C\|_F^2+2\langle A-B,\;B-C\rangle_F,
\label{eq:basic_frob_expand}
\end{equation}
we first apply \eqref{eq:basic_frob_expand} with $
(A,B,C)=(M_g,M^t,M^{t-1}), $
which gives
\begin{equation}
\|M_g-M^{t-1}\|_F^2
=
\|M_g-M^t\|_F^2+\|M^t-M^{t-1}\|_F^2
+2\langle M_g-M^t,\;M^t-M^{t-1}\rangle_F.
\label{eq:lem1_app_sq1}
\end{equation}
Next, apply \eqref{eq:basic_frob_expand} with
$(A,B,C)=(\widetilde M_g^t,M^t,M^{t-1}),$
which yields
\begin{equation}
\|\widetilde M_g^t-M^{t-1}\|_F^2
=
\|\widetilde M_g^t-M^t\|_F^2+\|M^t-M^{t-1}\|_F^2
+2\langle \widetilde M_g^t-M^t,\;M^t-M^{t-1}\rangle_F.
\label{eq:lem1_app_sq2}
\end{equation}
Subtracting \eqref{eq:lem1_app_sq1} from \eqref{eq:lem1_app_sq2}, the common term
\(\|M^t-M^{t-1}\|_F^2\) cancels, and we obtain
\begin{equation}
2\langle \widetilde M_g^t-M_g,\;M^{t-1}-M^t\rangle_F
=
\|M_g-M^{t-1}\|_F^2-\|M_g-M^t\|_F^2
+\|\widetilde M_g^t-M^t\|_F^2-\|\widetilde M_g^t-M^{t-1}\|_F^2.
\label{eq:lem1_app_sq3}
\end{equation}
Multiplying by \(\rho/2\) and dropping the final nonpositive term gives
\begin{equation}
\rho\langle \widetilde M_g^t-M_g,\;M^{t-1}-M^t\rangle_F
\le
\frac{\rho}{2}
\Big(
\|M_g-M^{t-1}\|_F^2-\|M_g-M^t\|_F^2
\Big)
+\frac{\rho}{2}\|\widetilde M_g^t-M^t\|_F^2.
\end{equation}
Finally, using the scaled dual update
\[
\mu_g^t-\mu_g^{t-1}=\widetilde M_g^t-M^t,
\]
we conclude that
\begin{equation}
\rho\langle \widetilde M_g^t-M_g,\;M^{t-1}-M^t\rangle_F
\le
\frac{\rho}{2}
\Big(
\|M_g-M^{t-1}\|_F^2-\|M_g-M^t\|_F^2
\Big)
+\frac{\rho}{2}\|\mu_g^t-\mu_g^{t-1}\|_F^2 .
\label{eq:lem1_app_6}
\end{equation}

Set $c_k:=\rho+\frac{1}{\eta_g},$ and $c_\ell:=\frac{1}{\nu_g}.$
We now split the second-to-last term in \eqref{eq:lem1_app_5}
into its \(M\)- and \(\lambda\)-components. Since $\bar\mu_g^{\,t-1}=(\mu_g^{t-1},0),\bar\xi_g^t=(\xi_g^t,0), $ and $ \bar\zeta_g^t=(0,\zeta_g^t),$
we have
\begin{multline}
\left\langle
h_g^{t\prime}(\widetilde z_g^{\,t-1})
+\rho \bar\mu_g^{\,t-1}
+\rho(\widetilde M_g^t-M^{t-1},0)
-c_k\bar\xi_g^t
-c_\ell\bar\zeta_g^t,\,
\widetilde z_g^t-z_g
\right\rangle
\\
=
\left\langle
[h_g^{t\prime}(\widetilde z_g^{\,t-1})]_M
+\rho(\widetilde M_g^t-M^{t-1}+\mu_g^{t-1})
-\left(\rho+\frac{1}{\eta_g}\right)\xi_g^t,\,
\widetilde M_g^t-M_g
\right\rangle_F
\\
\qquad
+
\left(
[h_g^{t\prime}(\widetilde z_g^{\,t-1})]_{\lambda}
-\frac{1}{\nu_g}\zeta_g^t
\right)
(\widetilde\lambda_g^t-\lambda_g).
\label{eq:lem1_app_7split}
\end{multline}

Applying Lemma~5 of \citet{huang2019dp} to the \(M\)-block term in \eqref{eq:lem1_app_7split}, with $D(X,Y)=\frac12\|X-Y\|_F^2 $ and $
s=\frac{1}{\eta_g},$
yields
\begin{multline}
\left\langle
[h_g^{t\prime}(\widetilde z_g^{\,t-1})]_M
+\rho(\widetilde M_g^t-M^{t-1}+\mu_g^{t-1})
-\left(\rho+\frac{1}{\eta_g}\right)\xi_g^t,\,
\widetilde M_g^t-M_g
\right\rangle_F
\\
\le
\frac{1}{2\eta_g}
\Big(
\|M_g-\widetilde M_g^{\,t-1}\|_F^2
-\|M_g-\widetilde M_g^t\|_F^2
-\|\widetilde M_g^t-\widetilde M_g^{\,t-1}\|_F^2
\Big).
\label{eq:lem1_app_7M}
\end{multline}

Applying Lemma~5 of \citet{huang2019dp} to the \(\lambda\)-block, with
$D(x,y)=\frac12(x-y)^2$ and
$s=\frac{1}{\nu_g},$
yields
\begin{equation}
\left(
[h_g^{t\prime}(\widetilde z_g^{\,t-1})]_{\lambda}
-\frac{1}{\nu_g}\zeta_g^t
\right)
(\widetilde\lambda_g^t-\lambda_g)
\le
\frac{1}{2\nu_g}
\Big(
(\lambda_g-\widetilde\lambda_g^{t-1})^2
-(\lambda_g-\widetilde\lambda_g^t)^2
-(\widetilde\lambda_g^t-\widetilde\lambda_g^{t-1})^2
\Big).
\label{eq:lem1_app_7lambda}
\end{equation}

Combining \eqref{eq:lem1_app_7split}, \eqref{eq:lem1_app_7M},
and \eqref{eq:lem1_app_7lambda}, we obtain
\begin{multline}
\left\langle
h_g^{t\prime}(\widetilde z_g^{\,t-1})
+\rho \bar\mu_g^{\,t-1}
+\rho(\widetilde M_g^t-M^{t-1},0)
-\left(\rho+\frac{1}{\eta_g}\right)\bar\xi_g^t
-\frac{1}{\nu_g}\bar\zeta_g^t,\,
\widetilde z_g^t-z_g
\right\rangle
\\
\le
\frac{1}{2\eta_g}
\Big(
\|M_g-\widetilde M_g^{\,t-1}\|_F^2
-\|M_g-\widetilde M_g^t\|_F^2
-\|\widetilde M_g^t-\widetilde M_g^{\,t-1}\|_F^2
\Big)
\\
\qquad
+
\frac{1}{2\nu_g}
\Big(
(\lambda_g-\widetilde\lambda_g^{\,t-1})^2
-(\lambda_g-\widetilde\lambda_g^t)^2
-(\widetilde\lambda_g^t-\widetilde\lambda_g^{\,t-1})^2
\Big).
\label{eq:lem1_app_7full}
\end{multline}

Then, using the block decomposition
\begin{equation}
\begin{aligned}
&\left\langle
h_g^{t\prime}(\widetilde z_g^{\,t-1})
-\left(\rho+\frac{1}{\eta_g}\right)\bar\xi_g^t
-\frac{1}{\nu_g}\bar\zeta_g^t,\,
\widetilde z_g^{\,t-1}-\widetilde z_g^t
\right\rangle
\\
&=
\left\langle
[h_g^{t\prime}(\widetilde z_g^{\,t-1})]_M
-\left(\rho+\frac{1}{\eta_g}\right)\xi_g^t,\,
\widetilde M_g^{\,t-1}-\widetilde M_g^t
\right\rangle_F
\\
&\qquad
+
\left(
[h_g^{t\prime}(\widetilde z_g^{\,t-1})]_\lambda
-\frac{1}{\nu_g}\zeta_g^t
\right)
(\widetilde\lambda_g^{\,t-1}-\widetilde\lambda_g^t),
\end{aligned}
\label{eq:lem1_app_8split}
\end{equation}
and applying Cauchy--Schwarz and Young's inequality blockwise, we obtain
\begin{multline}
\left\langle
h_g^{t\prime}(\widetilde z_g^{\,t-1})
-\left(\rho+\frac{1}{\eta_g}\right)\bar\xi_g^t
-\frac{1}{\nu_g}\bar\zeta_g^t,\,
\widetilde z_g^{\,t-1}-\widetilde z_g^t
\right\rangle
\\
\le
\frac{\eta_g}{2}
\left\|
[h_g^{t\prime}(\widetilde z_g^{\,t-1})]_M
-\left(\rho+\frac{1}{\eta_g}\right)\xi_g^t
\right\|_F^2
+
\frac{1}{2\eta_g}
\|\widetilde M_g^t-\widetilde M_g^{\,t-1}\|_F^2
\\
\qquad
+
\frac{\nu_g}{2}
\left|
[h_g^{t\prime}(\widetilde z_g^{\,t-1})]_\lambda
-\frac{1}{\nu_g}\zeta_g^t
\right|^2
+
\frac{1}{2\nu_g}
(\widetilde\lambda_g^t-\widetilde\lambda_g^{\,t-1})^2 .
\label{eq:lem1_app_8full}
\end{multline}

Combining the previous three upper bounds, we obtain
\begin{multline}
h_g^t(\widetilde z_g^{\,t-1})-h_g^t(z_g)
+\rho\langle \widetilde M_g^t-M_g,\mu_g^t\rangle_F
\\
\le
\frac{\eta_g}{2}
\left\|
[h_g^{t\prime}(\widetilde z_g^{\,t-1})]_M
-\left(\rho+\frac{1}{\eta_g}\right)\xi_g^t
\right\|_F^2
\\
\qquad
+\frac{\nu_g}{2}
\left|
[h_g^{t\prime}(\widetilde z_g^{\,t-1})]_\lambda
-\frac{1}{\nu_g}\zeta_g^t
\right|^2
\\
\qquad
-\left(\rho+\frac{1}{\eta_g}\right)
\langle \xi_g^t,\,
M_g-\widetilde M_g^{\,t-1}\rangle_F
\\
\qquad
-\frac{1}{\nu_g}\zeta_g^t
\bigl(\lambda_g-\widetilde\lambda_g^{\,t-1}\bigr)
\\
\qquad
+\frac{1}{2\eta_g}
\Big(
\|M_g-\widetilde M_g^{\,t-1}\|_F^2
-\|M_g-\widetilde M_g^t\|_F^2
\Big)
\\
\qquad
+\frac{1}{2\nu_g}
\Big(
(\lambda_g-\widetilde\lambda_g^{\,t-1})^2
-(\lambda_g-\widetilde\lambda_g^t)^2
\Big)
\\
\qquad
+\frac{\rho}{2}
\Big(
\|M_g-M^{t-1}\|_F^2
-\|M_g-M^{t}\|_F^2
\Big)
+\frac{\rho}{2}
\|\mu_g^t-\mu_g^{t-1}\|_F^2 .
\label{eq:lem1_app_9}
\end{multline}

For each client \(g\), add the terms
\[
-\rho\langle \mu_g^t,\;M^t-M\rangle_F
\qquad\text{and}\qquad
\rho\langle \mu_g-\mu_g^t,\;\widetilde M_g^t-M^t\rangle_F
\]
to both sides of \eqref{eq:lem1_app_9}. Then, weighting by \(\alpha_g\) and summing over \(g=1,\dots,G\), we arrive at
\begin{align}
\sum_{g=1}^G \alpha_g
&\Big(
h_g^t(\widetilde z_g^{\,t-1})-h_g^t(z_g)
+\rho\langle \widetilde M_g^t-M_g,\mu_g^t\rangle_F
\notag\\
&\qquad
-\rho\langle \mu_g^t,\,
M^t-M\rangle_F
+\rho\langle \mu_g-\mu_g^t,\,
\widetilde M_g^t-M^t\rangle_F
\Big)
\notag\\
&\le
\sum_{g=1}^G \alpha_g
\Bigg[
\frac{\eta_g}{2}
\left\|
[h_g^{t\prime}(\widetilde z_g^{\,t-1})]_M
-\left(\rho+\frac{1}{\eta_g}\right)\xi_g^t
\right\|_F^2
\notag\\
&\qquad
+\frac{\nu_g}{2}
\left|
[h_g^{t\prime}(\widetilde z_g^{\,t-1})]_\lambda
-\frac{1}{\nu_g}\zeta_g^t
\right|^2
\notag\\
&\qquad
-\left(\rho+\frac{1}{\eta_g}\right)
\langle \xi_g^t,\,
M_g-\widetilde M_g^{\,t-1}\rangle_F
\notag\\
&\qquad
-\frac{1}{\nu_g}\zeta_g^t
\bigl(\lambda_g-\widetilde\lambda_g^{\,t-1}\bigr)
\notag\\
&\qquad
+\frac{1}{2\eta_g}
\Big(
\|M_g-\widetilde M_g^{\,t-1}\|_F^2
-\|M_g-\widetilde M_g^t\|_F^2
\Big)
\notag\\
&\qquad
+\frac{1}{2\nu_g}
\Big(
(\lambda_g-\widetilde\lambda_g^{\,t-1})^2
-(\lambda_g-\widetilde\lambda_g^t)^2
\Big)
\notag\\
&\qquad
+\frac{\rho}{2}
\Big(
\|M_g-M^{t-1}\|_F^2
-\|M_g-M^{t}\|_F^2
\Big)
+\frac{\rho}{2}\|\mu_g^t-\mu_g^{t-1}\|_F^2
\notag\\
&\qquad
-\rho\langle \mu_g^t,\,
M^t-M\rangle_F
+\rho\langle \mu_g-\mu_g^t,\,
\widetilde M_g^t-M^t\rangle_F
\Bigg].
\label{eq:lem1_app_10}
\end{align}

Next, using the dual update and the server aggregation rule, we obtain the identities in \eqref{eq:lem1_app_11} and \eqref{eq:lem1_app_12}. In particular, \eqref{eq:lem1_app_11} follows from the fact that the weighted dual average is zero, while \eqref{eq:lem1_app_12} follows by substituting
$\widetilde M_g^t-M^t=\mu_g^t-\mu_g^{t-1}$
and expanding the squared Frobenius norm:
\begin{equation}
\sum_{g=1}^G \alpha_g \langle M^t-M,\mu_g^t\rangle_F
=
\left\langle
M^t-M,\,
\sum_{g=1}^G \alpha_g\mu_g^t
\right\rangle_F
=0 ,
\label{eq:lem1_app_11}
\end{equation}
\begin{equation}
\langle \mu_g-\mu_g^t,\,
\widetilde M_g^t-M^t\rangle_F
=
\langle \mu_g-\mu_g^t,\,
\mu_g^t-\mu_g^{t-1}\rangle_F
=
\frac{1}{2}
\Big(
\|\mu_g-\mu_g^{t-1}\|_F^2
-\|\mu_g-\mu_g^t\|_F^2
-\|\mu_g^t-\mu_g^{t-1}\|_F^2
\Big).
\label{eq:lem1_app_12}
\end{equation}
Substituting \eqref{eq:lem1_app_11} and \eqref{eq:lem1_app_12}
into \eqref{eq:lem1_app_10}, the term
$\frac{\rho}{2}\|\mu_g^t-\mu_g^{t-1}\|_F^2$
cancels with the last term in \eqref{eq:lem1_app_12}, and we obtain
\begin{multline}
\sum_{g=1}^G \alpha_g
\Big(
h_g^t(\widetilde z_g^{\,t-1})-h_g^t(z_g)
+\rho\langle \widetilde M_g^t-M_g,\mu_g^t\rangle_F
\\
\qquad
+\rho\langle \mu_g-\mu_g^t,\,
\widetilde M_g^t-M^t\rangle_F
\Big)
\\
\le
\sum_{g=1}^G \alpha_g
\Bigg[
\frac{\eta_g}{2}
\left\|
[h_g^{t\prime}(\widetilde z_g^{\,t-1})]_M
-\left(\rho+\frac{1}{\eta_g}\right)\xi_g^t
\right\|_F^2
\\
\qquad
+\frac{\nu_g}{2}
\left|
[h_g^{t\prime}(\widetilde z_g^{\,t-1})]_\lambda
-\frac{1}{\nu_g}\zeta_g^t
\right|^2
\\
\qquad
-\left(\rho+\frac{1}{\eta_g}\right)
\langle \xi_g^t,\,
M_g-\widetilde M_g^{\,t-1}\rangle_F
\\
\qquad
-\frac{1}{\nu_g}\zeta_g^t
\bigl(\lambda_g-\widetilde\lambda_g^{\,t-1}\bigr)
\\
\qquad
+\frac{1}{2\eta_g}
\Big(
\|M_g-\widetilde M_g^{\,t-1}\|_F^2
-\|M_g-\widetilde M_g^t\|_F^2
\Big)
\\
\qquad
+\frac{1}{2\nu_g}
\Big(
(\lambda_g-\widetilde\lambda_g^{\,t-1})^2
-(\lambda_g-\widetilde\lambda_g^t)^2
\Big)
\\
\qquad
+\frac{\rho}{2}
\Big(
\|M_g-M^{t-1}\|_F^2
-\|M_g-M^{t}\|_F^2
\Big)
\\
\qquad
+\frac{\rho}{2}
\Big(
\|\mu_g-\mu_g^{t-1}\|_F^2
-\|\mu_g-\mu_g^t\|_F^2
\Big)
\Bigg].
\label{eq:lem1_app_13}
\end{multline}
\end{proof}

\subsection{Auxiliary Lemmas}
\label{app:aux_lemmas}

\begin{lemma}\label{lem:coercive}
For each client $g$, assume that
\[
\eta_g>0,\qquad \nu_g>0,\qquad \rho>0.
\]
Then, for each iteration $t$, the proximal linearized surrogate $\widehat L_{\rho,g}^t(M_g,\lambda_g)$ is coercive in $(M_g,\lambda_g)$.
\end{lemma}

\begin{proof}
Recall that
\begin{multline}
\widehat L_{\rho,g}^t(M_g,\lambda_g)
=
f(\widetilde M_g^{t-1},\widetilde\lambda_g^{t-1})
+
\left\langle \partial_{M_g} f(\widetilde M_g^{t-1},\widetilde\lambda_g^{t-1}),\, M_g-\widetilde M_g^{t-1}\right\rangle_F
\\
+
\left\langle \partial_{\lambda_g} f(\widetilde M_g^{t-1},\widetilde\lambda_g^{t-1}),\, \lambda_g-\widetilde\lambda_g^{t-1}\right\rangle
+
\lambda_g r_g
+
\frac{\rho}{2}\left\|M_g-M^{t-1}+\mu_g^{t-1}\right\|_F^2
\\
+
\frac{1}{2\eta_g}\left\|M_g-\widetilde M_g^{t-1}\right\|_F^2
+
\frac{1}{2\nu_g}\left(\lambda_g-\widetilde\lambda_g^{t-1}\right)^2 .
\end{multline}

The term $f(\widetilde M_g^{t-1},\widetilde\lambda_g^{t-1})$ is constant with respect to $(M_g,\lambda_g)$ and therefore does not affect coercivity.

We first expand the terms depending on $M_g$. We have
\begin{equation}
\left\langle \partial_{M_g} f(\widetilde M_g^{t-1},\widetilde\lambda_g^{t-1}),\, M_g-\widetilde M_g^{t-1}\right\rangle_F
=
\left\langle \partial_{M_g} f(\widetilde M_g^{t-1},\widetilde\lambda_g^{t-1}),\, M_g\right\rangle_F
\\
-
\left\langle \partial_{M_g} f(\widetilde M_g^{t-1},\widetilde\lambda_g^{t-1}),\, \widetilde M_g^{t-1}\right\rangle_F ,
\end{equation}
and
\begin{equation}
\frac{1}{2\eta_g}\left\|M_g-\widetilde M_g^{t-1}\right\|_F^2
=
\frac{1}{2\eta_g}\|M_g\|_F^2
-
\frac{1}{\eta_g}\left\langle M_g,\widetilde M_g^{t-1}\right\rangle_F
\\
+
\frac{1}{2\eta_g}\left\|\widetilde M_g^{t-1}\right\|_F^2 ,
\end{equation}
while
\begin{equation}
\frac{\rho}{2}\left\|M_g-M^{t-1}+\mu_g^{t-1}\right\|_F^2
=
\frac{\rho}{2}\|M_g\|_F^2
-
\rho\left\langle M_g,\, M^{t-1}-\mu_g^{t-1}\right\rangle_F
\\
+
\frac{\rho}{2}\left\|M^{t-1}-\mu_g^{t-1}\right\|_F^2 .
\end{equation}

Therefore, all terms independent of \(M_g\) can be grouped into \(B_g^t\), and
\(\widehat L_{\rho,g}^t(M_g,\lambda_g)\) can be written as
\begin{equation}
\left(\frac{1}{2\eta_g}+\frac{\rho}{2}\right)\|M_g\|_F^2
+
\Bigl\langle
\partial_{M_g} f(\widetilde M_g^{t-1},\widetilde\lambda_g^{t-1})
-
\frac{1}{\eta_g}\widetilde M_g^{t-1}
\\
-
\rho\bigl(M^{t-1}-\mu_g^{t-1}\bigr),
\, M_g
\Bigr\rangle_F
+
B_g^t.
\end{equation}

Next, we expand the terms depending on $\lambda_g$. We have
\begin{equation}
\left\langle \partial_{\lambda_g} f(\widetilde M_g^{t-1},\widetilde\lambda_g^{t-1}),\, \lambda_g-\widetilde\lambda_g^{t-1}\right\rangle
=
\partial_{\lambda_g} f(\widetilde M_g^{t-1},\widetilde\lambda_g^{t-1})\,\lambda_g
\\
-
\partial_{\lambda_g} f(\widetilde M_g^{t-1},\widetilde\lambda_g^{t-1})\,\widetilde\lambda_g^{t-1},
\end{equation}
and
\begin{equation}
\frac{1}{2\nu_g}\left(\lambda_g-\widetilde\lambda_g^{t-1}\right)^2
=
\frac{1}{2\nu_g}\lambda_g^2
-
\frac{1}{\nu_g}\widetilde\lambda_g^{t-1}\lambda_g
+
\frac{1}{2\nu_g}\left(\widetilde\lambda_g^{t-1}\right)^2 .
\end{equation}

Hence, all terms independent of \(\lambda_g\) can be grouped into \(C_g^t\), and
\(\widehat L_{\rho,g}^t(M_g,\lambda_g)\) can be written with respect to \(\lambda_g\) as
\begin{equation}
\frac{1}{2\nu_g}\lambda_g^2
+
\left(
\partial_{\lambda_g} f(\widetilde M_g^{t-1},\widetilde\lambda_g^{t-1})
-
\frac{1}{\nu_g}\widetilde\lambda_g^{t-1}
+
r_g
\right)\lambda_g
+
C_g^t.
\end{equation}

Combining the two expansions, and absorbing all terms independent of \((M_g,\lambda_g)\) into \(D_g^t\), we obtain
\begin{multline}
\widehat L_{\rho,g}^t(M_g,\lambda_g)
=
\left(\frac{1}{2\eta_g}+\frac{\rho}{2}\right)\|M_g\|_F^2
+
\frac{1}{2\nu_g}\lambda_g^2
\\
+
\Bigl\langle
\partial_{M_g} f(\widetilde M_g^{t-1},\widetilde\lambda_g^{t-1})
-
\frac{1}{\eta_g}\widetilde M_g^{t-1}
-
\rho\bigl(M^{t-1}-\mu_g^{t-1}\bigr),
\, M_g
\Bigr\rangle_F
\\
+
\left(
\partial_{\lambda_g} f(\widetilde M_g^{t-1},\widetilde\lambda_g^{t-1})
-
\frac{1}{\nu_g}\widetilde\lambda_g^{t-1}
+
r_g
\right)\lambda_g
+
D_g^t.
\end{multline}

Now apply Cauchy--Schwarz to the linear terms. We get
\begin{multline}
\Bigl\langle
\partial_{M_g} f(\widetilde M_g^{t-1},\widetilde\lambda_g^{t-1})
-
\frac{1}{\eta_g}\widetilde M_g^{t-1}
-
\rho\bigl(M^{t-1}-\mu_g^{t-1}\bigr),
\, M_g
\Bigr\rangle_F
\\
\ge
-
\left\|
\partial_{M_g} f(\widetilde M_g^{t-1},\widetilde\lambda_g^{t-1})
-
\frac{1}{\eta_g}\widetilde M_g^{t-1}
-
\rho\bigl(M^{t-1}-\mu_g^{t-1}\bigr)
\right\|_F
\|M_g\|_F,
\end{multline}
and
\begin{equation}
\left(
\partial_{\lambda_g} f(\widetilde M_g^{t-1},\widetilde\lambda_g^{t-1})
-
\frac{1}{\nu_g}\widetilde\lambda_g^{t-1}
+
r_g
\right)\lambda_g
\\
\ge
-
\left|
\partial_{\lambda_g} f(\widetilde M_g^{t-1},\widetilde\lambda_g^{t-1})
-
\frac{1}{\nu_g}\widetilde\lambda_g^{t-1}
+
r_g
\right|
\, |\lambda_g|.
\end{equation}

Therefore,
\begin{align}
\widehat L_{\rho,g}^t(M_g,\lambda_g)
\ge\;&
\left(\frac{1}{2\eta_g}+\frac{\rho}{2}\right)\|M_g\|_F^2
\nonumber\\
&\quad
-
\left\|
\partial_{M_g} f(\widetilde M_g^{t-1},\widetilde\lambda_g^{t-1})
-
\frac{1}{\eta_g}\widetilde M_g^{t-1}
\right.
\nonumber\\
&\qquad\qquad
\left.
-
\rho\bigl(M^{t-1}-\mu_g^{t-1}\bigr)
\right\|_F
\|M_g\|_F
\nonumber\\
&\quad
+
\frac{1}{2\nu_g}\lambda_g^2
-
\left|
\partial_{\lambda_g} f(\widetilde M_g^{t-1},\widetilde\lambda_g^{t-1})
-
\frac{1}{\nu_g}\widetilde\lambda_g^{t-1}
+
r_g
\right|
\, |\lambda_g|
+
D_g^t.
\end{align}

Since
\[
\frac{1}{2\eta_g}+\frac{\rho}{2}>0
\qquad\text{and}\qquad
\frac{1}{2\nu_g}>0,
\]
the quadratic terms dominate the corresponding linear terms as
\[
\|M_g\|_F^2+|\lambda_g|^2\to\infty.
\]
Indeed,
\begin{align}
&\left(\frac{1}{2\eta_g}+\frac{\rho}{2}\right)\|M_g\|_F^2
-
\left\|
\partial_{M_g} f(\widetilde M_g^{t-1},\widetilde\lambda_g^{t-1})
-
\frac{1}{\eta_g}\widetilde M_g^{t-1}
\right.
\nonumber\\
&\qquad\qquad
\left.
-
\rho\bigl(M^{t-1}-\mu_g^{t-1}\bigr)
\right\|_F
\|M_g\|_F
\to\infty
\end{align}
as $\|M_g\|_F\to\infty$, and
\begin{equation}
\frac{1}{2\nu_g}\lambda_g^2
-
\left|
\partial_{\lambda_g} f(\widetilde M_g^{t-1},\widetilde\lambda_g^{t-1})
-
\frac{1}{\nu_g}\widetilde\lambda_g^{t-1}
+
r_g
\right|
\, |\lambda_g|
\to\infty
\end{equation}
as $|\lambda_g|\to\infty$. Hence,
\[
\widehat L_{\rho,g}^t(M_g,\lambda_g)\to\infty
\qquad\text{whenever}\qquad
\|M_g\|_F^2+|\lambda_g|^2\to\infty.
\]
Therefore, $\widehat L_{\rho,g}^t(M_g,\lambda_g)$ is coercive in $(M_g,\lambda_g)$.
\end{proof}

\begin{lemma}[Finite-horizon boundedness of the generated trajectory]\label{lem:bounded_traj}
Fix a finite training horizon $T\in \mathbb{N}$. Assume that, for each client $g$, the parameters satisfy
\[
\eta_g>0,\qquad \nu_g>0,\qquad \rho>0.
\]
Assume further that, for each client \(g=1,\dots,G\), the iterates are initialized as
\[
\widetilde M_g^0=0,\qquad
\widetilde\lambda_g^0=0,\qquad
\mu_g^0=0,\qquad
\tau_g^0=0,
\]
and that the global model is initialized as
\[
M^0=0.
\]
Then, for each $g$ and each $t=0,\dots,T$, the iterates
\[
\widetilde M_g^t,\qquad \widetilde\lambda_g^t,\qquad M^t,\qquad \mu_g^t,\qquad \tau_g^t
\]
are finite.

Consequently, the generated trajectory
\[
\mathcal T_T
:=
\left\{
(\widetilde M_g^t,\widetilde\lambda_g^t,M^t,\mu_g^t,\tau_g^t)
:\;
g=1,\dots,G,\; t=0,\dots,T
\right\}
\]
is a finite set, and therefore there exists a constant $E_T>0$ such that
\[
\|\widetilde M_g^t\|_F \le E_T,\qquad
|\widetilde\lambda_g^t| \le E_T,\qquad
\|M^t\|_F \le E_T,
\]
\[
\|\mu_g^t\|_F \le E_T,\qquad
|\tau_g^t| \le E_T,
\]
for all $g=1,\dots,G$ and all $t=0,\dots,T$.
\end{lemma}

\begin{proof}
We proceed by induction on the iteration index $t$.

For $t=0$, the initialization
\[
\widetilde M_g^0=0,\qquad
\widetilde\lambda_g^0=0,\qquad
M^0=0,\qquad
\mu_g^0=0,\qquad
\tau_g^0=0
\]
implies that all initial iterates are finite.

Now fix $t\in\{1,\dots,T\}$ and assume that, for every client $g$, the quantities
\[
\widetilde M_g^{t-1},\qquad \widetilde\lambda_g^{t-1},\qquad M^{t-1},\qquad \mu_g^{t-1},\qquad \tau_g^{t-1}
\]
are finite.By Lemma~\ref{lem:coercive}, for each client $g$, the local proximal linearized surrogate $\widehat L_{\rho,g}^t(M_g,\lambda_g)$
is coercive in $(M_g,\lambda_g)$. Since $\widehat L_{\rho,g}^t$ is also continuous in $(M_g,\lambda_g)$, it attains its minimum. Hence the local primal variables
\[
M_g^t,\qquad \lambda_g^t
\]
are finite for every client $g$, and so are the perturbed variables
\[
\widetilde M_g^t=M_g^t+\xi_g^t,\qquad \widetilde\lambda_g^t=\lambda_g^t+\zeta_g^t,
\]
since the Gaussian perturbations \(\xi_g^t\) and \(\zeta_g^t\) are finite.

Since the global model is updated according to
\[
M^t=\sum_{g=1}^G \alpha_g\bigl(\widetilde M_g^t+\mu_g^{t-1}\bigr),
\]
it follows that \(M^t\) is finite. Likewise, since the dual variable is updated as
\[
\mu_g^t=\mu_g^{t-1}+\widetilde M_g^t-M^t,
\]
it follows that \(\mu_g^t\) is finite. Finally, since
\[
\tau_g^t=\max\{0,\tau_g^{t-1}+\gamma(\phi(\widetilde M_g^t)-\widetilde\lambda_g^t)\},
\]
it follows that \(\tau_g^t\) is also finite, because it is obtained by applying finite-valued operations to previously computed finite quantities.

Thus, by induction, all iterates
\[
\widetilde M_g^t,\qquad \widetilde\lambda_g^t,\qquad M^t,\qquad \mu_g^t,\qquad \tau_g^t
\]
are finite for every $g=1,\dots,G$ and every $t=0,\dots,T$.

Since the index sets $\{1,\dots,G\}$ and $\{0,\dots,T\}$ are finite, the generated trajectory
\[
\mathcal T_T
=
\left\{
(\widetilde M_g^t,\widetilde\lambda_g^t,M^t,\mu_g^t,\tau_g^t)
:\;
g=1,\dots,G,\; t=0,\dots,T
\right\}
\]
contains finitely many points. Therefore, the maximum of the quantities
\[
\|\widetilde M_g^t\|_F,\qquad
|\widetilde\lambda_g^t|,\qquad
\|M^t\|_F,\qquad
\|\mu_g^t\|_F,\qquad
|\tau_g^t|
\]
over all $g=1,\dots,G$ and $t=0,\dots,T$ exists and is finite. Denote this maximum by $E_T$.
\end{proof}

\subsection{Proof of Theorem~\ref{thm:composition}}
\label{app:proof_comp_theorem}
\begin{proof}

Define the privacy losses associated with the perturbed model
update and local DRO variable at iteration \(t\) by
\begin{equation}
c_M(Z)
:=
\log\!\left(
\frac{p_M(Z)}{q_M(Z)}
\right),
\qquad
c_\lambda(z)
:=
\log\!\left(
\frac{p_\lambda(z)}{q_\lambda(z)}
\right),
\label{eq:comp_pf_1}
\end{equation}
where
\begin{equation}
\begin{aligned}
p_M(Z)
&:=
p_{\widetilde M_g^{\,t}\mid D}(Z),
&
q_M(Z)
&:=
p_{\widetilde M_g^{\,t}\mid D'}(Z),
\\
p_\lambda(z)
&:=
p_{\widetilde\lambda_g^{\,t}\mid D}(z),
&
q_\lambda(z)
&:=
p_{\widetilde\lambda_g^{\,t}\mid D'}(z),
\end{aligned}
\label{eq:comp_pf_2}
\end{equation}
are the corresponding densities under adjacent datasets
\(D\) and \(D'\).

For any real \(\tau>0\), define
\begin{equation}
\begin{aligned}
E_{M,1}
&:=
\mathbb E_{Z\sim p_M}
\left[
\left(\frac{p_M(Z)}{q_M(Z)}\right)^\tau
\right],
&
E_{M,2}
&:=
\mathbb E_{Z\sim q_M}
\left[
\left(\frac{q_M(Z)}{p_M(Z)}\right)^\tau
\right],
\\
E_{\lambda,1}
&:=
\mathbb E_{z\sim p_\lambda}
\left[
\left(\frac{p_\lambda(z)}{q_\lambda(z)}\right)^\tau
\right],
&
E_{\lambda,2}
&:=
\mathbb E_{z\sim q_\lambda}
\left[
\left(\frac{q_\lambda(z)}{p_\lambda(z)}\right)^\tau
\right].
\end{aligned}
\label{eq:comp_pf_3}
\end{equation}
The corresponding log moments are
\begin{equation}
\alpha_{M,g,t}(\tau)
:=
\log\max\{E_{M,1},E_{M,2}\},
\qquad
\alpha_{\lambda,g,t}(\tau)
:=
\log\max\{E_{\lambda,1},E_{\lambda,2}\}.
\label{eq:comp_pf_4}
\end{equation}

Consider first the perturbed model update
\begin{equation}
\widetilde M_g^{\,t}
=
M_g^{\,t}(D)+\xi_g^{\,t},
\qquad
\xi_g^{\,t}
\sim
\mathcal N\!\left(
0,\sigma_{M,g}^2 I_{KP}
\right).
\label{eq:comp_pf_release}
\end{equation}

The two neighboring output laws are Gaussian with identical covariance
\(\sigma_{M,g}^2 I_{KP}\) and means
\(\mu=\operatorname{vec}(M_g^{\,t}(D))\) and
\(\mu'=\operatorname{vec}(M_g^{\,t}(D'))\). By
Lemma 1,
\begin{equation}
\label{eq:comp_pf_sensitivity}
\begin{aligned}
\|\mu-\mu'\|_2
&=
\|M_g^{\,t}(D)-M_g^{\,t}(D')\|_F
\le
\Delta_{M,g},
\\
\Delta_{M,g}
&:=
\left(\rho+\frac{1}{\eta_g}\right)^{-1}
\frac{2\sqrt{2}L}{N_g}.
\end{aligned}
\end{equation}

By rotational invariance of the Gaussian distribution, the worst
case depends only on the mean separation
\(\Delta_{M,g}\), so the problem reduces to the
one-dimensional pair
\[
p_M(z_M)
=
\frac{1}{\sqrt{2\pi}\,\sigma_{M,g}}
\exp\!\left(
-\frac{z_M^2}{2\sigma_{M,g}^2}
\right),
\]
\[
q_M(z_M)
=
\frac{1}{\sqrt{2\pi}\,\sigma_{M,g}}
\exp\!\left(
-\frac{(z_M-\Delta_{M,g})^2}
{2\sigma_{M,g}^2}
\right).
\]
Hence,
\begin{equation}
\frac{p_M(z_M)}{q_M(z_M)}
=
\exp\!\left(
\frac{\Delta_{M,g}^2}
{2\sigma_{M,g}^2}
-
\frac{\Delta_{M,g}}
{\sigma_{M,g}^2}z_M
\right).
\label{eq:comp_pf_6}
\end{equation}

Substituting \eqref{eq:comp_pf_6} into \eqref{eq:comp_pf_3} gives
\begin{equation}
\mathbb E_{z_M\sim p_M}
\left[
\left(
\frac{p_M(z_M)}{q_M(z_M)}
\right)^\tau
\right]
=
\exp\!\left(
\tau
\frac{\Delta_{M,g}^2}
{2\sigma_{M,g}^2}
\right)
\mathbb E_{z_M\sim
\mathcal N(0,\sigma_{M,g}^2)}
\left[
\exp\!\left(
-\tau
\frac{\Delta_{M,g}}
{\sigma_{M,g}^2}
z_M
\right)
\right].
\label{eq:comp_pf_7}
\end{equation}

Using the moment generating function of a Gaussian random variable, $\mathbb E[e^{a z_M}]
=
\exp\!\left(
\frac{a^2\sigma_{M,g}^2}{2}
\right),$
with $a
=
-\tau
\frac{\Delta_{M,g}}
{\sigma_{M,g}^2},$
we obtain
\begin{equation}
E_{M,1}
=
\exp\!\left(
\frac{\tau(\tau+1)}{2}
\frac{\Delta_{M,g}^2}
{\sigma_{M,g}^2}
\right).
\label{eq:comp_pf_9a}
\end{equation}

Since the Gaussian moment-generating function is finite for every
real argument, the preceding calculation holds for every real
\(\tau>0\). By symmetry, \(E_{M,2}=E_{M,1}\).
Therefore, from the definition of
\(\alpha_{M,g,t}(\tau)\), we have
\begin{equation}
\alpha_{M,g,t}(\tau)
=
\log \max\{E_{M,1},E_{M,2}\}
\le
\frac{\tau(\tau+1)}{2}
\frac{\Delta_{M,g}^2}
{\sigma_{M,g}^2}.
\label{eq:comp_pf_9}
\end{equation}

Similarly, consider the perturbed local DRO variable
\begin{equation}
\widetilde{\lambda}_g^{\,t}
=
\lambda_g^{\,t}(D)
+
\zeta_g^{\,t},
\qquad
\zeta_g^{\,t}
\sim
\mathcal N\!\left(
0,\sigma_{\lambda,g}^2
\right).
\label{eq:comp_pf_lambda_release}
\end{equation}

The two neighboring output laws are Gaussian with identical variance
\(\sigma_{\lambda,g}^2\) and means
\(\lambda_g^{\,t}(D)\) and
\(\lambda_g^{\,t}(D')\). By
Lemma 2,
\begin{equation}
\left|
\lambda_g^{\,t}(D)
-
\lambda_g^{\,t}(D')
\right|
\le
\Delta_{\lambda,g},
\qquad
\Delta_{\lambda,g}
:=
\frac{\kappa\nu_g}{N_g}.
\label{eq:comp_pf_lambda_sensitivity}
\end{equation}

Since the mechanism is scalar-valued, the neighboring Gaussian
distributions can be reduced by translation to
\[
p_\lambda(z_\lambda)
=
\frac{1}{\sqrt{2\pi}\,\sigma_{\lambda,g}}
\exp\!\left(
-\frac{z_\lambda^2}
{2\sigma_{\lambda,g}^2}
\right),
\]
and
\[
q_\lambda(z_\lambda)
=
\frac{1}{\sqrt{2\pi}\,\sigma_{\lambda,g}}
\exp\!\left(
-\frac{(z_\lambda-\Delta_{\lambda,g})^2}
{2\sigma_{\lambda,g}^2}
\right).
\]

Hence,
\begin{equation}
\frac{p_\lambda(z_\lambda)}
{q_\lambda(z_\lambda)}
=
\exp\!\left(
\frac{\Delta_{\lambda,g}^2}
{2\sigma_{\lambda,g}^2}
-
\frac{\Delta_{\lambda,g}}
{\sigma_{\lambda,g}^2}
z_\lambda
\right).
\label{eq:comp_pf_lambda_ratio}
\end{equation}

Substituting \eqref{eq:comp_pf_lambda_ratio} into the definition of
\(E_{\lambda,1}\) gives
\begin{align}
E_{\lambda,1}
&=
\exp\!\left(
\tau
\frac{\Delta_{\lambda,g}^2}
{2\sigma_{\lambda,g}^2}
\right)
\mathbb E_{z_\lambda\sim
\mathcal N(0,\sigma_{\lambda,g}^2)}
\left[
\exp\!\left(
-\tau
\frac{\Delta_{\lambda,g}}
{\sigma_{\lambda,g}^2}
z_\lambda
\right)
\right]
\notag\\
&=
\exp\!\left(
\frac{\tau(\tau+1)}{2}
\frac{\Delta_{\lambda,g}^2}
{\sigma_{\lambda,g}^2}
\right).
\label{eq:comp_pf_lambda_moment}
\end{align}

By symmetry, \(E_{\lambda,2}=E_{\lambda,1}\). Therefore,
\begin{equation}
\alpha_{\lambda,g,t}(\tau)
=
\log\max\{E_{\lambda,1},E_{\lambda,2}\}
\le
\frac{\tau(\tau+1)}{2}
\frac{\Delta_{\lambda,g}^2}
{\sigma_{\lambda,g}^2}.
\label{eq:comp_pf_lambda}
\end{equation}

Since the Gaussian mechanisms are calibrated as
\[
\sigma_{M,g}
=
\Delta_{M,g}
\frac{\sqrt{2\ln(1.25/\delta_M)}}{\epsilon_M},
\qquad
\sigma_{\lambda,g}
=
\Delta_{\lambda,g}
\frac{\sqrt{2\ln(1.25/\delta_\lambda)}}{\epsilon_\lambda},
\]
it follows that
\begin{equation}
\alpha_{M,g,t}(\tau)
\le
\frac{\tau(\tau+1)\epsilon_M^2}
{4\ln(1.25/\delta_M)},
\label{eq:comp_pf_M}
\end{equation}
and
\begin{equation}
\alpha_{\lambda,g,t}(\tau)
\le
\frac{\tau(\tau+1)\epsilon_\lambda^2}
{4\ln(1.25/\delta_\lambda)}.
\label{eq:comp_pf_lambda_cal}
\end{equation}

We set \(\delta_M=\delta_\lambda=\delta\). By the moments accountant composition theorem of \citet{abadi2016deep}, the log moments of the two mechanisms add within each iteration and across the \(T\) iterations. Hence,
\begin{align}
\alpha_g(\tau)
&\le
\sum_{t=1}^T
\left[
\alpha_{M,g,t}(\tau)
+
\alpha_{\lambda,g,t}(\tau)
\right]
\notag\\
&\le
T\,
\frac{\tau(\tau+1)\left(\epsilon_M^2+\epsilon_\lambda^2\right)}
{4\ln(1.25/\delta)}.
\label{eq:comp_pf_13}
\end{align}

Applying the tail bound from \citet{abadi2016deep}, for any
\(\bar\epsilon>0\),
\begin{equation}
\inf_{\tau>0}
\exp\!\bigl(\alpha_g(\tau)-\tau\bar\epsilon\bigr)
\le
\inf_{\tau>0}
\exp\!\left(
T\,\frac{\tau(\tau+1)\left(\epsilon_M^2+\epsilon_\lambda^2\right)}
{4\ln(1.25/\delta)}
-
\tau\bar\epsilon
\right).
\label{eq:comp_pf_14}
\end{equation}

We now choose \(\bar\epsilon\) so that the tail bound in
\eqref{eq:comp_pf_14} is at most \(\bar\delta\). Define
\begin{equation}
A
:=
\frac{T\left(\epsilon_M^2+\epsilon_\lambda^2\right)}{4\ln(1.25/\delta)}.
\label{eq:comp_pf_15}
\end{equation}
Then \eqref{eq:comp_pf_13} can be written as
\begin{equation}
\alpha_g(\tau)
\le
A\tau(\tau+1).
\label{eq:comp_pf_16}
\end{equation}

It is sufficient that, for some \(\tau>0\),
\begin{equation}
\exp\!\left(\alpha_g(\tau)-\tau\bar\epsilon\right)
\le
\bar\delta .
\label{eq:comp_pf_17}
\end{equation}
Using the upper bound in \eqref{eq:comp_pf_16}, it is
sufficient that
\begin{align}
A\tau(\tau+1)-\tau\bar\epsilon
&\le
-\ln(1/\bar\delta),
\label{eq:comp_pf_18}\\
\tau\bar\epsilon
&\ge
A\tau(\tau+1)+\ln(1/\bar\delta),
\label{eq:comp_pf_19}\\
\bar\epsilon
&\ge
A(\tau+1)
+
\frac{\ln(1/\bar\delta)}{\tau}.
\label{eq:comp_pf_20}
\end{align}
Therefore, it is sufficient to choose
\begin{equation}
\bar\epsilon
\ge
\inf_{\tau>0}
\left\{
A(\tau+1)
+
\frac{\ln(1/\bar\delta)}{\tau}
\right\}.
\label{eq:comp_pf_21}
\end{equation}
Define
\begin{equation}
\varphi(\tau)
:=
A(\tau+1)
+
\frac{\ln(1/\bar\delta)}{\tau}.
\label{eq:comp_pf_22}
\end{equation}
The first-order condition gives
\begin{equation}
\varphi'(\tau)
=
A
-
\frac{\ln(1/\bar\delta)}{\tau^2}
=
0,
\qquad
\tau^\star
=
\sqrt{
\frac{\ln(1/\bar\delta)}{A}
}.
\label{eq:comp_pf_23}
\end{equation}
Substituting \(\tau^\star\) into \eqref{eq:comp_pf_22} yields
\begin{equation}
\inf_{\tau>0}\varphi(\tau)
=
\varphi(\tau^\star)
=
A+2\sqrt{A\ln(1/\bar\delta)}.
\label{eq:comp_pf_24}
\end{equation}
Thus, the adaptive composition over \(T\) iterations satisfies
\((\bar\epsilon,\bar\delta)\)-differential privacy for
\begin{equation}
\bar\epsilon
=
A
+
2\sqrt{A\ln(1/\bar\delta)}.
\label{eq:comp_pf_25}
\end{equation}
Substituting the definition of \(A\) from \eqref{eq:comp_pf_15}, we obtain
\begin{equation}
\bar\epsilon
=
\frac{T\left(\epsilon_M^2+\epsilon_\lambda^2\right)}{4\ln(1.25/\delta)}
+
\sqrt{
\frac{T\left(\epsilon_M^2+\epsilon_\lambda^2\right)\ln(1/\bar\delta)}
{\ln(1.25/\delta)}
}.
\label{eq:comp_pf_26}
\end{equation}
\end{proof}

\subsection{Proof of Theorem~\ref{thm:ergodic_convergence_dp_fed_dro}}
\label{app:proof_thm3}
\begin{proof}
Let \(v_g^t := (\widetilde M_g^t,\widetilde\lambda_g^t,M^t,\mu_g^t)\), and let $
\bar v_g^T := (\bar{\widetilde M}_g^T,\bar{\widetilde\lambda}_g^T,\bar M^T,\bar\mu_g^T)$
denote the corresponding average, where
\[
\bar{\widetilde M}_g^T := \frac{1}{T}\sum_{t=1}^{T}\widetilde M_g^t,
\quad
\bar{\widetilde\lambda}_g^T := \frac{1}{T}\sum_{t=1}^{T}\widetilde\lambda_g^t,
\quad
\bar M^T := \frac{1}{T}\sum_{t=1}^T M^t,
\quad
\bar\mu_g^T := \frac{1}{T}\sum_{t=1}^T \mu_g^t.
\]
Define
\[
H(v_g^t):=(\mu_g^t,\,0,\,-\mu_g^t,\,-(\widetilde M_g^t-M^t)).
\]
By direct expansion of the stacked inner product, we have
\begin{equation}
\langle v_g^t-v_g,\;H(v_g^t)\rangle
=
\langle \mu_g^t,\widetilde M_g^t-M_g\rangle_F
-\langle \mu_g^t,M^t-M\rangle_F
+\langle \mu_g-\mu_g^t,\widetilde M_g^t-M^t\rangle_F .
\label{eq:thm6_pf_1}
\end{equation}

To account for the update of $\tau_g$, write
\[
h_g^{t}(\cdot):=h_g(\cdot\,;\tau_g^{t-1})
\]
for the local objective used at round $t$, which is convex in $z_g$ for each fixed $\tau_g^{t-1}$.
Since $H$ is linear and monotone
\begin{equation}
\sum_{g=1}^G \alpha_g \langle \bar v_g^T-v_g,\;H(\bar v_g^T)\rangle
\le
\frac{1}{T}\sum_{t=1}^T \sum_{g=1}^G \alpha_g
\langle v_g^t-v_g,\;H(v_g^t)\rangle .
\label{eq:thm6_pf_3}
\end{equation}

Adding the per-round objective gaps to both sides of \eqref{eq:thm6_pf_3} and using
\eqref{eq:thm6_pf_1}, we obtain
\begin{multline}
\frac{1}{T}\sum_{t=1}^T \sum_{g=1}^G \alpha_g
\Bigl(h_g^{t}(\widetilde z_g^{\,t-1})-h_g^{t}(z_g)\Bigr)
+\rho\sum_{g=1}^G \alpha_g
\langle \bar v_g^T-v_g,\;H(\bar v_g^T)\rangle
\\
\le
\frac{1}{T}\sum_{t=1}^T \sum_{g=1}^G \alpha_g
\Bigl(
h_g^{t}(\widetilde z_g^{\,t-1})-h_g^{t}(z_g)
+\rho \langle v_g^t-v_g,\;H(v_g^t)\rangle
\Bigr)
\\
=
\frac{1}{T}\sum_{t=1}^T \sum_{g=1}^G \alpha_g
\Bigl(
h_g^{t}(\widetilde z_g^{\,t-1})-h_g^{t}(z_g)
+\rho\langle \mu_g^t,\widetilde M_g^t-M_g\rangle_F
\\
\qquad
-\rho\langle \mu_g^t,M^t-M\rangle_F
+\rho\langle \mu_g-\mu_g^t,\widetilde M_g^t-M^t\rangle_F
\Bigr)
\\
=
\frac{1}{T}\sum_{t=1}^T \sum_{g=1}^G \alpha_g
\Bigl(
h_g^{t}(\widetilde z_g^{\,t-1})-h_g^{t}(z_g)
+\rho\langle \mu_g^t,\widetilde M_g^t-M_g\rangle_F
\\
\qquad
+\rho\langle \mu_g-\mu_g^t,\widetilde M_g^t-M^t\rangle_F
\Bigr),
\label{eq:thm6_pf_4}
\end{multline}
where the last equality uses $\sum_{g=1}^G\alpha_g\mu_g^t=0$.

Applying Lemma 3 to each summand on the right-hand side of
\eqref{eq:thm6_pf_4} yields
\begin{multline}
\frac{1}{T}\sum_{t=1}^T \sum_{g=1}^G \alpha_g
\Bigl(h_g^{t}(\widetilde z_g^{\,t-1})-h_g^{t}(z_g)\Bigr)
+\rho\sum_{g=1}^G \alpha_g
\langle \bar v_g^T-v_g,\;H(\bar v_g^T)\rangle
\\
\le
\frac{1}{T}\sum_{t=1}^T
\sum_{g=1}^G \alpha_g
\Bigg[
\frac{\eta_g}{2}
\left\|
[h_g^{t\prime}(\widetilde z_g^{\,t-1})]_M
-\left(\rho+\frac{1}{\eta_g}\right)\xi_g^t
\right\|_F^2
+\frac{\nu_g}{2}
\left|
[h_g^{t\prime}(\widetilde z_g^{\,t-1})]_\lambda
-\frac{1}{\nu_g}\zeta_g^t
\right|^2
\\
\qquad
-\left(\rho+\frac{1}{\eta_g}\right)
\langle \xi_g^t,\,
M_g-\widetilde M_g^{\,t-1}\rangle_F
-\frac{1}{\nu_g}\,\zeta_g^t
\bigl(\lambda_g-\widetilde\lambda_g^{\,t-1}\bigr)
\\
\qquad
+\frac{1}{2\eta_g}
\Bigl(
\|M_g-\widetilde M_g^{\,t-1}\|_F^2
-\|M_g-\widetilde M_g^t\|_F^2
\Bigr)
\\
\qquad
+\frac{1}{2\nu_g}
\Bigl(
(\lambda_g-\widetilde\lambda_g^{\,t-1})^2
-(\lambda_g-\widetilde\lambda_g^t)^2
\Bigr)
\\
\qquad
+\frac{\rho}{2}
\Bigl(
\|M_g-M^{t-1}\|_F^2
-\|M_g-M^t\|_F^2
\Bigr)
+\frac{\rho}{2}
\Bigl(
\|\mu_g-\mu_g^{t-1}\|_F^2
-\|\mu_g-\mu_g^t\|_F^2
\Bigr)
\Bigg].
\label{eq:thm6_pf_5}
\end{multline}

Since the step sizes \(\eta_g\) and \(\nu_g\) are constant in \(t\), the
difference-of-squares terms in \eqref{eq:thm6_pf_5} telescope exactly, and
dropping the nonpositive terms
\(-\|M_g-\widetilde M_g^T\|_F^2\), \(-(\lambda_g-\widetilde\lambda_g^T)^2\),
\(-\|M_g-M^T\|_F^2\), and \(-\|\mu_g-\mu_g^T\|_F^2\) gives
\begin{multline}
\frac{1}{T}\sum_{t=1}^T \sum_{g=1}^G \alpha_g
\Bigl(h_g^{t}(\widetilde z_g^{\,t-1})-h_g^{t}(z_g)\Bigr)
+\rho\sum_{g=1}^G \alpha_g
\langle \bar v_g^T-v_g,\;H(\bar v_g^T)\rangle
\\
\le
\sum_{g=1}^G \alpha_g
\Biggl[
\frac{1}{T}\sum_{t=1}^T \frac{\eta_g}{2}
\left\|
[h_g^{t\prime}(\widetilde z_g^{\,t-1})]_M
-\left(\rho+\frac{1}{\eta_g}\right)\xi_g^t
\right\|_F^2
+\frac{1}{T}\sum_{t=1}^T \frac{\nu_g}{2}
\left|
[h_g^{t\prime}(\widetilde z_g^{\,t-1})]_\lambda
-\frac{1}{\nu_g}\zeta_g^t
\right|^2
\\
\qquad
-\frac{1}{T}\sum_{t=1}^T
\left(\rho+\frac{1}{\eta_g}\right)
\langle \xi_g^t,\;M_g-\widetilde M_g^{\,t-1}\rangle_F
-\frac{1}{T}\sum_{t=1}^T
\frac{1}{\nu_g}\,\zeta_g^t\bigl(\lambda_g-\widetilde\lambda_g^{\,t-1}\bigr)
\\
\qquad
+\frac{1}{2T\eta_g}\|M_g-\widetilde M_g^0\|_F^2
+\frac{1}{2T\nu_g}(\lambda_g-\widetilde\lambda_g^0)^2
+\frac{\rho}{2T}\|M_g-M^0\|_F^2
+\frac{\rho}{2T}\|\mu_g-\mu_g^0\|_F^2
\Biggr].
\label{eq:thm6_pf_6}
\end{multline}

Next, choose the primal comparison point to be an optimal solution,
\[
z_g=z_g^\star=(M_g^\star,\lambda_g^\star),
\qquad
M_g^\star=M^\star \text{ for all } g,
\qquad
M=M^\star .
\]
With this choice, \eqref{eq:thm6_pf_6} becomes
\begin{multline}
\frac{1}{T}\sum_{t=1}^T \sum_{g=1}^G \alpha_g
\Bigl(h_g^{t}(\widetilde z_g^{\,t-1})-h_g^{t}(z_g^\star)\Bigr)
+\rho\sum_{g=1}^G \alpha_g
\langle \bar v_g^T-v_g,\;H(\bar v_g^T)\rangle
\\
\le
\sum_{g=1}^G \alpha_g
\Biggl[
\frac{1}{T}\sum_{t=1}^T \frac{\eta_g}{2}
\left\|
[h_g^{t\prime}(\widetilde z_g^{\,t-1})]_M
-\left(\rho+\frac{1}{\eta_g}\right)\xi_g^t
\right\|_F^2
+\frac{1}{T}\sum_{t=1}^T \frac{\nu_g}{2}
\left|
[h_g^{t\prime}(\widetilde z_g^{\,t-1})]_\lambda
-\frac{1}{\nu_g}\zeta_g^t
\right|^2
\\
\qquad
-\frac{1}{T}\sum_{t=1}^T
\left(\rho+\frac{1}{\eta_g}\right)
\langle \xi_g^t,\;M^\star-\widetilde M_g^{\,t-1}\rangle_F
-\frac{1}{T}\sum_{t=1}^T
\frac{1}{\nu_g}\,\zeta_g^t\bigl(\lambda_g^\star-\widetilde\lambda_g^{\,t-1}\bigr)
\\
\qquad
+\frac{1}{2T\eta_g}\|M^\star-\widetilde M_g^0\|_F^2
+\frac{1}{2T\nu_g}(\lambda_g^\star-\widetilde\lambda_g^0)^2
+\frac{\rho}{2T}\|M^\star-M^0\|_F^2
+\frac{\rho}{2T}\|\mu_g-\mu_g^0\|_F^2
\Biggr],
\label{eq:thm6_pf_7}
\end{multline}
where \(v_g=(M^\star,\lambda_g^\star,M^\star,\mu_g)\).

Take the initialization
\[
M^0=0,
\qquad
\widetilde M_g^0=0,
\qquad
\widetilde\lambda_g^0=0,
\qquad
\mu_g^0=0
\quad \text{for all } g.
\]

Then \eqref{eq:thm6_pf_7} simplifies to
\begin{multline}
\frac{1}{T}\sum_{t=1}^T \sum_{g=1}^G \alpha_g
\Bigl(h_g^{t}(\widetilde z_g^{\,t-1})-h_g^{t}(z_g^\star)\Bigr)
+\rho\sum_{g=1}^G \alpha_g
\langle \bar v_g^T-v_g,\;H(\bar v_g^T)\rangle
\\
\le
\frac{1}{T}\sum_{t=1}^T \sum_{g=1}^G \alpha_g \frac{\eta_g}{2}
\left\|
[h_g'(\widetilde z_g^{\,t-1})]_M
-\left(\rho+\frac{1}{\eta_g}\right)\xi_g^t
\right\|_F^2
+\frac{1}{T}\sum_{t=1}^T \sum_{g=1}^G \alpha_g \frac{\nu_g}{2}
\left|
[h_g'(\widetilde z_g^{\,t-1})]_\lambda
-\frac{1}{\nu_g}\zeta_g^t
\right|^2
\\
\qquad
-\frac{1}{T}\sum_{t=1}^T \sum_{g=1}^G \alpha_g
\left(\rho+\frac{1}{\eta_g}\right)
\langle \xi_g^t,\;M^\star-\widetilde M_g^{\,t-1}\rangle_F
-\frac{1}{T}\sum_{t=1}^T \sum_{g=1}^G \alpha_g
\frac{1}{\nu_g}\,\zeta_g^t\bigl(\lambda_g^\star-\widetilde\lambda_g^{\,t-1}\bigr)
\\
\qquad
+\frac{\|M^\star\|_F^2}{T}\sum_{g=1}^G \alpha_g \frac{1}{2\eta_g}
+\frac{1}{T}\sum_{g=1}^G \alpha_g \frac{(\lambda_g^\star)^2}{2\nu_g}
+\frac{\rho}{2T}\|M^\star\|_F^2
+\frac{\rho}{2T}\sum_{g=1}^G \alpha_g \|\mu_g\|_F^2 .
\label{eq:thm6_pf_8}
\end{multline}

We now simplify the term \(\rho \langle \bar v_g^T-v_g,\;H(\bar v_g^T)\rangle\).
For a comparison point \(v_g=(M,\lambda_g,M,\mu_g)\), we have
\begin{multline}
\langle \bar v_g^T-v_g,\;H(\bar v_g^T)\rangle
=
\langle \bar\mu_g^T,\bar{\widetilde M}_g^T-M\rangle_F
-\langle \bar\mu_g^T,\bar M^T-M\rangle_F
-\langle \bar\mu_g^T-\mu_g,\bar{\widetilde M}_g^T-\bar M^T\rangle_F
\\
=
\langle \mu_g,\bar{\widetilde M}_g^T-\bar M^T\rangle_F.
\label{eq:thm6_pf_9}
\end{multline}
Therefore, the left-hand side of \eqref{eq:thm6_pf_8} equals
\begin{equation}
\frac{1}{T}\sum_{t=1}^T \sum_{g=1}^G \alpha_g
\Bigl(h_g^{t}(\widetilde z_g^{\,t-1})-h_g^{t}(z_g^\star)\Bigr)
+\rho\sum_{g=1}^G \alpha_g
\langle \mu_g,\bar{\widetilde M}_g^T-\bar M^T\rangle_F .
\label{eq:thm6_pf_10}
\end{equation}

Now restrict the comparison dual variable to the ball
\[
\{\mu_g:\|\mu_g\|_F\le \beta/\rho\}.
\]
Let
\[
d_g^T:=\bar{\widetilde M}_g^T-\bar M^T.
\]
By Cauchy--Schwarz,
\begin{equation}
\max_{\|\mu_g\|_F\le \beta/\rho}
\rho\langle \mu_g,d_g^T\rangle_F
\le
\rho \max_{\|\mu_g\|_F\le \beta/\rho}
\|\mu_g\|_F\,\|d_g^T\|_F
=
\beta \|d_g^T\|_F.
\label{eq:thm6_pf_11}
\end{equation}

The bound is tight by choosing
\[
\hat\mu_g = \frac{\beta}{\rho}\frac{d_g^T}{\|d_g^T\|_F}
\qquad \text{when } d_g^T\neq 0,
\]
hence
\begin{equation}
\max_{\|\mu_g\|_F\le \beta/\rho}
\rho\langle \mu_g,\bar{\widetilde M}_g^T-\bar M^T\rangle_F
=
\beta \|\bar{\widetilde M}_g^T-\bar M^T\|_F.
\label{eq:thm6_pf_12}
\end{equation}

Consequently, using \eqref{eq:thm6_pf_10} and \eqref{eq:thm6_pf_12},
\begin{multline}
\max_{\substack{\|\mu_g\|_F\le \beta/\rho\\ g=1,\dots,G}}
\Biggl[
\frac{1}{T}\sum_{t=1}^T \sum_{g=1}^G \alpha_g
\Bigl(h_g^{t}(\widetilde z_g^{\,t-1})-h_g^{t}(z_g^\star)\Bigr)
+\rho\sum_{g=1}^G \alpha_g
\langle \bar v_g^T-v_g,\;H(\bar v_g^T)\rangle
\Biggr]
\\
=
\frac{1}{T}\sum_{t=1}^T \sum_{g=1}^G \alpha_g
\Bigl(h_g^{t}(\widetilde z_g^{\,t-1})-h_g^{t}(z_g^\star)\Bigr)
+\beta\sum_{g=1}^G \alpha_g
\|\bar{\widetilde M}_g^T-\bar M^T\|_F .
\label{eq:thm6_pf_13}
\end{multline}
We now maximize the last term on the right-hand side of \eqref{eq:thm6_pf_8}
over the same ball. Since \(\|\mu_g\|_F\le \beta/\rho\) and \(\sum_{g=1}^G\alpha_g=1\),
\begin{equation}
\max_{\substack{\|\mu_g\|_F\le \beta/\rho\\ g=1,\dots,G}}
\frac{\rho}{2T}\sum_{g=1}^G \alpha_g \|\mu_g\|_F^2
\le
\frac{\rho}{2T}\sum_{g=1}^G \alpha_g \frac{\beta^2}{\rho^2}
=
\frac{\beta^2}{2\rho T}.
\label{eq:thm6_pf_14}
\end{equation}

Since \eqref{eq:thm6_pf_8} holds for every \(\mu_g\), we may take the maximum
over \(\|\mu_g\|_F\le\beta/\rho\), \(g=1,\dots,G\), on both sides. Using
\eqref{eq:thm6_pf_13} for the left-hand side and \eqref{eq:thm6_pf_14} for the
only \(\mu_g\)-dependent term on the right-hand side, we obtain
\begin{multline}
\frac{1}{T}\sum_{t=1}^T \sum_{g=1}^G \alpha_g
\Bigl(h_g^{t}(\widetilde z_g^{\,t-1})-h_g^{t}(z_g^\star)\Bigr)
+\beta\sum_{g=1}^G \alpha_g
\|\bar{\widetilde M}_g^T-\bar M^T\|_F
\\
\le
\frac{1}{T}\sum_{t=1}^T \sum_{g=1}^G \alpha_g \frac{\eta_g}{2}
\left\|
[h_g^{t\prime}(\widetilde z_g^{\,t-1})]_M
-\left(\rho+\frac{1}{\eta_g}\right)\xi_g^t
\right\|_F^2
+\frac{1}{T}\sum_{t=1}^T \sum_{g=1}^G \alpha_g \frac{\nu_g}{2}
\left|
[h_g^{t\prime}(\widetilde z_g^{\,t-1})]_\lambda
-\frac{1}{\nu_g}\zeta_g^t
\right|^2
\\
\qquad
-\frac{1}{T}\sum_{t=1}^T \sum_{g=1}^G \alpha_g
\left(\rho+\frac{1}{\eta_g}\right)
\langle \xi_g^t,\;M^\star-\widetilde M_g^{\,t-1}\rangle_F
-\frac{1}{T}\sum_{t=1}^T \sum_{g=1}^G \alpha_g
\frac{1}{\nu_g}\,\zeta_g^t\bigl(\lambda_g^\star-\widetilde\lambda_g^{\,t-1}\bigr)
\\
\qquad
+\frac{\|M^\star\|_F^2}{T}\sum_{g=1}^G \alpha_g \frac{1}{2\eta_g}
+\frac{1}{2T}\sum_{g=1}^G \alpha_g
\frac{(\lambda_g^\star)^2}{\nu_g}
+\frac{\rho}{2T}\|M^\star\|_F^2
+\frac{\beta^2}{2\rho T}.
\label{eq:thm6_pf_15}
\end{multline}

We now take expectations on both sides of \eqref{eq:thm6_pf_15}. By construction
of Algorithm~2, for each \(g\) and \(t\), the perturbations \(\xi_g^t\) and
\(\zeta_g^t\) are drawn independently of each other and of all past iterates, and satisfy
\[
\mathbb{E}[\xi_g^t]=0,\qquad
\mathbb{E}\|\xi_g^t\|_F^2=KP\,\sigma_{M,g}^2,\qquad
\mathbb{E}[\zeta_g^t]=0,\qquad
\mathbb{E}(\zeta_g^t)^2=\sigma_{\lambda,g}^2 .
\]
Since \(M^\star-\widetilde M_g^{\,t-1}\) and \(\lambda_g^\star-\widetilde\lambda_g^{\,t-1}\)
depend only on past iterates, while \(\xi_g^t\) and \(\zeta_g^t\) are independent of the
past and have zero mean, it follows that
\begin{equation}
\mathbb{E}\langle \xi_g^t,\; M^\star-\widetilde M_g^{\,t-1}\rangle_F=0,
\qquad
\mathbb{E}\bigl[\zeta_g^t\,(\lambda_g^\star-\widetilde\lambda_g^{\,t-1})\bigr]=0.
\label{eq:thm6_pf_16}
\end{equation}

Recall that, at round \(t\),
\[
h_g^{t}(z_g)=\lambda_g r_g+\frac{1}{N_g}\sum_{n=1}^{N_g}L_n(M_g,\lambda_g)+R_g(M_g,\lambda_g;\tau_g^{t-1}),
\qquad z_g=(M_g,\lambda_g),
\]
where
\[
L_n(M_g,\lambda_g)
:=
\max_{y\in\mathcal Y}
\Bigl(
\ell_{\mathrm{CS}}(M_g;x_n,y)-\kappa\lambda_g\,\mathbf{1}_{\{y\neq y_n\}}
\Bigr),
\]
and
\[
R_g(M_g,\lambda_g;\tau_g^{t-1})
:=
\frac{1}{2\gamma}
\Bigl(
\max\{0,\tau_g^{t-1}+\gamma(\phi(M_g)-\lambda_g)\}^2-(\tau_g^{t-1})^2
\Bigr).
\]

Next, define
\[
a_g^{t-1}:=[h_g^{t\prime}(\widetilde z_g^{\,t-1})]_M,
\qquad
b_g^{t-1}:=[h_g^{t\prime}(\widetilde z_g^{\,t-1})]_\lambda,
\qquad
c_g:=\rho+\frac{1}{\eta_g}.
\]
By direct evaluation of the selected subgradients of the robust loss term,
\[
\|\partial_{M_g}L_n(M_g,\lambda_g)\|_F \le \sqrt{2}\,L,
\qquad
\partial_{\lambda_g}L_n(M_g,\lambda_g)\in[-\kappa,0].
\]
Since \([h_g^{t\prime}(\widetilde z_g^{\,t-1})]_\lambda
= r_g+\frac{1}{N_g}\sum_{n=1}^{N_g}\partial_{\lambda_g}L_n+\partial_{\lambda_g}R_g\),
the first two terms together lie in \([r_g-\kappa,\,r_g]\), so that
\[
\Bigl|\,r_g+\frac{1}{N_g}\sum_{n=1}^{N_g}\partial_{\lambda_g}L_n\Bigr|
\le S:=\max\{r_g,\ \kappa-r_g\}.
\]

Moreover, since
\(\partial_{\lambda_g}R_g=-\max\{0,\tau_g^{t-1}+\gamma(\phi(M_g)-\lambda_g)\}\) and
\(\partial_{M_g}R_g=\max\{0,\tau_g^{t-1}+\gamma(\phi(M_g)-\lambda_g)\}\,\partial\phi(M_g)\),
using \(\phi(M_g)\le\sqrt{2}\,\|M_g\|_F\) and \(\|\partial\phi(M_g)\|_F\le C_\phi\) gives
\[
\|\partial_{M_g}R_g(M_g,\lambda_g;\tau_g^{t-1})\|_F
\le
\bigl(|\tau_g^{t-1}|+\gamma\sqrt{2}\,\|M_g\|_F+\gamma[-\lambda_g]_+\bigr)C_\phi,
\]
and
\[
|\partial_{\lambda_g}R_g(M_g,\lambda_g;\tau_g^{t-1})|
\le
|\tau_g^{t-1}|+\gamma\sqrt{2}\,\|M_g\|_F+\gamma[-\lambda_g]_+ .
\]

Further, Lemma~\ref{lem:bounded_traj} shows that the trajectory generated by Algorithm~2 remains finite
over the training horizon, so that the penalty subgradients evaluated along it are
bounded. Accordingly, we assume that these bounds are uniform: there exist constants
\(C,V>0\) such that
\[
\bigl\|\partial_{M_g}R_g(\widetilde M_g^{\,t-1},\widetilde\lambda_g^{\,t-1};\tau_g^{t-1})\bigr\|_F
\le \sqrt{2}\,C,
\qquad
\bigl|\partial_{\lambda_g}R_g(\widetilde M_g^{\,t-1},\widetilde\lambda_g^{\,t-1};\tau_g^{t-1})\bigr|
\le V,
\]
for all \(g=1,\dots,G\) and \(t=1,\dots,T\).

Therefore,
\begin{equation}
\|a_g^{t-1}\|_F \le \sqrt{2}\,(L+C),
\qquad
|b_g^{t-1}| \le S+V.
\label{eq:thm6_pf_17}
\end{equation}

Expanding the square and using that \(\xi_g^t\) is zero-mean and independent of
\(a_g^{t-1}\), which depends only on past iterates, gives
\begin{equation}
\mathbb E\|a_g^{t-1}-c_g\xi_g^t\|_F^2
=
\mathbb E\|a_g^{t-1}\|_F^2
-2c_g\,\mathbb E\langle a_g^{t-1},\xi_g^t\rangle_F
+c_g^2\,\mathbb E\|\xi_g^t\|_F^2
=
\mathbb E\|a_g^{t-1}\|_F^2 + c_g^2 KP\,\sigma_{M,g}^2 ,
\end{equation}
and likewise, since \(\zeta_g^t\) is zero-mean and independent of \(b_g^{t-1}\),
\begin{equation}
\mathbb E\Bigl|b_g^{t-1}-\frac{1}{\nu_g}\zeta_g^t\Bigr|^2
=
\mathbb E|b_g^{t-1}|^2+\frac{\sigma_{\lambda,g}^2}{\nu_g^2} .
\end{equation}
Using \eqref{eq:thm6_pf_17}, we obtain
\begin{equation}
\mathbb E\|a_g^{t-1}-c_g\xi_g^t\|_F^2
\le
2(L+C)^2 + c_g^2 KP\,\sigma_{M,g}^2,
\qquad
\mathbb E\Bigl|b_g^{t-1}-\frac{1}{\nu_g}\zeta_g^t\Bigr|^2
\le
(S+V)^2+\frac{\sigma_{\lambda,g}^2}{\nu_g^2} .
\label{eq:thm6_pf_18}
\end{equation}

Taking expectations in \eqref{eq:thm6_pf_15} and using
\eqref{eq:thm6_pf_16}, \eqref{eq:thm6_pf_17}, and \eqref{eq:thm6_pf_18}, we get
\begin{multline}
\mathbb E\Biggl[
\frac{1}{T}\sum_{t=1}^T \sum_{g=1}^G \alpha_g
\Bigl(h_g^{t}(\widetilde z_g^{\,t-1})-h_g^{t}(z_g^\star)\Bigr)
+\beta\sum_{g=1}^G \alpha_g
\|\bar{\widetilde M}_g^T-\bar M^T\|_F
\Biggr]
\\
\le
\sum_{g=1}^G \alpha_g
\frac{\eta_g}{2}
\Biggl(
2(L+C)^2+\left(\rho+\frac{1}{\eta_g}\right)^2 KP\,\sigma_{M,g}^2
\Biggr)
+
\sum_{g=1}^G \alpha_g
\frac{\nu_g}{2}
\Biggl(
(S+V)^2+\frac{\sigma_{\lambda,g}^2}{\nu_g^2}
\Biggr)
\\
+\frac{\|M^\star\|_F^2}{T}\sum_{g=1}^G \alpha_g \frac{1}{2\eta_g}
+\frac{1}{2T}\sum_{g=1}^G \alpha_g
\frac{(\lambda_g^\star)^2}{\nu_g}
+\frac{\rho}{2T}\|M^\star\|_F^2
+\frac{\beta^2}{2\rho T}.
\label{eq:thm6_pf_19}
\end{multline}

Now substitute the DP-calibrated Gaussian variances from Theorem~1,
\[
\sigma_{M,g}^2
=
\frac{8L^2}{N_g^2\left(\rho+\frac{1}{\eta_g}\right)^2}
\frac{2\ln(1.25/\delta_M)}{\epsilon_M^2},
\qquad
\sigma_{\lambda,g}^2
=
\frac{\kappa^2\nu_g^2}{N_g^2}
\frac{2\ln(1.25/\delta_\lambda)}{\epsilon_\lambda^2}.
\]
Then
\begin{equation}
\left(\rho+\frac{1}{\eta_g}\right)^2 KP\,\sigma_{M,g}^2
=
KP\cdot \frac{8L^2}{N_g^2}\cdot \frac{2\ln(1.25/\delta_M)}{\epsilon_M^2},
\qquad
\frac{\sigma_{\lambda,g}^2}{\nu_g^2}
=
\frac{\kappa^2}{N_g^2}\cdot\frac{2\ln(1.25/\delta_\lambda)}{\epsilon_\lambda^2}.
\label{eq:thm6_pf_20}
\end{equation}

Substituting \eqref{eq:thm6_pf_20} into \eqref{eq:thm6_pf_19} yields
\begin{multline}
\mathbb E\Biggl[
\frac{1}{T}\sum_{t=1}^T \sum_{g=1}^G \alpha_g
\Bigl(h_g^{t}(\widetilde z_g^{\,t-1})-h_g^{t}(z_g^\star)\Bigr)
+\beta\sum_{g=1}^G \alpha_g
\|\bar{\widetilde M}_g^T-\bar M^T\|_F
\Biggr]
\\
\le
\sum_{g=1}^G \alpha_g
\frac{\eta_g}{2}
\left(
2(L+C)^2
+
KP\cdot \frac{8L^2}{N_g^2}\cdot \frac{2\ln(1.25/\delta_M)}{\epsilon_M^2}
\right)
+
\sum_{g=1}^G \alpha_g
\frac{\nu_g}{2}
\left(
(S+V)^2
+
\frac{\kappa^2}{N_g^2}\cdot\frac{2\ln(1.25/\delta_\lambda)}{\epsilon_\lambda^2}
\right)
\\
+\frac{\|M^\star\|_F^2}{T}\sum_{g=1}^G \alpha_g \frac{1}{2\eta_g}
+\frac{1}{2T}\sum_{g=1}^G \alpha_g
\frac{(\lambda_g^\star)^2}{\nu_g}
+\frac{\rho}{2T}\|M^\star\|_F^2
+\frac{\beta^2}{2\rho T}.
\label{eq:thm6_pf_21}
\end{multline}

For constants \(c_M,c_{\lambda,g}>0\) satisfying \(\|M^\star\|_F\le c_M\) and
\(|\lambda_g^\star|\le c_{\lambda,g}\) for all \(g\), define the \(M_g\)-block and
\(\lambda_g\)-block step sizes as
\begin{equation}
\begin{aligned}
\eta_g
&:=
\frac{c_M}
{
\sqrt{
T\left(
2(L+C)^2
+
KP\cdot \frac{8L^2}{N_g^2}\cdot \frac{2\ln(1.25/\delta_M)}{\epsilon_M^2}
\right)
}
},
\\
\nu_g
&:=
\frac{c_{\lambda,g}}
{
\sqrt{
T\left(
(S+V)^2
+
\frac{\kappa^2}{N_g^2}\cdot\frac{2\ln(1.25/\delta_\lambda)}{\epsilon_\lambda^2}
\right)
}
}.
\end{aligned}
\label{eq:thm6_pf_22}
\end{equation}
For convenience, define
\begin{equation}
D_{M,g}
:=
\sqrt{
2(L+C)^2
+
KP\cdot \frac{8L^2}{N_g^2}\cdot \frac{2\ln(1.25/\delta_M)}{\epsilon_M^2}
},
\qquad
D_{\lambda,g}
:=
\sqrt{
(S+V)^2
+
\frac{\kappa^2}{N_g^2}\cdot\frac{2\ln(1.25/\delta_\lambda)}{\epsilon_\lambda^2}
}.
\label{eq:thm6_pf_23}
\end{equation}

Substituting \eqref{eq:thm6_pf_22} into \eqref{eq:thm6_pf_21} and using
\(\|M^\star\|_F^2\le c_M^2\) and \((\lambda_g^\star)^2\le c_{\lambda,g}^2\), we obtain
\begin{multline}
\mathbb E\Biggl[
\frac{1}{T}\sum_{t=1}^T \sum_{g=1}^G \alpha_g
\Bigl(h_g^{t}(\widetilde z_g^{\,t-1})-h_g^{t}(z_g^\star)\Bigr)
+\beta\sum_{g=1}^G \alpha_g
\|\bar{\widetilde M}_g^T-\bar M^T\|_F
\Biggr]
\\
\le
\frac{1}{\sqrt T}\sum_{g=1}^G \alpha_g\, c_M D_{M,g}
+
\frac{1}{\sqrt T}\sum_{g=1}^G \alpha_g\, c_{\lambda,g} D_{\lambda,g}
+
\frac{\rho}{2T}c_M^2
+\frac{\beta^2}{2\rho T}.
\label{eq:thm6_pf_24}
\end{multline}

Defining
\begin{equation}
M_{1,T}
:=
\sum_{g=1}^G \alpha_g\, c_M D_{M,g}
+
\sum_{g=1}^G \alpha_g\, c_{\lambda,g} D_{\lambda,g},
\qquad
M_2
:=
\frac{\rho}{2}c_M^2+\frac{\beta^2}{2\rho},
\label{eq:thm6_pf_26}
\end{equation}
we can rewrite (122) as
\begin{multline}
\mathbb E\Biggl[
\frac{1}{T}\sum_{t=1}^T \sum_{g=1}^G \alpha_g
\Bigl(h_g^{t}(\widetilde z_g^{\,t-1})-h_g^{t}(z_g^\star)\Bigr)
+\beta\sum_{g=1}^G \alpha_g
\|\bar{\widetilde M}_g^T-\bar M^T\|_F
\Biggr]
\\
\le
\frac{1}{\sqrt T}
\Biggl[
\sum_{g=1}^G \alpha_g\, c_M D_{M,g}
+
\sum_{g=1}^G \alpha_g\, c_{\lambda,g} D_{\lambda,g}
\Biggr]
+
\frac{1}{T}
\Bigl(
\frac{\rho}{2}c_M^2+\frac{\beta^2}{2\rho}
\Bigr)
=
\frac{M_{1,T}}{\sqrt T}+\frac{M_2}{T}.
\label{eq:thm6_pf_27}
\end{multline}
This proves the claimed convergence bound.
\end{proof}

\section{Further Experimental Details and Supplementary Results}

In this section we provide all the details associated with the experiments presented in this paper as well as supplementary results to the experiments. All the code associated with the experiments will be made publicly available upon acceptance. 

\subsection{Software and Hardware Details}
All experiments were conducted on a high-performance computing cluster using Intel Xeon Gold 6226 processors, with two CPU cores allocated per job-array task. Table~\ref{tab:software} presents all the details on the software used in the paper.

\begin{table}[H]
\centering
\caption{Software environment used in the experiments.}
\label{tab:software}
\begin{tabular}{lll}
\toprule
Software & Version & License \\
\midrule
Gurobi       & 11.0.3 & Academic license \\
MATLAB & R2025b & Academic license \\
Python       & 3.8.3  & Open source license \\
Scikit-Learn & 0.23.1 & Open source license \\
NumPy        & 1.18.5 & Open source license \\
SciPy        & 1.5.0  & Open source license \\
UCIMLRepo    & 0.0.7  & Open source license \\
Matplotlib   & 3.2.2  & Open source license \\
Pandas       & 1.0.5  & Open source license \\
\bottomrule
\end{tabular}
\end{table}

\subsection{Datasets Utilized}

\subsubsection{Simulated Industrial Fault Diagnosis Dataset}
The simulated pump dataset was generated using the pump fault diagnosis model in MATLAB/Simulink \citep{mathworks_pump_fault_example}. We considered four operating conditions: healthy operation, cylinder leak, blocked inlet, and increased bearing friction. Only single fault conditions were considered, so that at most one fault was active in each simulation run. The simulator parameterizes each fault on a grid of 10 severity levels, the first of which corresponds to nominal operation. For each fault type, we sampled from a subset of the upper severity levels (levels 5 and 6 for the cylinder leak, levels 3 and 4 for the blocked inlet, and levels 7 to 9 for bearing friction), so that the three faults differ in difficulty while remaining distinguishable from healthy operation. The dataset contains 2{,}000 simulation runs, with 750 healthy samples and 418, 418, and 414 samples for the cylinder leak, blocked inlet, and bearing friction classes, respectively. The simulated time series data were then processed to extract 13 tabular features, which were subsequently used as inputs for the classification models.

\subsubsection{UCI Datasets} 

In addition to the simulated pump dataset, we evaluated the proposed method on five public benchmark datasets from the UCI Machine Learning Repository. Table~\ref{tab:uci_datasets} summarizes the benchmark datasets considered in this study.

\begin{table}[H]
\centering
\caption{Public benchmark datasets used in the experiments.}
\label{tab:uci_datasets}
\begin{tabular}{lcc}
\toprule
Dataset & Number of Classes & License \\
\midrule
Dermatology~\citep{uci_dermatology} & 6 & CC BY 4.0 \\
Iris~\citep{uci_iris} & 3 & CC BY 4.0 \\
Wine~\citep{uci_wine} & 3 & CC BY 4.0 \\
Palmer Penguins~\citep{palmer_penguins_data} & 3 & CC0 \\
Seeds~\citep{uci_seeds} & 3 & CC BY 4.0 \\
\bottomrule
\end{tabular}
\end{table}

\subsection{Additional Convergence Figures}

This section reports additional convergence figures that decompose the combined metric used in the main text into its two components, the running average of the augmented objective at the privatized iterates and the consensus violation, for all considered values of \(\epsilon\) and \(\delta\). The consensus violation rises during the first few rounds, as the initial noisy updates move the clients apart, and then decreases steadily toward zero for all privacy levels, with a smaller peak and faster decay as \(\epsilon\) increases. The augmented objective, in turn, reveals two regimes. At the smallest budget, \(\epsilon=0.05\), the injected noise dominates the updates, and the objective keeps increasing before settling at a higher level, which is highest for the tightest values of \(\delta\). As \(\epsilon\) increases and the perturbations become smaller, this growth vanishes. At \(\epsilon=0.1\), the objective remains nearly flat after the first rounds, and for \(\epsilon\ge0.15\) it decreases steadily below its initial peak, with the curves for different \(\delta\) nearly coinciding, which is consistent with convergence. In both components, the effect of \(\delta\) is most visible at the smallest budget and fades as \(\epsilon\) increases.

\begin{figure}[H]
    \centering

    % ============================================================
    % Row 1: Objective gap plots
    % ============================================================
    \begin{minipage}[t]{0.19\textwidth}
        \centering
        \includegraphics[width=\linewidth]{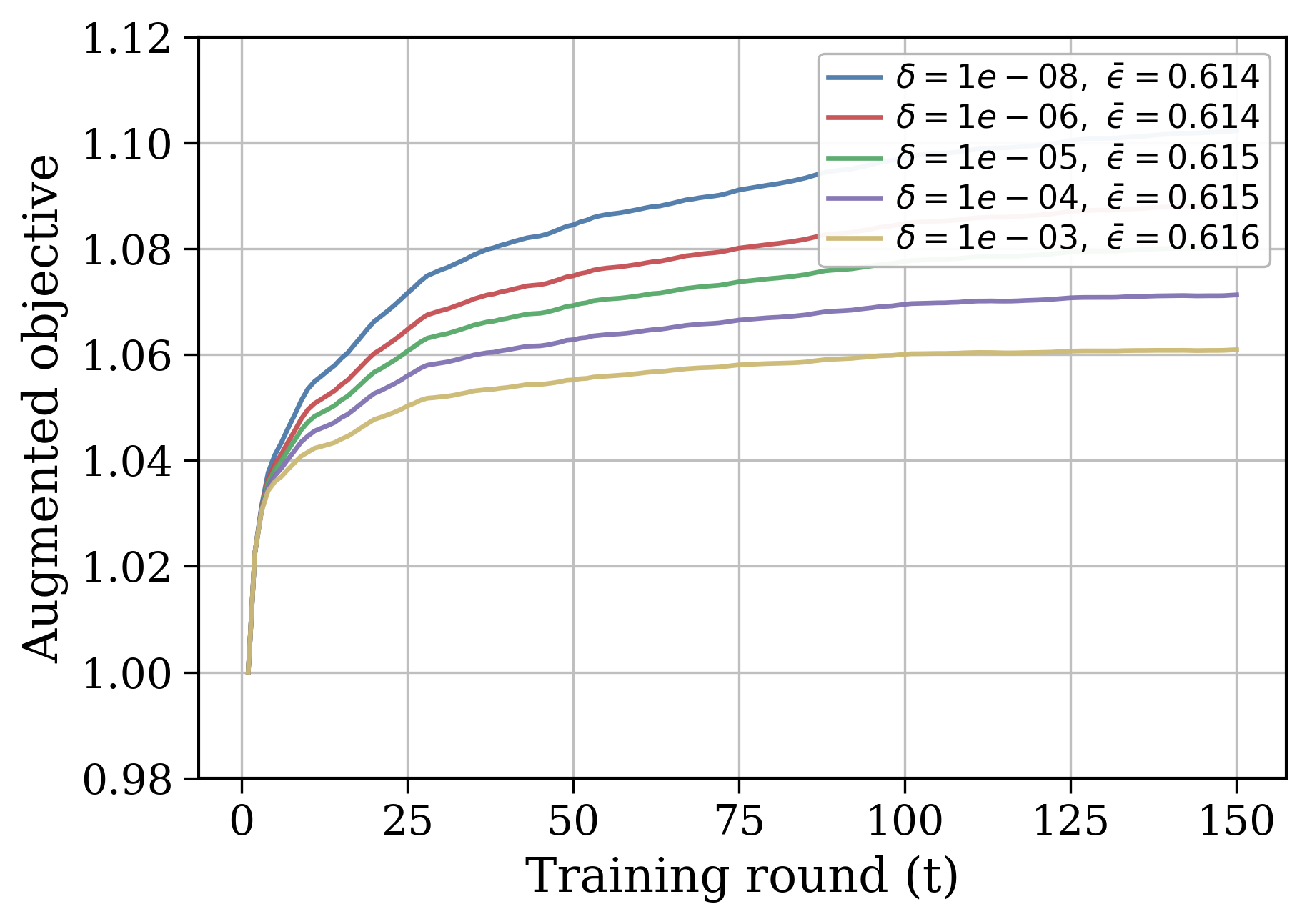}
        
        \vspace{2pt}
        \footnotesize $\epsilon = 0.05$
    \end{minipage}
    \hfill
    \begin{minipage}[t]{0.19\textwidth}
        \centering
        \includegraphics[width=\linewidth]{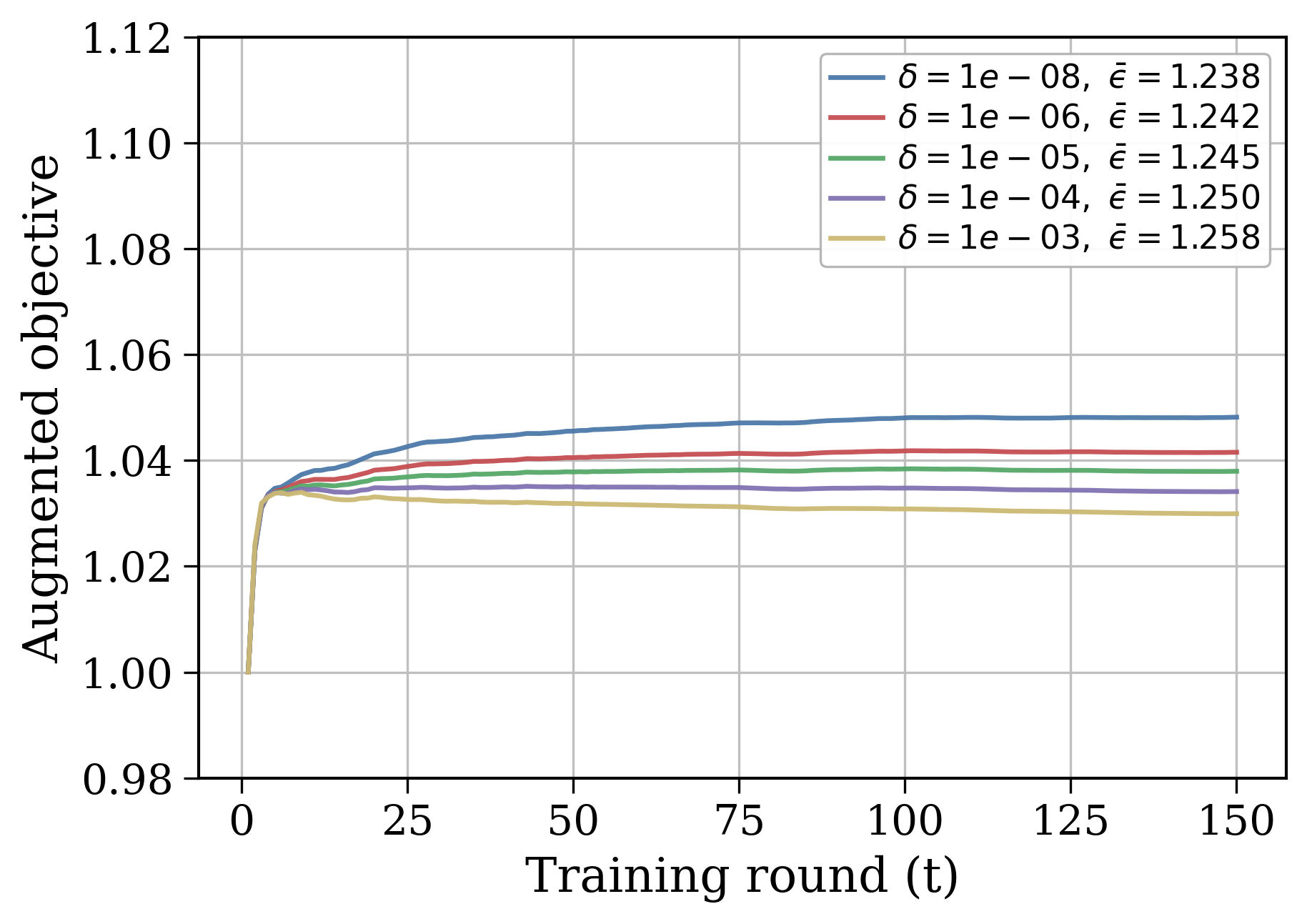}
        
        \vspace{2pt}
        \footnotesize $\epsilon = 0.1$
    \end{minipage}
    \hfill
    \begin{minipage}[t]{0.19\textwidth}
        \centering
        \includegraphics[width=\linewidth]{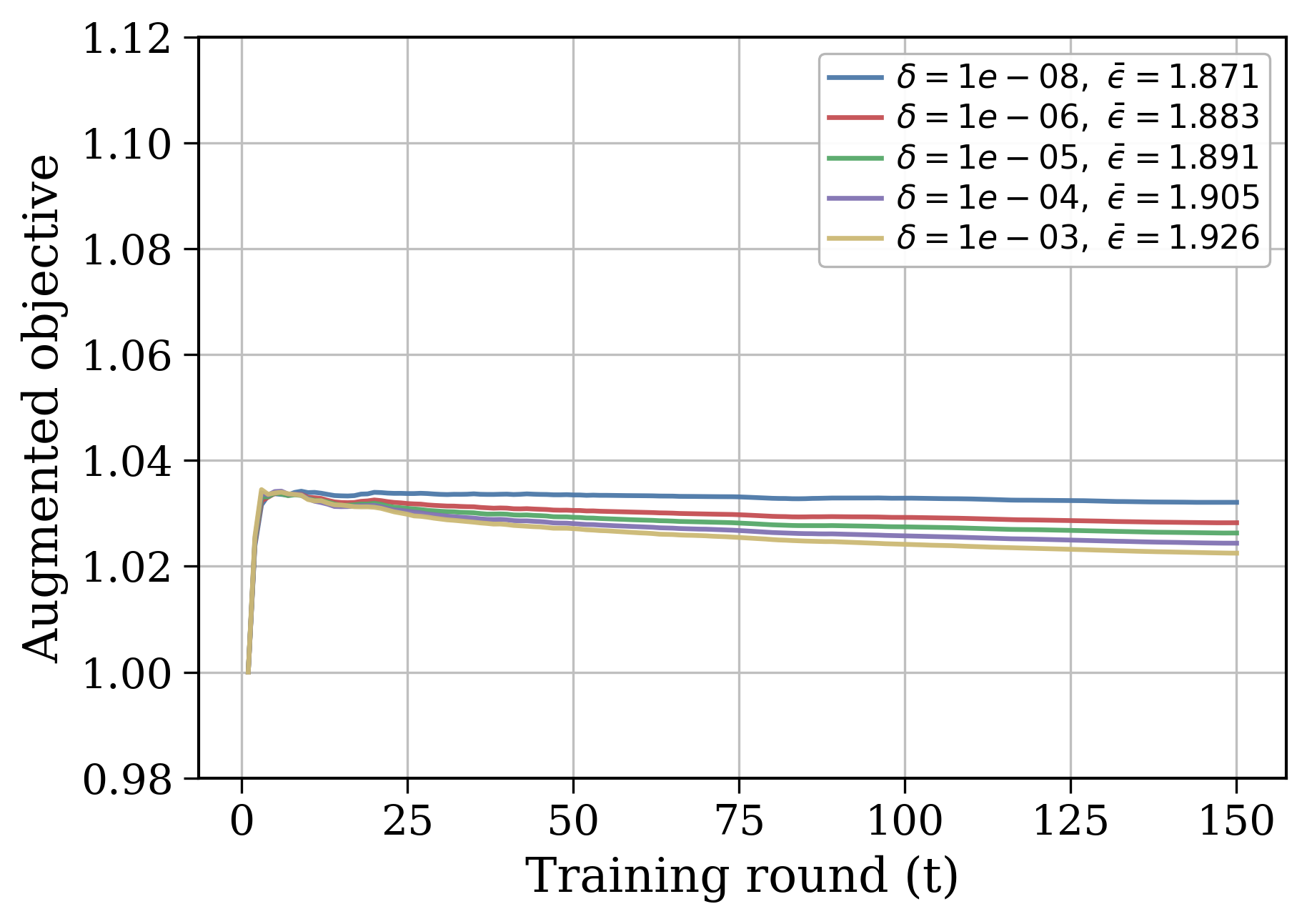}
        
        \vspace{2pt}
        \footnotesize $\epsilon = 0.15$
    \end{minipage}
    \hfill
    \begin{minipage}[t]{0.19\textwidth}
        \centering
        \includegraphics[width=\linewidth]{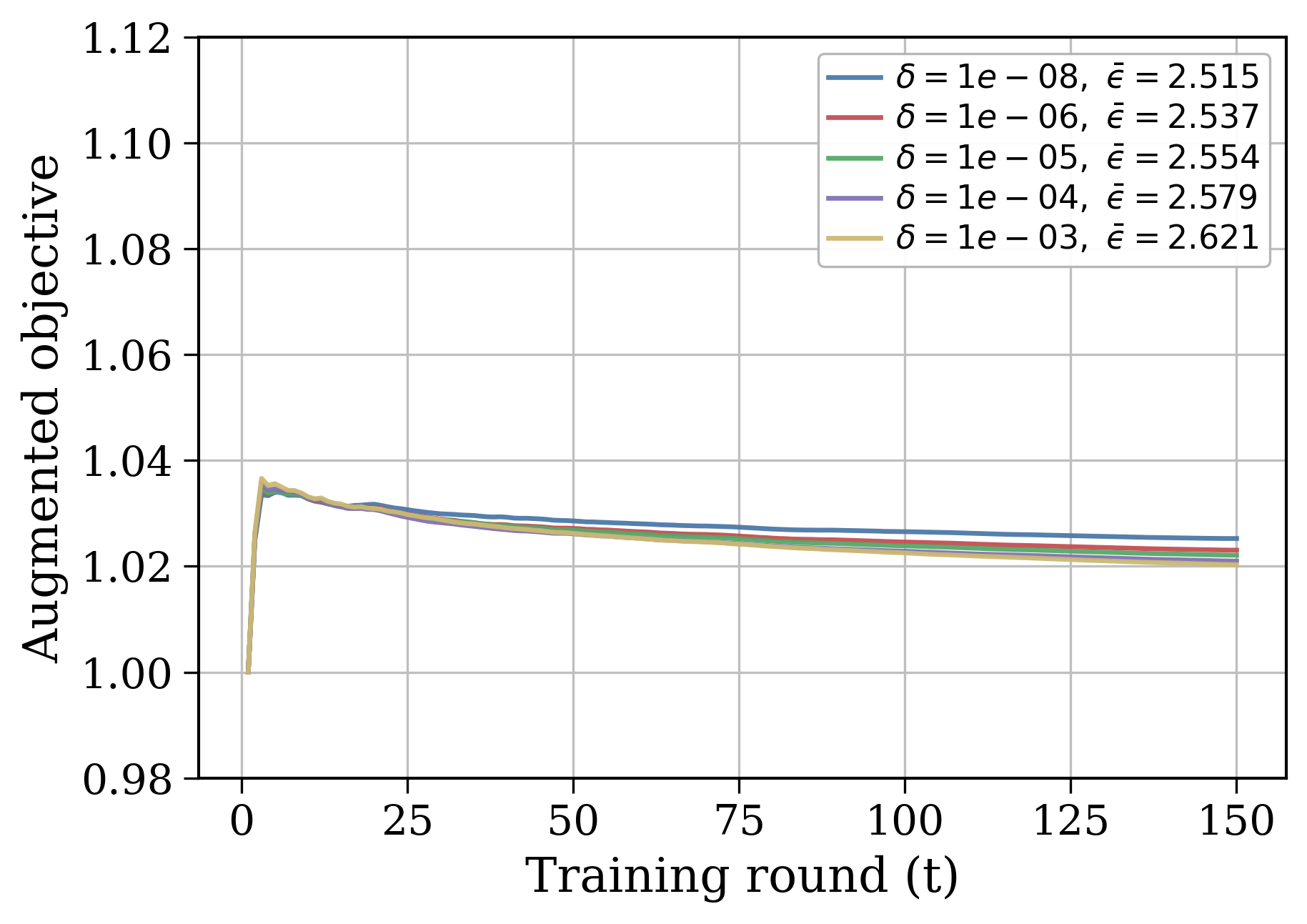}
        
        \vspace{2pt}
        \footnotesize $\epsilon = 0.2$
    \end{minipage}

    \vspace{10pt}

    % ============================================================
    % Row 2: Consensus violation plots
    % ============================================================
    \begin{minipage}[t]{0.19\textwidth}
        \centering
        \includegraphics[width=\linewidth]{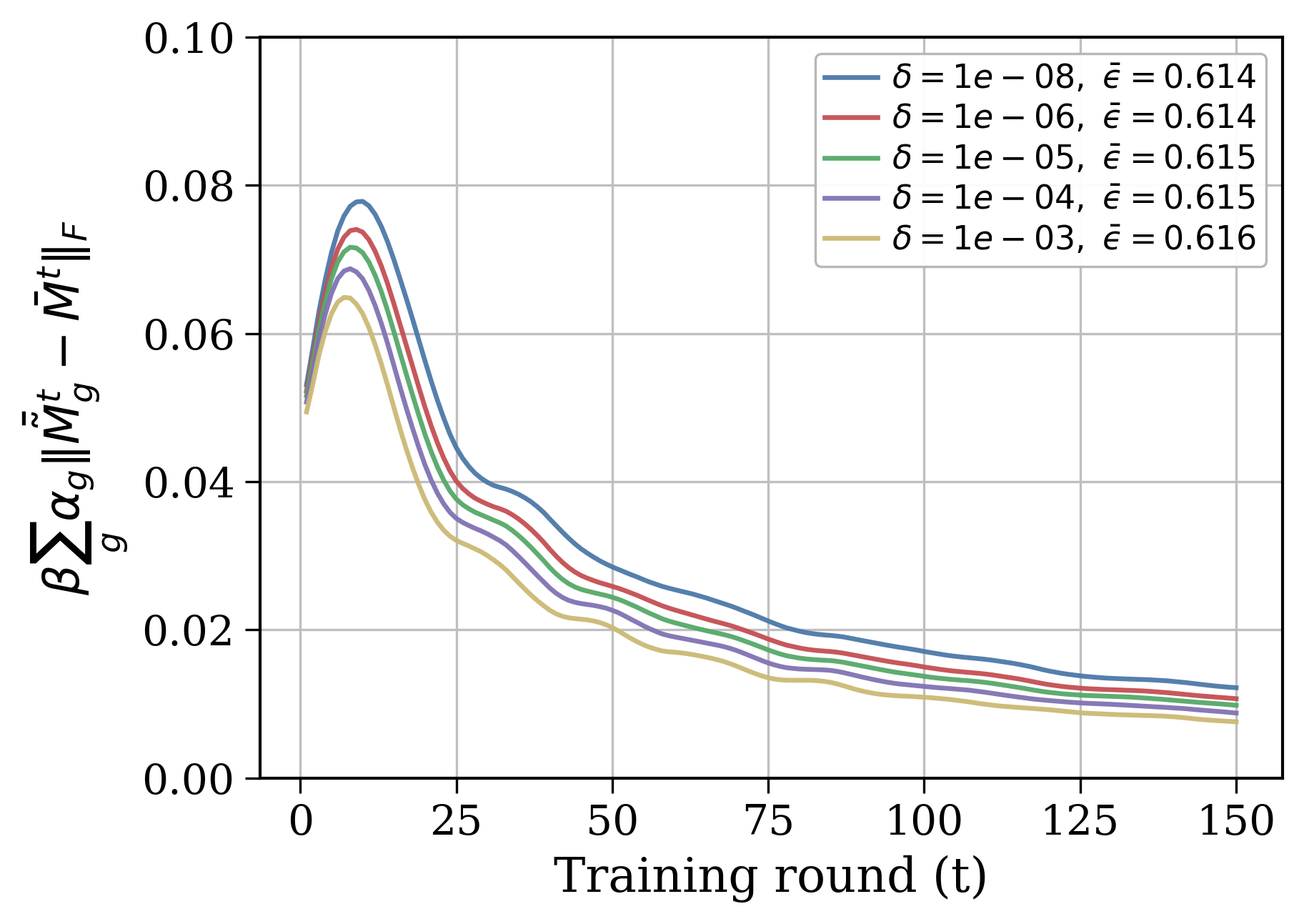}
        
        \vspace{2pt}
        \footnotesize $\epsilon = 0.05$
    \end{minipage}
    \hfill
    \begin{minipage}[t]{0.19\textwidth}
        \centering
        \includegraphics[width=\linewidth]{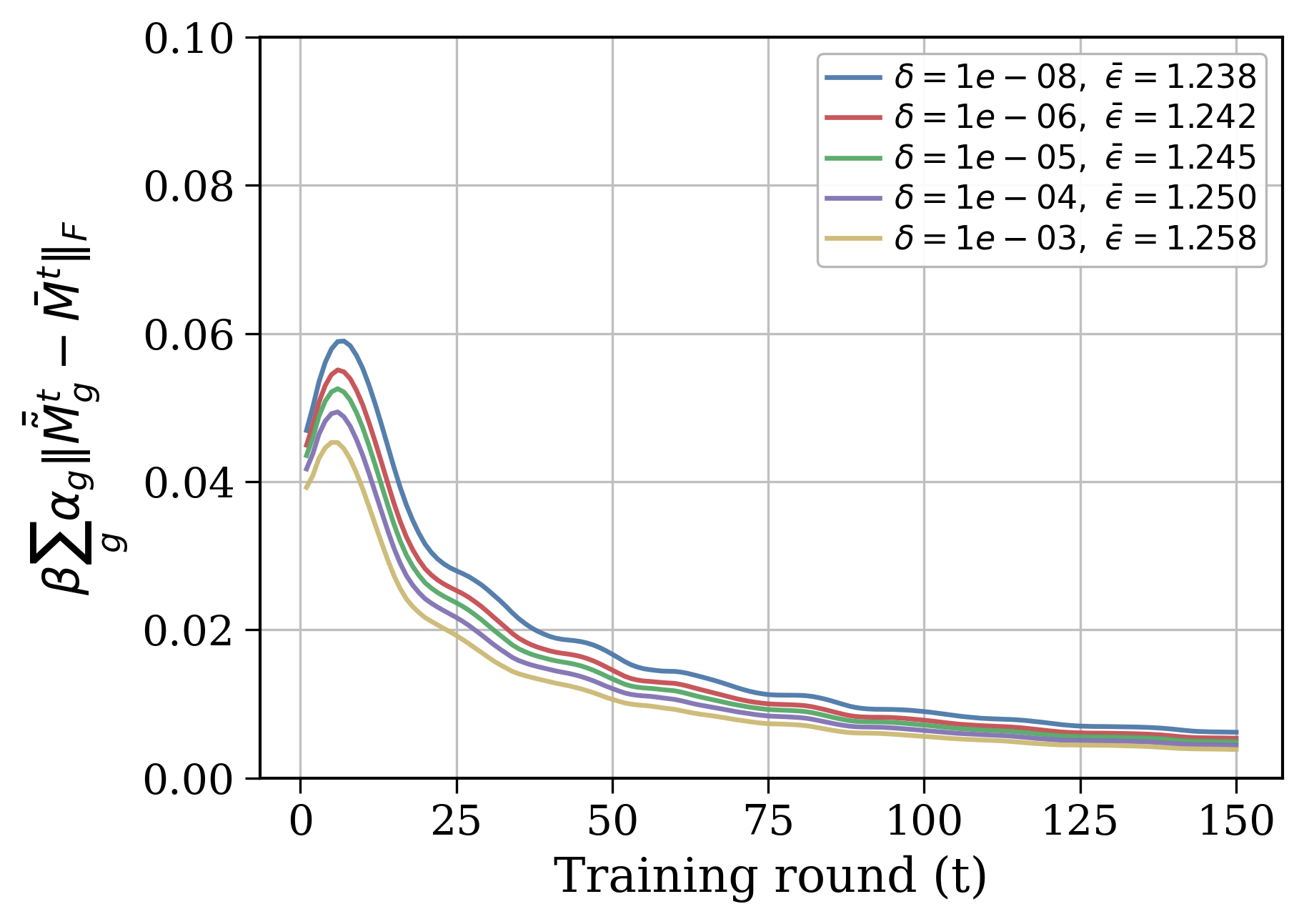}
        
        \vspace{2pt}
        \footnotesize $\epsilon = 0.1$
    \end{minipage}
    \hfill
    \begin{minipage}[t]{0.19\textwidth}
        \centering
        \includegraphics[width=\linewidth]{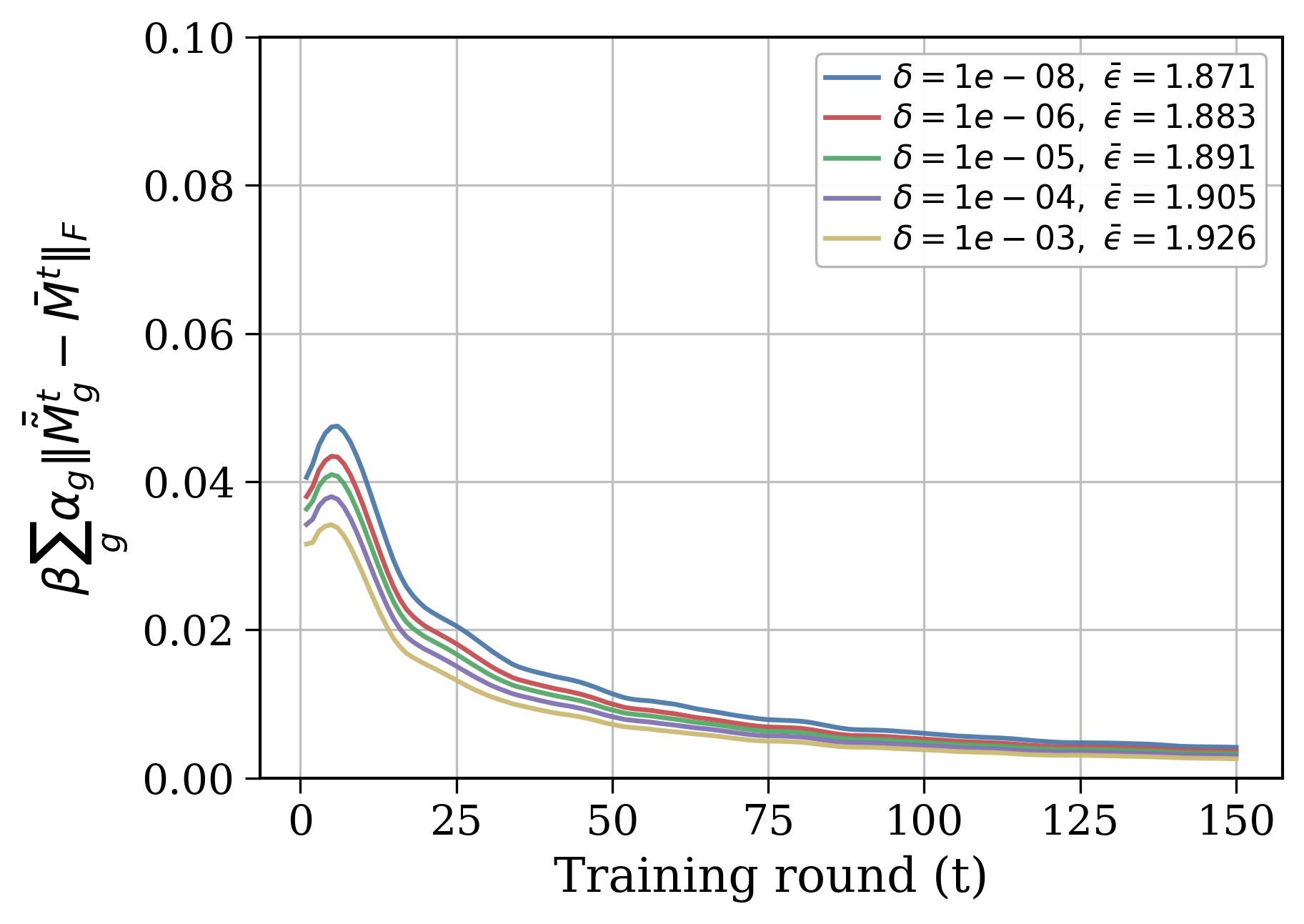}
        
        \vspace{2pt}
        \footnotesize $\epsilon = 0.15$
    \end{minipage}
    \hfill
    \begin{minipage}[t]{0.19\textwidth}
        \centering
        \includegraphics[width=\linewidth]{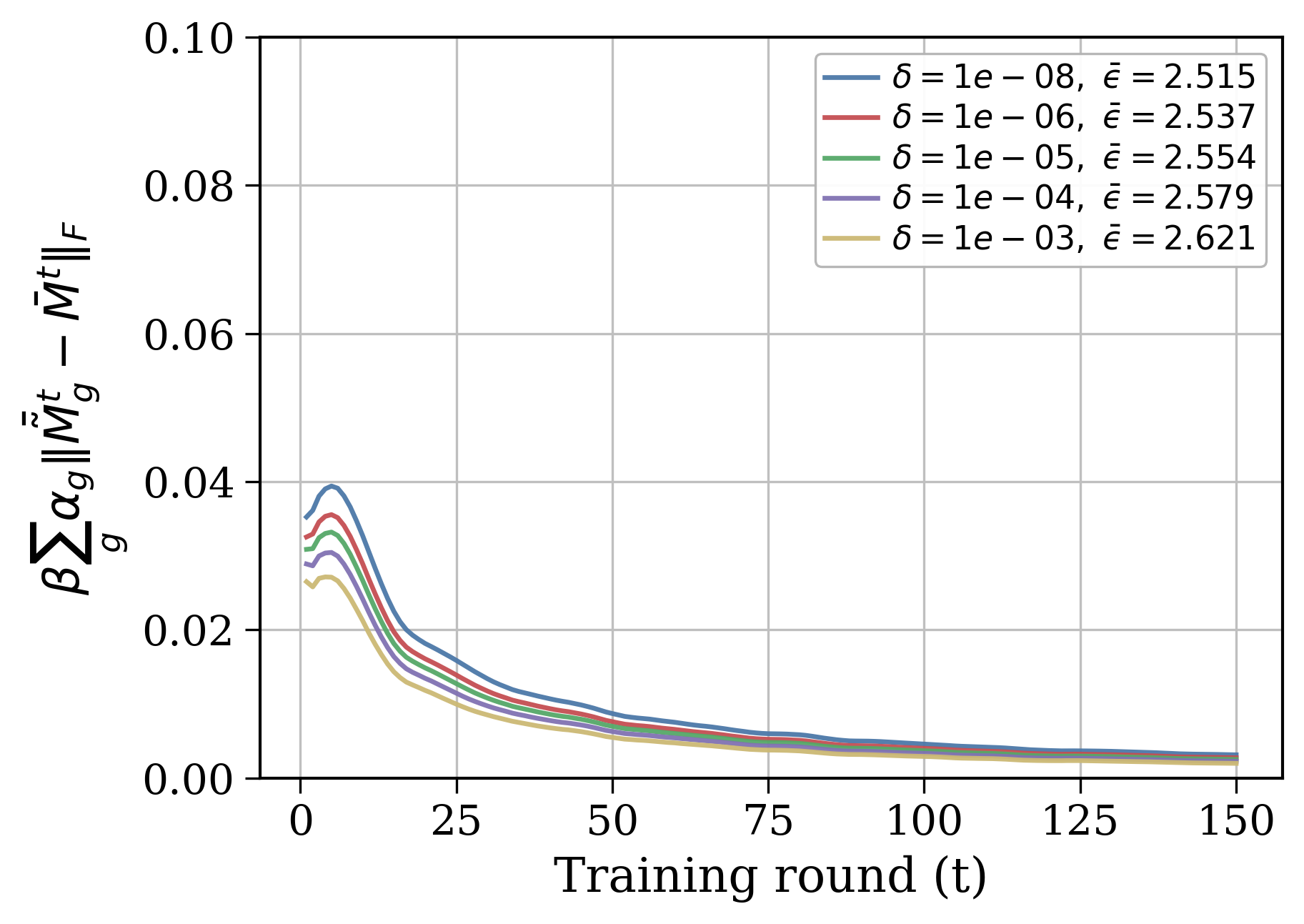}
        
        \vspace{2pt}
        \footnotesize $\epsilon = 0.2$
    \end{minipage}

    \caption{Additional convergence plots for the augmented objective and the consensus violation. The top row reports the running average of the augmented objective at the privatized iterates, while the bottom row reports the consensus violation. Each column corresponds to a different privacy budget \(\epsilon\).}
    \label{fig:appendix_convergence_components}
\end{figure}

\subsection{Controlled Study on Correlated Fault Classes}
\label{app:synthetic_overlap}

To identify the conditions under which a native multiclass formulation is most beneficial relative to OvA, we consider a controlled synthetic setting with correlated features and varying overlap between fault classes. Such conditions may arise in physical fault-diagnosis systems when faults in different components influence the same subsystem measurements, producing correlated sensor responses and similar fault signatures. This creates a setting in which the relationships among fault classes may be important for classification, making it well suited for comparing the non-private federated multiclass DRO model (Non-Private Fed DRO), which learns all class margins jointly, with its one-vs-all counterpart (Non-Private Fed OvA DRO), which trains independent binary classifiers. 

We generate a four-class synthetic dataset consisting of Normal, Fault A, Fault B, and Fault C, with $N_{\mathrm{train}}= 860$ training samples and $N_{\mathrm{test}}= 216$ test samples. The training set is imbalanced, with the Normal class overrepresented relative to the fault classes, whereas the test set is balanced. For each class, feature vectors are sampled from a multivariate normal distribution in an eight-dimensional feature space. The class-specific mean vector determines the class center, while a covariance matrix shared across all classes controls the feature spread and correlation structure. The covariance matrix is constructed so that two subsets of features are internally correlated, while the remaining features exhibit a weaker correlation structure. The class centers are chosen so that Fault A and Fault B form a related fault pair, while Fault C represents a more distinct fault mode. To vary the degree of overlap, we change the center of Fault B relative to Fault A, producing strong-, moderate-, and low-overlap regimes while keeping the covariance structure fixed. We introduce \(10\%\) random label noise only in the training set, while the test labels remain clean. Each experiment is repeated over 10 random seeds with \(G=4\) clients and \(T=50\) communication rounds. We select the Wasserstein radius \(r \in \{10^{-4},10^{-3},10^{-2}\}\) and label transportation cost \(\kappa \in \{0.25,0.5,0.75\}\) by 3-fold cross-validation, while fixing \(\rho=0.5\) and using the \(\ell_\infty\)-norm.

\begin{table}[H]
\centering
\caption{Performance of native multiclass DRO and OvA DRO across Fault A/Fault B overlap regimes.}
\label{tab:overlap_ablation_full_appendix}
\setlength{\tabcolsep}{4pt}
\renewcommand{\arraystretch}{1.0}
\begin{tabular}{llcc}

\toprule
Overlap setup & Metric & Non-Private Fed DRO & Non-Private Fed OvA DRO \\
\midrule

\multirow{10}{*}{Strong overlap}
& mCCR / Accuracy
& \textbf{0.8366 $\pm$ 0.0340}
& 0.7894 $\pm$ 0.0320 \\
& Macro-F1
& \textbf{0.8140 $\pm$ 0.0471}
& 0.7559 $\pm$ 0.0456 \\
& Macro Recall
& \textbf{0.8366 $\pm$ 0.0340}
& 0.7894 $\pm$ 0.0320 \\
& Worst-class Recall
& \textbf{0.4074 $\pm$ 0.1448}
& 0.3000 $\pm$ 0.1190 \\
& Normal Recall
& \textbf{0.9963 $\pm$ 0.0078}
& \textbf{0.9963 $\pm$ 0.0078} \\
& Fault A Recall
& \textbf{0.4074 $\pm$ 0.1448}
& 0.3000 $\pm$ 0.1190 \\
& Fault B Recall
& \textbf{0.9500 $\pm$ 0.0381}
& 0.8981 $\pm$ 0.0439 \\
& Fault C Recall
& \textbf{0.9926 $\pm$ 0.0096}
& 0.9630 $\pm$ 0.0276 \\
& Fault A/B Confusion Rate
& \textbf{0.2185 $\pm$ 0.0556}
& \textbf{0.2185 $\pm$ 0.0561} \\
& Fault A/B to Normal Rate
& \textbf{0.0843 $\pm$ 0.0345}
& 0.1481 $\pm$ 0.0503 \\

\midrule

\multirow{10}{*}{Moderate overlap}
& mCCR / Accuracy
& \textbf{0.9292 $\pm$ 0.0215}
& 0.8301 $\pm$ 0.0218 \\
& Macro-F1
& \textbf{0.9259 $\pm$ 0.0240}
& 0.8040 $\pm$ 0.0329 \\
& Macro Recall
& \textbf{0.9292 $\pm$ 0.0215}
& 0.8301 $\pm$ 0.0218 \\
& Worst-class Recall
& \textbf{0.7444 $\pm$ 0.0947}
& 0.3759 $\pm$ 0.1007 \\
& Normal Recall
& 0.9963 $\pm$ 0.0078
& \textbf{0.9981 $\pm$ 0.0059} \\
& Fault A Recall
& \textbf{0.7444 $\pm$ 0.0947}
& 0.3759 $\pm$ 0.1007 \\
& Fault B Recall
& 0.9889 $\pm$ 0.0199
& \textbf{0.9926 $\pm$ 0.0096} \\
& Fault C Recall
& \textbf{0.9870 $\pm$ 0.0232}
& 0.9537 $\pm$ 0.0200 \\
& Fault A/B Confusion Rate
& \textbf{0.0620 $\pm$ 0.0107}
& 0.1657 $\pm$ 0.0369 \\
& Fault A/B to Normal Rate
& \textbf{0.0630 $\pm$ 0.0292}
& 0.1222 $\pm$ 0.0335 \\

\midrule

\multirow{10}{*}{Low overlap}
& mCCR / Accuracy
& \textbf{0.9514 $\pm$ 0.0224}
& 0.8653 $\pm$ 0.0296 \\
& Macro-F1
& \textbf{0.9503 $\pm$ 0.0237}
& 0.8535 $\pm$ 0.0369 \\
& Macro Recall
& \textbf{0.9514 $\pm$ 0.0224}
& 0.8653 $\pm$ 0.0296 \\
& Worst-class Recall
& \textbf{0.8333 $\pm$ 0.0956}
& 0.5259 $\pm$ 0.1185 \\
& Normal Recall
& 0.9963 $\pm$ 0.0078
& \textbf{0.9981 $\pm$ 0.0059} \\
& Fault A Recall
& \textbf{0.8333 $\pm$ 0.0956}
& 0.5259 $\pm$ 0.1185 \\
& Fault B Recall
& 0.9981 $\pm$ 0.0059
& \textbf{1.0000 $\pm$ 0.0000} \\
& Fault C Recall
& \textbf{0.9778 $\pm$ 0.0259}
& 0.9370 $\pm$ 0.0250 \\
& Fault A/B Confusion Rate
& \textbf{0.0176 $\pm$ 0.0134}
& 0.0528 $\pm$ 0.0191 \\
& Fault A/B to Normal Rate
& \textbf{0.0630 $\pm$ 0.0351}
& 0.1556 $\pm$ 0.0503 \\

\bottomrule
\end{tabular}
\end{table}

Table~\ref{tab:overlap_ablation_full_appendix} shows that the Non-Private Fed DRO outperforms the Non-Private Fed OvA DRO across all three overlap regimes. The advantage is most pronounced in the moderate- and low-overlap settings, where Non-Private Fed DRO achieves substantially higher mCCR, Macro-F1, and worst-class recall. In particular, the gains are driven largely by improved recovery of Fault A, which is the most difficult class across all regimes. Under strong overlap, both methods experience similar Fault A/B confusion rates, but Non-Private Fed DRO produces fewer Fault A/B samples classified as Normal and achieves higher recall for both Fault A and Fault B. As the overlap decreases, the native model continues to improve while also reducing the Fault A/B confusion rate, indicating that the joint multiclass formulation better exploits the separation between related fault classes. Overall, these results support the use of a native multiclass DRO formulation when fault modes have correlated features and partially overlapping class structure.

\end{document}